\pdfoutput=1
\documentclass[10pt]{amsart}
\usepackage{amsthm,amsmath,amsfonts,amssymb}
\usepackage[section]{algorithm}
\usepackage{algpseudocode}
\usepackage{amssymb}
\usepackage[toc,page]{appendix}
\usepackage{array}
\usepackage{bbm}
\usepackage{calligra}
\usepackage{color}
\usepackage{comment}
\usepackage[dvipsnames]{xcolor}
\usepackage{enumerate}
\usepackage{enumitem}
\usepackage{fullpage}
\usepackage{graphicx}
\usepackage{mathMacros}
\usepackage{mathtools}
\usepackage{multirow}
\usepackage{pifont}
\usepackage{caption}
\usepackage{subcaption}
\usepackage{stmaryrd}
\usepackage{tablefootnote}

\usepackage[square,numbers]{natbib} 
\usepackage{hyperref}
\usepackage{cleveref}
\usepackage{autonum}
\cslet{blx@noerroretextools}\empty 
\hypersetup{
    colorlinks=true,
    linkcolor={red!85!black},
    citecolor={blue!75!black},
    filecolor=magenta,      
    urlcolor={blue!75!black},
    }
\colorlet{linkequation}{blue}

\numberwithin{equation}{section}
\theoremstyle{plain}
\newtheorem{theorem}{Theorem}[section]

\newtheorem{lemma}[theorem]{Lemma}
\newtheorem{proposition}[theorem]{Proposition}
\newtheorem{definition}[theorem]{Definition}
\newtheorem{assumption}[theorem]{Assumption}
\newtheorem{remark}[theorem]{Remark}

\theoremstyle{remark}

\DeclareMathAlphabet{\mathcalligra}{T1}{calligra}{m}{n}
\DeclareMathAlphabet{\mathpzc}{OT1}{pzc}{m}{it}

\algrenewcommand\algorithmicrequire{\textbf{Input:}}
\algrenewcommand\algorithmicensure{\textbf{Output:}}

\renewcommand{\epsilon}{\varepsilon}

\newcommand{\iid}{\stackrel{\mathrm{\tiny{i.i.d.}}}{\sim}}
\newcommand{\norm}[1]{\left\lVert#1\right\rVert}
\newcommand{\CSBM}{\textnormal{CSBM}}
\newcommand{\SBM}{\textnormal{SBM}}
\newcommand{\GMM}{\textnormal{GMM}}

\newcommand{\diag}{\textnormal{diag}}

\renewcommand{\sT}{\top}

\title{Optimal Exact Recovery in Semi-Supervised Learning:\\
A Study of Spectral Methods and Graph Convolutional Networks}

\author{Hai-Xiao Wang}
\address{Department of Mathematics, University of California, San Diego, La Jolla, CA 92093}
\email{h9wang@ucsd.edu}

\author{Zhichao Wang}
\address{Department of Mathematics, University of California, San Diego, La Jolla, CA 92093}
\email{zhw036@ucsd.edu}

\date{\today}

\begin{document}
\maketitle

\begin{abstract}
We delve into the challenge of semi-supervised node classification on the Contextual Stochastic Block Model (CSBM) dataset. Here, nodes from the two-cluster Stochastic Block Model (SBM) are coupled with feature vectors, which are derived from a Gaussian Mixture Model (GMM) that corresponds to their respective node labels. With only a subset of the CSBM node labels accessible for training, our primary objective becomes the accurate classification of the remaining nodes. Venturing into the transductive learning landscape, we, for the first time, pinpoint the information-theoretical threshold for the exact recovery of all test nodes in CSBM. Concurrently, we design an optimal spectral estimator inspired by Principal Component Analysis (PCA) with the training labels and essential data from both the adjacency matrix and feature vectors. We also evaluate the efficacy of graph ridge regression and Graph Convolutional Networks (GCN) on this synthetic dataset. Our findings underscore that graph ridge regression and GCN possess the ability to achieve the information threshold of exact recovery in a manner akin to the optimal estimator when using the optimal weighted self-loops. This highlights the potential role of feature learning in augmenting the proficiency of GCN, especially in the realm of semi-supervised learning.
\end{abstract}


\section{Introduction}\label{sec:intro}

Graph Neural Networks (GNNs) have emerged as a powerful method for tackling various problems in the domain of graph-structured data, such as social networks, biological networks, and knowledge graphs. The versatility of GNNs allows for applications ranging from node classification to link prediction and graph classification. To explore the mechanism and functionality behind GNNs, it is natural to assume certain data generation models, such that the fundamental limits of certain tasks appear mathematically. In particular, we focus on the synthetic data sampled from \emph{Contextual Stochastic Block Model} (CSBM) introduced in \cite{deshpande2018contextual}. In the binary \emph{Stochastic Block Model} (SBM), vertices are connected with probability $p$ when they are from the same community; otherwise $q$. The CSBM dataset extends the traditional SBM, where each node is additionally associated with a feature vector sampled from a corresponding \emph{Gaussian Mixture Model} (GMM). The parameters of CSBM are composed of the connection probabilities $p$ and $q$ in SBM and \emph{signal-to-noise ratio} (SNR) in GMM.

We investigate the \emph{semi-supervised} node classification problem, which aims to recover the labels of unknown nodes when some node labels are revealed. In the literature, the existing work has focused on the generalization property of GNN \cite{esser2021learning,bruna2017community,baranwal2021graph}, the role of nonlinearity \cite{lampert2023self} and the phenomenon of oversmoothing \cite{wu2022non}. While, in this paper, we focus on the fundamental limits of CSBM and explore the following questions.
\begin{enumerate}
    \item What is the necessary condition on the parameters of CSBM to classify all nodes correctly?
    \item What is the best possible accuracy for any algorithm when given the model parameters?
    \item Can we design an efficient estimator to achieve the best possible accuracy?
    \item How well does GNN perform under this evaluation metric?
\end{enumerate}
In addressing these challenges, we consider the \emph{transductive learning} framework, where the connections between known and unknown nodes are utilized efficiently, in contrast to the \emph{inductive learning} where only the connections among known nodes are involved. For the first time, we identify the \emph{Information-Theoretic} (IT) limits of CSBM for all algorithms and necessary condition to classify all nodes correctly, especially all unknown nodes. This discovery is pivotal as it sets a benchmark for evaluating the performance of various algorithms on this type of data for the node classification problem.

\subsection{Related work}
We provide some relevant previous literature below.

\subsubsection*{Unsupervised learning on \emph{CSBM}} For benchmarking and demonstrating theoretical guarantees, CSBM has emerged as a widely adopted data structure for the theoretical analysis of diverse algorithms. In the extremely \emph{sparse} graph setting, the first tight IT analysis for community detection on CSBM was provided by \cite{deshpande2018contextual} via a non-rigorous cavity method of statistical physics. Later, \cite{lu2023contextual} proved the conjecture in \cite{deshpande2018contextual}, and characterized the sharp threshold to detect the community. Meanwhile, \cite{abbe2022lp} established the sharp threshold for correctly classifying all nodes and provided a simple spectral algorithm that reaches the IT limits for optimal recovery.


\subsubsection*{Semi-supervised learning}
Theoretical analyses within the semi-supervised learning framework have previously addressed various aspects. First, semi-supervised linear regression was explored in a line of research work, e.g. \cite{azriel2022semi,ryan2015semi,chakrabortty2018efficient,tony2020semisupervised}. Using an information-theoretic approach, the generalization error was characterized in \cite{he2022information} for iterative methods and in \cite{aminian2022information} under covariate-shift setting. Moreover, \cite{belkin2004regularization} explored the task of labeling a large partially labeled graph via regularization and semi-supervised regression on graphs.
\cite{lelarge2019asymptotic, nguyen2023asymptotic} explored asymptotic Bayes risks on GMM in semi-supervised learning, whereas we extend this to CSBM under the perfect classification setting.

\subsubsection*{Graph Convolutional Networks \emph{(GCNs)}}
GCN is one of the most fundamental GNN architectures, introduced by \cite{kipf2017semisupervised}. There are many works that theoretically analyze GCNs from different perspectives. For example, \cite{wu2022non} studied the phenomenon of oversmoothing of GCN in CSBM; the expressivity of deep GCNs is studied by \cite{oono2019graph}; \cite{wei2022understanding,baranwal2023optimality} analyzed Bayesian inference in nonlinear GCNs; \cite{ma2022is} showed that GCNs can perform well over some heterophilic graphs, especially in CSBM. Additionally, \cite{huang2023graph} analyzed the feature learning of GCNs on modified CSBM (SNM therein), and \cite{lu2022learning} showed the learning performance on SBM based on the dynamics of coordinate descent on GCNs. However, currently, there is no complete analysis of the training dynamics for GCNs on CSBM. Based on the analysis of GCNs, there are many modifications of GCN architectures with theoretical improvements, for instance, line GNNs with the non-backtracking operator \cite{chen2018supervised} for community detection, simple spectral graph convolution \cite{zhu2021simple}, and graph attention \cite{fountoulakis2023graph,fountoulakis2022classification}.

\subsubsection*{Generalization theory of \emph{GCNs}} 
Many works have studied the generalization performance of GCNs. \cite{tang2023generalization} controlled the transductive generalization gap of GCNs trained by SGD, and \cite{bruna2017community} explored the community detection for SBM with GCNs. The generalization performance of GCNs on CSBM has been considered in \cite{baranwal2021graph,baranwal2022effects,chen2018supervised}. Compared with \cite{baranwal2021graph}, our result studied the exact recovery performance of linear GCNs on sparser CSBM. Moreover, 
\cite{shi2024homophily} provided heuristic formulas of the regression generalization error in GCN for CSBM, showing the double descent phenomenon. Later, \cite{duranthon2024asymptotic} extended the computation to arbitrary convex
loss and regularization for extreme sparse CSBMs. Differently, we proved the asymptotic training and test errors for linear regression on GCNs for sparse CSBMs. Recently, \cite{duranthon2023optimal} compared the optimal belief-propagation-based algorithm with general GNNs for CSBM under the constant degree regime. In terms of the generalization, the roles of self-loops and nonlinearity in GCNs have been studied in 
\cite{lampert2023self,kipf2017semisupervised}. Our results in GCNs also provide a way to choose the optimal self-loop weight in GCN to achieve optimal performance. 

\subsection{Main contributions}
Our contribution lies in the following five perspectives.
\begin{enumerate}
    \item Mathematically, for any algorithm, we derive the necessary and sufficient conditions for correctly classifying all nodes on CSBM under the semi-supervised setting. 
    \item When perfect classification is impossible, we characterize the lower bound of the asymptotic misclassification ratio for any algorithm.

    \item We devise a spectral estimator, provably achieving perfect classification down to IT limits. 
    \item We evaluate the efficacy of graph ridge regression and GCN on the CSBM for perfect classification. 
    \item We present a method for selecting the \textit{optimal self-loop} weight in a graph to optimize its performance. This approach offers novel insights into the modification of GCN architectures.
\end{enumerate}


\section{Preliminaries}\label{sec:prel}
\subsection{Node classification}
Let $\cV$ and $\cE$ denote the set of vertices and edges of graph $\gG$ respectively, with $|\cV| = N \in \N_{+}$. Assume that $\cV$ is composed of two disjoint sets $\cV_{+}, \cV_{-}$, i.e., $\cV = \cV_{+} \cup \cV_{-}$ and $\cV_{+} \cap \cV_{-} = \emptyset$. Let $\by \coloneqq [y_1, \ldots, y_N]^{\top} \in \{\pm 1\}^{N}$ denote the label vector encoding the community memberships, i.e., $\cV_{+} = \{i\in [N]: y_i >0\}$ and $\cV_{-} = \{i\in [N]: y_i < 0\}$. Assume the access to $\cG$ in practice. The goal is to recover the underlying $\by$ using the observations. Let $\widehat{\by}$ denote some estimator of $\by$. To measure the performance of the above estimator, define the \textit{mismatch} ratio between $\by$ and $\widehat{\by}$ by
\begin{align}
    \psi_N(\by, \widehat{\by}) \coloneqq \frac{1}{N} \min_{s \in \{\pm 1\} } |\{i\in[N]:\by_i\neq s\cdot \widehat{\by}_i\}|.
\end{align}
For the symmetric case, $|\cV_{+}| = |\cV_{-}| = N/2$, the \emph{random guess} estimation, i.e., determining the node label by flipping a fair coin, would achieve $50\%$ accuracy on average. An estimator is meaningful only if it outperforms the random guess, i.e., $\psi_N(\by, \widehat{\by}) \leq 0.5$. If so, $\widehat{\by}$ is said to accomplish \emph{weak} recovery. See \cite{abbe2018community} for a detailed introduction.

In this paper, we aim to address another interesting scenario when all the nodes can be perfectly classified, i.e., $\psi_N = 0$, which leads to the concept of \emph{exact} recovery.
\begin{definition}[Exact recovery]
The $\widehat{\by}$ is said to achieve the exact recovery (strong consistency) if 
\begin{align}
    \lim_{N\to \infty} \P(\psi_N(\by, \widehat{\by}) = 0) = \lim_{N\to \infty} \P(\widehat{\by} =  \pm \,\, \by) = 1.
\end{align}
\end{definition}

\subsection{Contextual Stochastic Block Model}
It is natural to embrace certain data generation models to study the mathematical limits of algorithm performance. The following model is in particular of our interests.
\begin{definition}[Binary Stochastic Block Model, \SBM]
     Assume $\ones^{\top}\by= 0$, i.e., $|\cV_{+}| = |\cV_{-}| = N/2$. Given $0< \alpha, \beta < 1$, for any pair of node $i$ and $j$, the edge $\{i, j\}\in\cE$ is sampled independently with probability $\alpha$ if $y_i = y_j$, i.e., $\P(A_{ij} = 1) = \alpha$, otherwise $\P(A_{ij} = 1) = \beta$. Furthermore, if $\bA \in \{0, 1\}^{N \times N}$ is symmetric and $A_{ii} = 0$ for each $i\in [N]$, we then write $\bA \sim \SBM(\by, \alpha, \beta)$.
\end{definition}
For each node $i \in \cV$, there is a feature vector $\bx_i$ attached to it. We are interested in the scenario where $\bx_i$ is sampled from the following \emph{Gaussian Mixture Model}.
\begin{definition}[Gaussian Mixture Model, \GMM]
    Given $N, d\in \N_{+}$, label vector $\by \in \{\pm 1\}^{N}$ and some fixed $\bmu \in \sS^{d-1}$ with $\|\bmu\|_2 = 1$, we write $\{\bx_i\}_{i=1}^{N} \sim \GMM (\bmu, \by, \theta)$ if $\bx_i = \theta y_i \bmu + \bz_i \in \R^d$ for each $i\in [N]$, where $\theta >0$ denote the signal strength, and $\{\bz_i \}_{i=1}^{N}\subset \R^{d}$ are i.i.d. random column vectors sampled from $\Normal (\bzero, \bI_d)$. Then by denoting $\bZ \coloneqq [\bz_1,\ldots,\bz_N]^{\sT}$, we re-write $\bX \in \R^{N \times d}$ as
    \begin{align}
        \bX \coloneqq [ \bx_1, \bx_2, \ldots, \bx_N]^{\sT} =\theta \by \bmu^{\top} + \bZ.\label{eqn:gauss_mixture}
    \end{align}
\end{definition}
In particular, this paper focuses on the scenario where $\cG$ and $\bX$ are generated in the following manner. 
\begin{definition}[Contextual Stochastic Block Model, \CSBM]\label{def:CSBM}
    Suppose that $N, d \in \N_{+}$, $0 \leq \alpha, \beta \leq 1$ and $\theta > 0$. We write $(\bA, \bX) \sim \CSBM (\by, \bmu, \alpha, \beta, \theta)$ if
    \begin{enumerate}[topsep=0pt,itemsep=-1ex,partopsep=1ex,parsep=1ex,label=(\alph*)]
        \item the label vector $\by$ is uniformly sampled from the set $\{\pm 1\}^{N}$, satisfying $\ones^{\sT} \by = 0$;  
        \item independently, $\bmu$ is sampled from uniform distribution over the $\sS^{d -1}\coloneqq \{ \bv \in \R^{d}: \|\bv\|_2 = 1\}$; 
        \item given $\by$, independently, we sample $\bA \sim \SBM (\by, \alpha, \beta)$ and $\bX \sim \GMM (\bmu, \by, \theta)$.
    \end{enumerate}
\end{definition}
{\CSBM} was first introduced in \cite{deshpande2018contextual}, where a tight analysis for the inference of latent community structure was provided. The \emph{information-theoretic} thresholds on \emph{exact} recovery \cite{abbe2022lp} and weak recovery \cite{lu2023contextual} were established under the \emph{unsupervised} learning regime, i.e., none of the node labels is revealed. However, the modern learning methods \cite{kipf2017semisupervised} performed on the popular datasets \cite{yang2016revisiting, bojchevski2018deep, shchur2018pitfalls} rely on the model training procedure, i.e., a fraction of node labels are revealed, which is the regime we will focus on.
\subsection{Semi-supervised learning on graph}
Assume that $n \in [0, N)$ node labels are revealed, denoted by $y_1, \ldots, y_n$ without loss of generality. Let $\sL = \{(\bx_{i}, y_i)\}_{i=1}^{n}$ denote the training samples and $\sU = \{ \bx_j\}_{j = n + 1}^{N}$ denote the set of feature vectors corresponding to the unrevealed nodes. Each vertex $v\in \cV$ is assigned to either $\cV_{\sL}$ or $\cV_{\sU}$ depending on the disclosure of its label, where $n \coloneqq |\cV_{\sL}|$, $m \coloneqq |\cV_{\sU}|$ with $N = n + m$. Let $\tau \coloneqq n/N$ denote the \emph{training ratio}. For simplicity, let $\by_{\sL} \in \{\pm 1\}^{n}$ and $\by_{\sU} \in \{\pm 1\}^{m}$ denote the \emph{revealed} and \emph{unrevealed} label vectors. We further denote $\cV_{\sL,+}=\{i\in\cV_{\sL}:y_i>0\}$, $\cV_{\sL,-}=\{i\in\cV_{\sL}:y_i<0\}$, $\cV_{\sU,+}=\{i\in\cV_{\sU}:y_i>0\}$ and $\cV_{\sU,-}=\{i\in\cV_{\sU}:y_i<0\}$. For instance in Figure~\ref{fig:SBM}, we aim to recover the labels of $\cV_{\sU,+}$ and $\cV_{\sU,-}$ based on known labels in $\cV_{\sL,+}$ and $\cV_{\sL,-}$. Under the \emph{semi-supervised} regime, the graph $\cG$ and feature matrix $\bX$ are generated in the following manner.
\begin{definition}[Semisupervised CSBM]\label{def:SemiCSBM}
    Suppose that $\by_{\sL}$, $\by_{\sU}$ are uniformly sampled from $\{\pm 1\}^{n}$, $\{\pm 1\}^{m}$ respectively, satisfying $\ones^{\sT}_n \by_{\sL} = \ones^{\sT}_m \by_{\sU} = 0$. After concatenating $\by = [\by_{\sL}^{\sT}, \by_{\sU}^{\sT}]^{\sT}$, we have $(\bA, \{\bx_i\}_{i=1}^{N})$ sampled from $\CSBM ( \by, \bmu, \alpha, \beta, \theta)$ in Definition~\ref{def:CSBM}.
\end{definition}

\begin{remark}
It reduces to the unsupervised regime if $n = 0$.
\end{remark}
Let $\cX = \mathrm{span}(\{\bx_i\}_{i=1}^{N})$ denote the \emph{feature space} and $\cY = \{\pm 1\}^{N}$ denote \emph{label} space. In practice, the access to the graph $\gG = (\cV, \cE)$, the feature vectors $\{\bx_i\}_{i=1}^{N}$ and the revealed labels $\by_{\sL}$ are guaranteed. At this stage, finding a predictor $h: \cX \times \cG \times \cY_{\sL} \mapsto \cY_{\sU}$ is our primary interest. Let $\widehat{\by}_{\sU}$ denote some estimator of $\by_{\sU}$. The \textit{mismatch} ratio under the semi-supervised regime can be re-written as
\begin{align}
    \psi_m(\by_{\sU}, \widehat{\by}_{\sU}) =  \frac{1}{m} \min_{s \in \{\pm 1\} } |\{i\in[m]:(\by_{\sU})_i\neq s(\widehat{\by}_{\sU})_i\}|.
\end{align}

\begin{figure}
    \centering
\begin{minipage}[h]{0.4\linewidth}
\centering
{\includegraphics[width=1\textwidth]{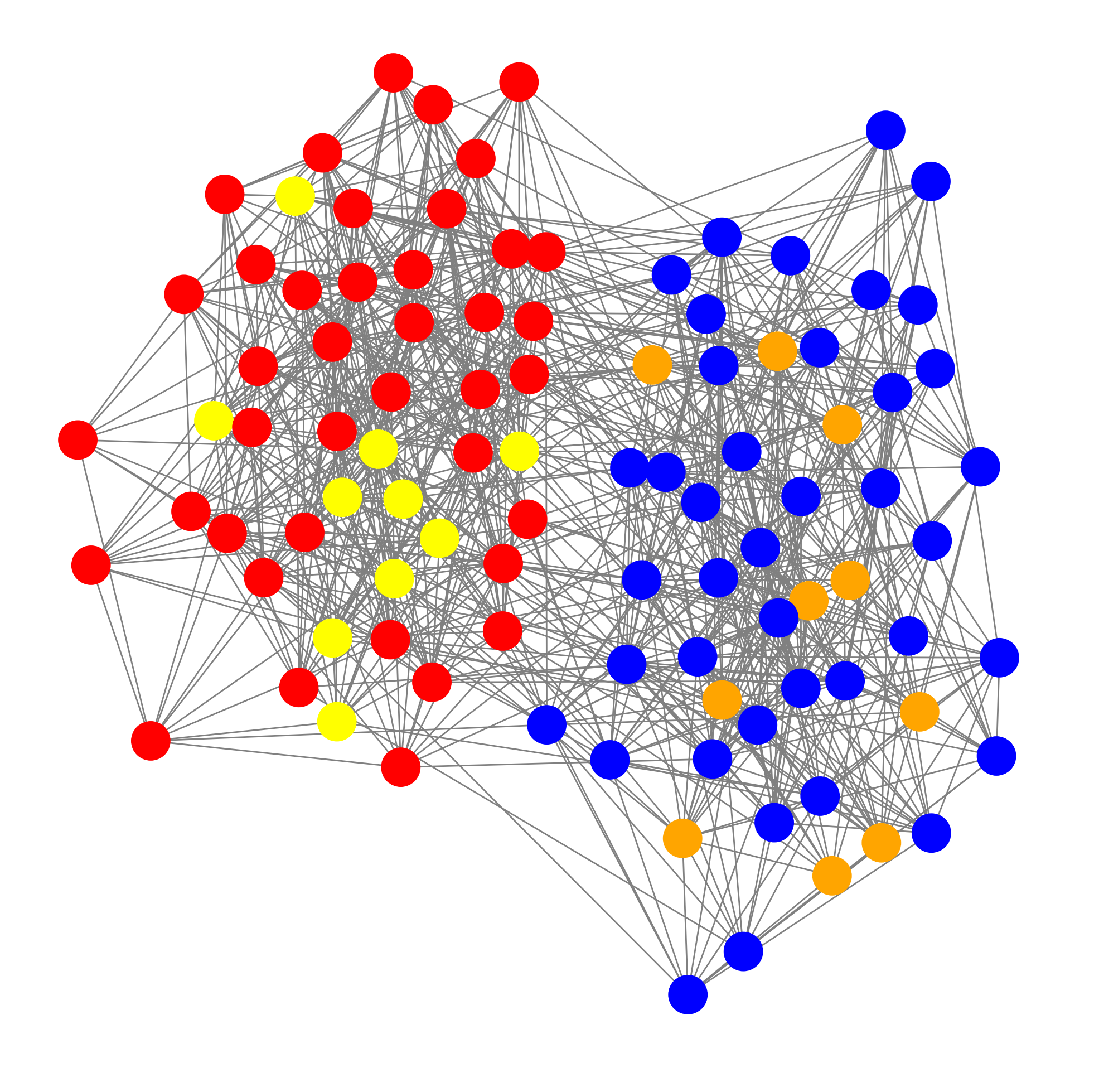}}
\end{minipage}
    \caption{An example of SBM under semi-supervised learning. Red: $\cV_{\sL,+}$; blue: $\cV_{\sL,-}$; yellow: $\cV_{\sU,+}$; and orange $\cV_{\sU,-}$.}
        \label{fig:SBM}
\end{figure}

\subsection{Graph-based transductive learning}
To efficiently represent the training and test data, we define the following two \emph{sketching} matrices.
\begin{definition}\label{def:sketching}
    Define $\bS_{\sL} \in \{0, 1\}^{n \times N}$, $\bS_{\sU} \in \{0, 1\}^{m \times N}$
    \begin{align}
        (\bS_{\sL})_{ij} \coloneqq &\, \indi\{i = j\}\cap \indi \{i \in \cV_{\sL}\},\\
        (\bS_{\sU})_{ij} \coloneqq &\, \indi\{i = j\}\cap \indi \{i \in \cV_{\sU}\}.
    \end{align}
Immediately, $\by_{\sL} = \bS_{\sL}\by$, $\by_{\sU} = \bS_{\sU} \by$. Define $\bX_{\sL} \coloneqq \bS_{\sL} \bX$, $\bX_{\sU} \coloneqq \bS_{\sU} \bX$, then $\bX = [\bX_{\sL}^{\sT}, \bX_{\sU}^{\sT}]^{\sT}$. The adjacency matrix $\bA \in \R^{N \times N}$ adapts the following block form
\begin{align}
    \bA = \begin{bmatrix}
        \bA_{\sL} & \bA_{\sL \sU}\\
        \bA_{\sU\sL} & \bA_{\sU}
    \end{bmatrix} \coloneqq \begin{bmatrix}
        \bS_{\sL} \bA \bS_{\sL}^{\sT} & \bS_{\sL} \bA \bS_{\sU}^{\sT}\\
        \bS_{\sU} \bA \bS_{\sL}^{\sT} & \bS_{\sU} \bA \bS_{\sU}^{\sT}\label{eqn:decomposion_A}
    \end{bmatrix}.
\end{align}
\end{definition}
In \textit{inductive} learning, algorithms are unaware of the nodes for testing during the learning stage, i.e., only $\bA_{\sL}$ and $\bX$ are used for training. The test graph $\bA_{\sU}$ is disjoint from $\bA_{\sL}$ and entirely unseen by the algorithm during the training procedure, since $\bA_{\sL \sU}$ is not used either. Notably, this kind of information wastage will reduce the estimator's accuracy.

In contrast, the entire graph $\bA$ is used for algorithm training under \textit{transductive} learning. The estimator benefits from the message-passing mechanism among seen and unseen nodes.

\section{Main results}\label{sec:main_results_CSBM}
To state our main results, we start with several basic assumptions. Recall that $\tau \coloneqq n/N$ denotes the \emph{training ratio}, where $\tau \in (0, 1)$ is some fixed constant.
\begin{assumption}[Asymptotics]\label{ass:asymptotics}
Let $q_m$ be some function of $m$ and $q_m \to \infty$ as $m \to \infty$. For $\CSBM ( \by, \bmu, \alpha, \beta, \theta)$ in  model ~\ref{def:CSBM}, we assume $\alpha = a \cdot q_m /m$ and $\beta = b \cdot q_m /m$, for some constants $a\neq b\in \R^+$, and 
    \begin{align}
        c_\tau \coloneqq \theta^4 [q_m (\theta^2 + (1-\tau)d/m)]^{-1},\label{eqn:ctau}
    \end{align}
is a fixed positive constant as $m \to \infty$.
Furthermore, we fix $n/N =\tau\in (0,1)$ as $N,n,m\to\infty$.
\end{assumption}
For instance, $\tau =0.2$, $\alpha=0.3$ and $\beta=0.05$ in Figure~\ref{fig:SBM}.
For $a, b \in \R^{+}$, denote $a_{\tau} = (1 - \tau)^{-1}a$, $b_{\tau} =(1 - \tau)^{-1}b$. Define the following \emph{rate function} by  
\begin{align}
    I(a_\tau,b_\tau,c_\tau) \coloneqq [(\sqrt{a_\tau} - \sqrt{b_\tau})^2 + c_\tau]/2 \label{eqn:rate_I_abc_tau},
\end{align} 
which will be applied to our large deviation analysis.

\subsection{Information-theoretic limits}\label{sec:ITLowerBoundsCSBM}
Note that $\by = [\by_{\sL}^{\sT}, \by_{\sU}^{\sT}]^{\sT}$, $(\bA, \bX) \sim \CSBM (\by, \bmu, \alpha, \beta, \theta)$, we first present the necessary condition for any estimator $\widehat{\by}_{\sU}$ to reconstruct $\by_{\sU}$ exactly.

\begin{theorem}[Impossibility]\label{thm:impossibility_CSBM}
Under \Cref{ass:asymptotics} with $q_m = \log(m)$, as $m \to \infty$, every algorithm will mis-classify at least $2$ vertices with probability tending to $1$ if $I(a_{\tau}, b_{\tau}, c_{\tau}) < 1$.
\end{theorem}
We explain the proof sketch above. For the node classification problem, the best estimator is the Maximum Likelihood Estimator (MLE). If MLE fails exact recovery, then no other algorithm could achieve exact recovery. When $I(a_{\tau}, b_{\tau}, c_{\tau}) < 1$, we can prove that with high probability, MLE will not return the true label vector $\by_{\mathbb{U}}$, but some other configuration $\widetilde{\by}_{\mathbb{U}} \neq \by_{\mathbb{U}}$ instead, which leads to the failure of exact recovery. Similar idea showed up in \cite{abbe2015exact, kim2018stochastic, wang2023strong} before. On the other hand, the following result concerns the fundamental limits of any algorithm.
\begin{theorem}\label{thm:ITlowerbounds_CSBM}
Under \Cref{ass:asymptotics} with $q_{m} \gg 1$, any sequence of estimators $\widehat{\by}_{\sU}$ satisfies
    \begin{align}
        \liminf_{m \to \infty} q^{-1}_m \log \E \psi_m(\by_{\sU}, \widehat{\by}_{\sU}) \geq - I(a_{\tau}, b_{\tau}, c_{\tau}).
    \end{align}
\end{theorem}
Informally, the result of \Cref{thm:ITlowerbounds_CSBM} can be interpreted as $\E \psi_m(\by_{\sU}, \widehat{\by}_{\sU}) \geq e^{-I(a_{\tau}, b_{\tau}, c_{\tau})q_m}$, which gives the lower bound on the expected mismatch ratio for any estimator $\widehat{\by}_{\sU}$. This rate function $I(a_\tau,b_\tau,c_\tau)$ in \eqref{eqn:rate_I_abc_tau} is derived from the analysis of the \emph{large deviation principle} (LDP) for $\bA$ and $\bX$, with details deferred to \Cref{lem:WmuLDP}.

\subsection{Optimal spectral estimator}\label{sec:spectral}

\subsubsection{The construction of spectral estimators}
Define the \emph{hollowed Gram} matrix $\bG = \cH(\bX \bX^{\top}) \in \R^{N \times N}$ by $G_{ij} = \<\bx_i, \bx_j\>\indi_{\{i \neq j\}}$. Similarly, $\bG$ adapts the block form as in \eqref{eqn:decomposion_A}. Let $\lambda_i(\bA)$, $\lambda_i(\bA_{\sU})$ (resp. $\lambda_i(\bG)$, $\lambda_i(\bG_{\sU})$) denote the $i$-th largest eigenvalue of $\bA$, $\bA_{\sU}$ (resp. $\bG$, $\bG_{\sU}$), and $\bu_i(\bA_{\sU})$, $\bu_i(\bG_{\sU})$ are the corresponding unit eigenvectors. Define the index $\ell^{\star} = 2 \cdot \indi\{ a > b\} + m \cdot \indi\{ a < b\}$ and the ratio
\begin{align}
    \widehat{\kappa}_{\ell^{\star}} = \log\Big( \frac{\lambda_1(\bA_{\sU}) + \lambda_{\ell^{\star}}(\bA_{\sU})}{\lambda_1(\bA_{\sU}) - \lambda_{\ell^{\star}}(\bA_{\sU})} \Big). \label{eqn:hatkappa_lstar}
\end{align}
The index $\ell^{\star}$ is used to differentiate the homophilic $(a > b)$ and heterophilic $(a < b)$ graphs. We then define
\begin{subequations}
\begin{align}
  \widehat{\by}_{\mathrm{SBM}} &\, \coloneqq
   \widehat{\kappa}_{\ell^{\star}} \big(\frac{1}{\sqrt{m}} \bA_{\sU \sL} \by_{\sL}  + \lambda_{\ell^{\star}}(\bA_{\sU}) \bu_{\ell^{\star}}(\bA_{\sU}) \big)  \label{eqn:yhatSBM} \\
 \widehat{\by}_{\mathrm{GMM}} &\, \coloneqq \frac{2\lambda_1 (\bG)}{N\lambda_1(\bG) + dN} \Big( \frac{\bG_{\sU\sL} \by_{\sL}}{\sqrt{m}} + \lambda_1 (\bG_{\sU})\bu_1(\bG_{\sU}) \Big). \label{eqn:yhatGMM}
\end{align}
\end{subequations}
It is natural to discard the graph estimator when $a = b$ reflected by $\widehat{\kappa}_{\ell^{\star}} = 0$, since no algorithm could outperform random guess on the Erd\H{o}s-R\'{e}nyi graph. Consequently, the ideal estimator, inspired by \emph{semi-supervised} \emph{principal component analysis}, is given by $\sign(\widehat{\by}_{\mathrm{PCA}})$, where
\begin{equation}
 \widehat{\by}_{\mathrm{PCA}} \coloneqq \widehat{\by}_{\mathrm{SBM}} + \widehat{\by}_{\mathrm{GMM}}. \label{eqn:pcaEstimator}
\end{equation}
Pseudocode of the spectral algorithm is given below.
\begin{algorithm} 
   \caption{Partition via spectral estimator}
   \label{alg:PCA}
\begin{algorithmic}
   \Require $\bA$, $\bX$, $\by_{\sL}$.  
   \State{Compute the gram matrix $\bG$.}
   \State{Construct $\widehat{\by}_{\mathrm{SBM}}$ and $\widehat{\by}_{\mathrm{GMM}}$ defined in \eqref{eqn:yhatSBM} and \eqref{eqn:yhatGMM} respectively. }
    \State{Construct the $\widehat{\by}_{\mathrm{PCA}}$ in \eqref{eqn:pcaEstimator}. }
   \Ensure $\widehat{\cV}_{\sU, +} \coloneqq \{i \in \cV_{\sU}: (\widehat{\by}_{\mathrm{PCA}})_{i} > 0\}$ and $\widehat{\cV}_{\sU, -} \coloneqq \{i \in \cV_{\sU}: (\widehat{\by}_{\mathrm{PCA}})_{i} < 0\}$.
\end{algorithmic}
\end{algorithm}

\subsubsection{The regime $q_m \gtrsim \log(m)$}
\Cref{thm:impossibility_CSBM} and \Cref{thm:achievability_CSBM} (a) establish the sharp threshold for exact recovery, i.e., $I(a_{\tau}, b_{\tau}, c_{\tau}) = 1$, verified by the numerical simulations in Figures \ref{fig:pca_exact_c50} and \ref{fig:pca_optimal_c50}.

\begin{theorem}\label{thm:achievability_CSBM}
Let \Cref{ass:asymptotics} hold and $q_m \gtrsim \log(m)$.
\begin{enumerate}[topsep=0pt,itemsep=-1ex,partopsep=1ex,parsep=1ex,label=(\alph*)]
    \item (Exact). When $I_{\tau} = I(a_{\tau}, b_{\tau}, c_{\tau}) > 1$, $\widehat{\by}_{\mathrm{PCA}}$ achieves exact recovery with probability at least $1 - m^{1 - I_{\tau}}$.
    \item (Optimal). When $I_{\tau} = I(a_{\tau}, b_{\tau}, c_{\tau}) \leq 1$, it follows
        $$
            \limsup_{m \to \infty} q^{-1}_m \log \E \psi_m(\by_{\sU}, \widehat{\by}_{\mathrm{PCA}}) \leq - I_{\tau}.
        $$
\end{enumerate}
\end{theorem}
Informally, the second part of \Cref{thm:achievability_CSBM} can be understood as $\E \psi_m(\by_{\sU}, \widehat{\by}_{\sU}) \leq e^{-I_{\tau} \cdot q_m}$, which establishes an upper bound of the expected mismatch ratio. It matches the lower bound in \Cref{thm:ITlowerbounds_CSBM}. In that sense, even though exact recovery is impossible when $I(a_{\tau}, b_{\tau}, c_{\tau}) \leq 1$ by \Cref{thm:impossibility_CSBM}, the estimator $\widehat{\by}_{\mathrm{PCA}}$ in \eqref{eqn:pcaEstimator} arrives the lowest possible error rate when $q_m \gtrsim \log(m)$.

\begin{figure*}[h]  
\centering
\begin{minipage}[t]{0.49\linewidth}
\centering
\subcaptionbox{$c_{\tau} = 0.5$.}
{\includegraphics[width=1.18\textwidth]{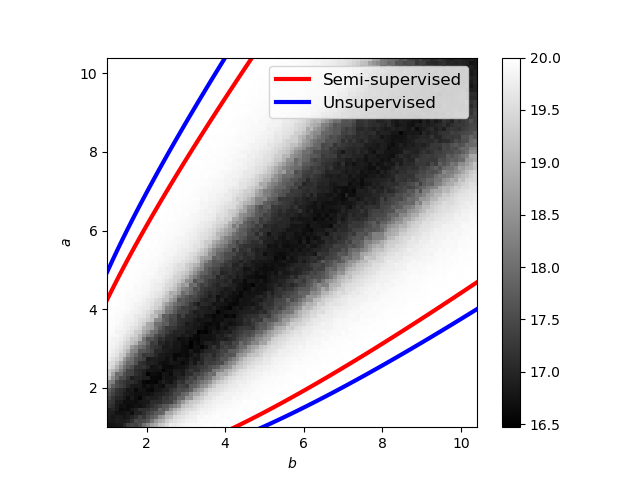}} 
\end{minipage}
\begin{minipage}[t]{0.49\linewidth}
\centering
\subcaptionbox{$c_{\tau} = 1.5$}
{\includegraphics[width=0.95\textwidth]{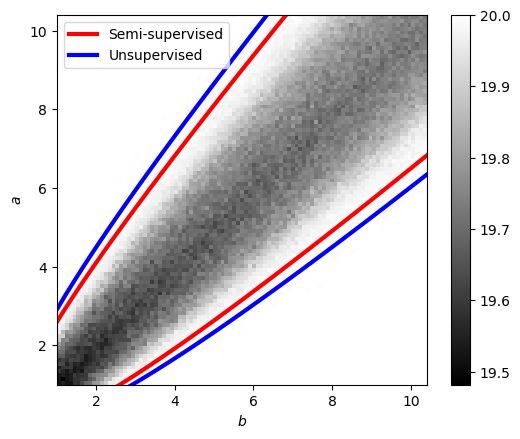}} 
\end{minipage}
\caption{ \small{Performance of $\widehat{\by}_{\mathrm{PCA}}$ in \eqref{eqn:pcaEstimator}: fix $N = 800$, $\tau = 0.25$ and vary $a$ ($y$-axis) and $b$ ($x$-axis) from $1$ to $10.5$. For each parameter configuration $(a_{\tau}, b_{\tau}, c_{\tau})$, we compute the frequency of exact recovery over $20$ independent runs. Light color represents a high chance of success. Phase transitions occurs at the red curve $I(a_{\tau}, b_{\tau}, c_{\tau}) = 1$, as proved by Theorems \ref{thm:impossibility_CSBM} and \ref{thm:achievability_CSBM}.}}  \label{fig:pca_exact_c50}
\end{figure*} 

\subsubsection{The regime $1 \ll q_m \ll \log(m)$}
When the graph becomes even sparser, where the expected degree of each vertex goes to infinity slower than $\log(m)$, the previous estimator $\widehat{\by}_{\mathrm{PCA}}$ in \eqref{eqn:pcaEstimator} is no longer valid. There are two main issues. First, $\widehat{\kappa}_{\ell^{\star}}$ was designed for the estimation of $\log(a_{\tau}/b_{\tau})$, but it does not converge to $\log(a_{\tau}/b_{\tau})$ anymore when $1 \ll q_m \ll \log(m)$, since $\lambda_{1}(\bA_{\sU})$ and $\lambda_{\ell^{\star}}(\bA_{\sU})$ no longer concentrate around $\frac{\alpha + \beta}{2}$ and $\frac{\alpha - \beta}{2}$ \cite{feige2005spectral}. To get rid of that, we refer to the quadratic forms $\ones^{\sT} \bA_{\sL}\ones$ and $ \by_{\sL}^{\sT}\bA_{\sL}\by_{\sL}$, which still present good concentration properties. Formally, we use the following $\widetilde{\kappa}_{\ell^{\star}}$ instead
\begin{align}
    \widetilde{\kappa}_{\ell^{\star}} \coloneqq \log\Big( \frac{ \ones^{\sT} \bA_{\sL}\ones + \by_{\sL}^{\sT}\bA_{\sL}\by_{\sL} }{ \ones^{\sT} \bA_{\sL}\ones - \by_{\sL}^{\sT}\bA_{\sL}\by_{\sL} } \Big).
\end{align}
The second issue is that, the entrywise eigenvector analysis of $\bu_{2}(\bA_{\sU})$ breaks down due to the lack of concentration. To overcome that, we let $\widehat{\by}_{\mathrm{G}} = \sign(\widehat{\by}_{\mathrm{GMM}})$. Note that $\bA_{\sU}\widehat{\by}_{\mathrm{G}}$ is closed to $\sqrt{m}\lambda_{\ell^{\star}}(\bA_{\sU}) \bu_{\ell^{\star}}(\bA_{\sU})$, then the new graph estimator is defined through
\begin{align}
      \widetilde{\by}_{\mathrm{SBM}} &\, \coloneqq \widetilde{\kappa}_{\ell^{\star}} \big(\bA_{\sU \sL} \by_{\sL}  + \bA_{\sU} \widehat{\by}_{\mathrm{G}}\big)/\sqrt{m}  \label{eqn:ytildeSBM}
\end{align}
Combining the above reasoning together, the estimator for under the general case is given by $\sign(\widetilde{\by}_{\mathrm{PCA}})$, where
\begin{align}
    \widetilde{\by}_{\mathrm{PCA}} = \widetilde{\by}_{\mathrm{SBM}} + \widehat{\by}_{\mathrm{GMM}}.\label{eqn:tildepcaEstimator}
\end{align}
The following result shows that $\widetilde{\by}_{\mathrm{PCA}}$ achieves the lowest possible expected error rate when $1 \ll q_{m} \ll \log(m)$.
\begin{theorem}\label{thm:general_pcaEstimator}
Let \Cref{ass:asymptotics} hold, then it follows
    \begin{align}
        \limsup_{m \to \infty} q^{-1}_m \log \E \psi_m \big(\by_{\sU}, \sign(\widetilde{\by}_{\mathrm{PCA}}) \big) \leq - I(a_{\tau}, b_{\tau}, c_{\tau}).
    \end{align}
\end{theorem}

\begin{figure}[h]
    \centering
    \begin{minipage}[t]{0.5\linewidth}
    \centering
    {\includegraphics[width=1\textwidth]{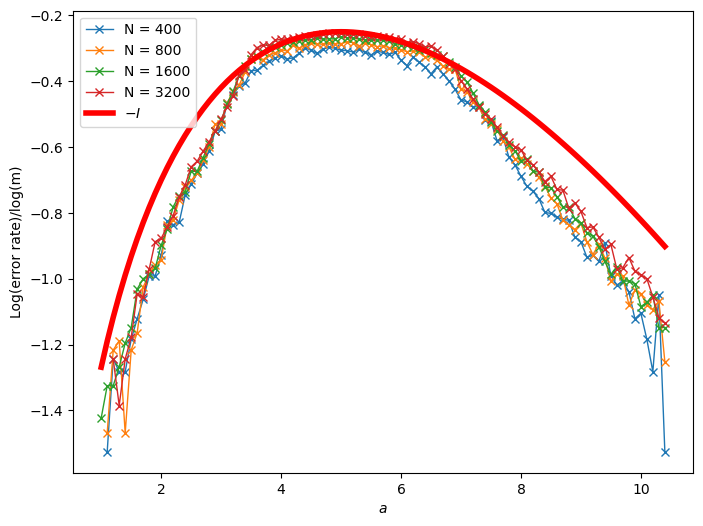}}
    \end{minipage}
    \caption{ \small{The $y$-axis is $q_m^{-1}\log(\E \psi_m)$, the average mismatch ratio on the logarithmic scale. The $x$-axis is $a$, varying from $0$ to $10.5$. Fix $b = 5$, $\tau = 0.25$, $c_{\tau} = 0.5$. The red curve is $-I(a_{\tau}, b_{\tau}, c_{\tau})$, the lower bound predicted by Theorem \ref{thm:ITlowerbounds_CSBM}. The experiments over different $N$ shows that $\widehat{\by}_{\mathrm{PCA}}$ achieves the information-theoretical limits, as proved in Theorems \ref{thm:achievability_CSBM} and \ref{thm:general_pcaEstimator}. }}
    \label{fig:pca_optimal_c50}
    \vspace{-3mm}
\end{figure}

\subsubsection{Comparation with unsupervised regime}
When only the sub-graph $\cG_{\sU} = (\cV_{\sU}, \cE_{\sU})$ is observed, it becomes an unsupervised learning task on $\cG_{\sU}$, where the data is equivalently sampled from $(\bA_{\sU}, \{\bx_i\}_{i=1}^{m}) \sim \CSBM (\by_{\sU}, \bmu, \alpha, \beta, \theta)$ with $\alpha = a q_m/m$ and $\beta = b q_m/m$. The rate function can be obtained by simply taking $\tau = 0$ with
$a_0 = a$, $b_0=b$, and $c_0 = q_m^{-1} (\theta^2 + d/m)^{-1} \theta^4$, aligning with the result in \cite{abbe2022lp}. The difference between the two boundaries $I(a_\tau,b_\tau,c_\tau) = 1$ (\textcolor{red}{red}) and $I(a_0,b_0,c_0) = 1$ (\textcolor{blue}{blue}) is presented in \Cref{fig:pca_exact_c50}. A crucial observation is that, the extra information from $\bX_{\sU}$, $\bA_{\sU}$ and $\bA_{\sU\sL}$ shrinks the boundary for exact recovery, making the task easier compared with the unsupervised regime.
\subsection{Performance of ridge regression on linear GCN}\label{sec:ridge}
For $\CSBM ( \by, \bmu, \alpha, \beta, \theta)$, in this section, we focus on analyzing how these parameters $a,b,c_{\tau}$ and $\tau$ defined in Assumption~\ref{ass:asymptotics} affect the learning performances of the \textit{linear} graph convolutional networks. We consider a graph convolutional kernel $h(\bX) \in\R^{N\times d}$ which is a function of data matrix $\bX$ and adjacency matrix $\bA$ sampled from $\CSBM ( \by, \bmu, \alpha, \beta, \theta)$. We add self-loops and define the new adjacency matrix $\bA_{\rho} \coloneqq \bA + \rho \bI_{N}$, where $\rho > 0$ represents the intensity of self connections in the graph. Let $\bD_{\rho}$ be the diagonal matrix whose diagonals are the average degree for $\bA_{\rho}$, i.e., $[\bD_{\rho}]_{ii} = \frac{1}{N}\sum_{i=1}^{N}\sum_{j=1}^N(\bA_{\rho})_{ij}$ for each $i \in [N]$. For the linear graph convolutional layer, we will consider the following normalization:
\begin{align}\label{eq:h(X)}
 h(\bX)=\frac{1}{\sqrt{Nq_m}} \bD^{-1}_{\rho}\bA_{\rho}\bX,
\end{align}
Denote $\bD:=\bD_0$, indicating no self-loop added, and $D_0$ as the first diagonal of $\bD$. We study the linear ridge regression on $h(\bX)$. Compared with the general GCN defined in \cite{kipf2017semisupervised}, here we simplify the graph convolutional layer by replacing the degree matrix of $\bA$ by the average degree among all vertices. In this case, we can directly employ 
$ 
 \widetilde{h}(\bX)=\frac{1}{\widetilde d \cdot\sqrt{Nq_m}} \bA_{\rho}\bX
$
to approximate the $h(\bX)$, where $\widetilde d$ is the expected average degree defined by \eqref{eq:tilde_d}. Notice that for sparse graph $\bA$ under Assumption~\ref{ass:asymptotics}, the degree concentration for each degree is \textit{not} guaranteed, which is a different situation from \cite{baranwal2021graph,baranwal2022effects}.  

We now consider transductive learning on CSBM following the idea from \cite{baranwal2021graph,shi2024homophily}. Recall that the vertex set $\cV$ is split into two disjoint sets $\cV_{\sL}$ and $\cV_{\sU}$, where $n = |\cV_{\sL}|$,  $m = |\cV_{\sU}|$ and $N = n + m$. The training ratio $\tau = \frac{n}{N}$ as $N\to\infty$. From Definition~\ref{def:sketching}, we know that $\bS_{\sL}   \in [0, 1]^{n \times N}$, $\bS_{\sU}   \in [0, 1]^{m \times N}$,  $\bS_{\sL} \bX = \bX_{\sL} \in \R^{n\times d}$, and $\bS_{\sU} \bX = \bX_{\sU} \in \R^{m \times d}$. Then, the empirical loss of linear ridge regression (LLR) on $h(\bX)$ can be written as
\begin{align}
    L(\bbeta) = \frac{1}{n}\|\bS_{\sL} (h(\bX) \bbeta - \by ) \|_2^2 + \frac{\lambda}{n} \|\bbeta\|_2^2,
\end{align}
for any $\lambda>0$, where the solution to this problem is  
\begin{align}
    \widehat{\bbeta} =~& \underset{\bbeta \in \R^d}{\arg \min} \, L(\bbeta)\\ 
    = ~& (h(\bX)^{\top} \bP_{\sL} h(\bX) + \lambda \bI_d )^{-1}h(\bX)^{\top} \bP_{\sL}\by,\label{eq:regression_solu}
\end{align}
where $\bP_\sL =\bS_{\sL}^\top\bS_\sL\in \R^{N\times N}$ is a diagonal matrix. Similarly, define $\bP_\sU =\bS_{\sU}^\top\bS_\sU\in \R^{N\times N}$. Then the estimator of this linear ridge regression for $\lambda>0$ is
\begin{equation}\label{eq:regression_solu_y}
    \widehat\by_{\mathrm{LRR}}=\bS_\sU h(\bX)(h(\bX)^{\top} \bP_{\sL} h(\bX) + \lambda \bI_d )^{-1}h(\bX)^{\top} \bP_{\sL}\by.
\end{equation}
In the following, we aim to analyze the misclassification rate $\psi_m(\by_\sU,\widehat\by_{\mathrm{LRR}})$, the associated test and training risks in mean square error (MSE) defined by
\begin{align}
    \cR(\lambda) \coloneqq~& \frac{1}{m}\|\bS_{\sU}(h(\bX)\widehat\bbeta - \by)\|_2^2\label{eq:test}\\
    \cE(\lambda) \coloneqq~& \frac{1}{n}\|\bS_{\sL}(h(\bX)\widehat\bbeta - \by)\|_2^2.\label{eq:train}
\end{align}
Notice that \citet{shi2024homophily} also studied the asymptotic test and training risks for CSBM on linear GCNs but in a sparser graph $\bA$ with constant average degree. They utilized statistical physics methods with some Gaussian equivalent conjecture to compute the asymptotic risks. Below, we provide detailed statements for the exact recovery thresholds of $\widehat\by_{\mathrm{LRR}}$.

\begin{figure*}[t!]  
\centering
\begin{minipage}[t]{0.49\linewidth}
\centering
\subcaptionbox{Without self-loop.}
{\includegraphics[width=\textwidth]{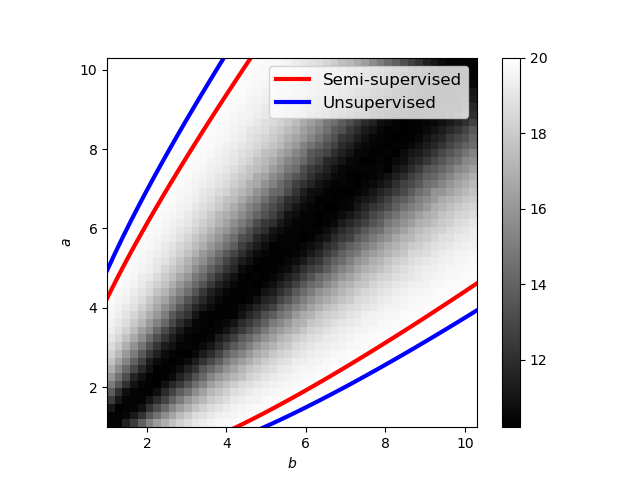}}
\end{minipage}
\begin{minipage}[t]{0.49\linewidth}
\centering
\subcaptionbox{With optimal self-loop}
{\includegraphics[width=\textwidth]{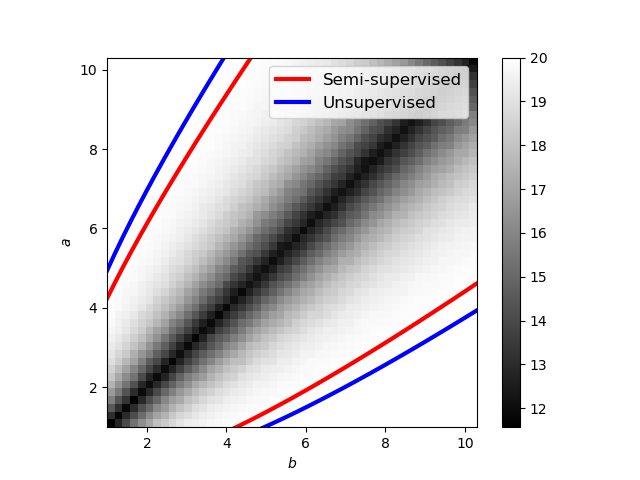}}
\end{minipage}
\caption{ \small{Performance of $\widehat{\by}_{\mathrm{LRR}}$ in \eqref{eq:regression_solu_y}. Fix $N = 800$, $\tau = 0.25$, $c_{\tau} = 0.5$. Compute the frequency of exact recovery over $20$ independent runs. When $I(a_{\tau}, b_{\tau}, c_{\tau}) > 1$, $\widehat{\by}_{\mathrm{LRR}}$ achieves exact recovery, as proved in Theorem \ref{thm:exact_linear} (a) and (b).}}  \label{fig:lrr_exact_c50}
\end{figure*}

\begin{theorem}[Exact recovery for graph convolution linear ridge regression]\label{thm:exact_linear}
Consider the ridge regression on the linear graph convolution $h(\bX)$ defined in \eqref{eq:h(X)} with estimator $\widehat\by_{\mathrm{LLR}}$ in \eqref{eq:regression_solu_y}. Assume that $\rho\lesssim q_m$, $\theta^2 = (1 + o(1)) c_{\tau} q_m$ and $q_m\lesssim d\lesssim \sqrt{Nq_m}$.
Then, under Assumption~\ref{ass:asymptotics}, we can conclude that
\begin{enumerate}[topsep=0pt,itemsep=-1ex,partopsep=1ex,parsep=1ex,label=(\alph*)]
\item When $\rho=0$,  then  $\P(\psi_m(\by_\sU,\sign(\widehat\by_{\mathrm{LRR}}))=0)\to 1$ as long as $I(a_\tau,b_\tau, 0 )>1$.
    \item When 
    \begin{equation}\label{eq:optimal_rho}
     \rho = \frac{2c_{\tau}}{\log(a_{\tau}/b_{\tau})} q_m,
    \end{equation} then  $\P(\psi_m(\by_\sU,\sign(\widehat\by_{\mathrm{LRR}}))=0)\to 1$ as long as $I(a_\tau,b_\tau,c_\tau)>1$.
    \item When $\rho = s q_m$ for some constant $s\in\R$, then  $$\P(\psi_m(\by_\sU,\sign(\widehat\by_{\mathrm{LRR}}))=0)\to 1$$ when $J(a_\tau, b_\tau, c_\tau,  \zeta, s )>1$, as $m\to\infty,$ where $\zeta:=\frac{\kappa \tau }{\kappa^2\tau+\lambda}$ and $\kappa:=\sqrt{c_\tau}\cdot\frac{ a_\tau-b_\tau+2s }{a_\tau+b_\tau+2s}$ for $\lambda>0$. Here rate function $J(a_\tau, b_\tau, c_\tau,  \zeta, s )$ is defined in Lemma~\ref{lem:rate_fun}. Additionally, we know that
    the exact recovery region $\{(a_\tau,b_\tau,c_\tau):J(a_\tau, b_\tau, c_\tau,  \zeta, s )>1\}$ is a subset of the optimal region $\{(a_\tau,b_\tau,c_\tau):I(a_\tau, b_\tau, c_\tau )>1\}.$
\end{enumerate}

\end{theorem}

Consequently, $\rho = \frac{2c_{\tau}}{\log(a/b)} q_m$ is the \textit{optimal} weighted self-loop to attain the exact recovery of labels $\by_\sU$ in this semi-supervised learning with linear ridge regression on $h(\bX)$. This is because in this case, the exact recovery for $\widehat\by_{\mathrm{LRR}}$ matches the information-theoretic lower bound in Theorem~\ref{thm:ITlowerbounds_CSBM}, i.e., below this threshold, no algorithms can perfectly recover all the unknown labels in $\cV_\sU$.

\begin{theorem}[Asymptotic training and test errors]\label{thm:error}
Consider $(\bA, \bX) \sim \CSBM (\by, \bmu, \alpha, \beta, \theta)$. Suppose that $\rho/q_m\to s\in\R$ and $d\lesssim N$. Under the Assumption~\ref{ass:asymptotics}, the training and test errors for linear ridge regression on $h(\bX)$ defined by \eqref{eq:regression_solu} are asymptotically satisfying the following results. For any fixed $\lambda>0,$ both training and test errors in MSE loss defined in \eqref{eq:test} and \eqref{eq:train} satisfy
\begin{align}
    \cE(\lambda) \text{ and }\cR(\lambda)\to~& \frac{\lambda^2}{(\kappa^2\tau+\lambda)^2},
\end{align}
almost surely, as $m,N\to\infty$, where $\kappa$ is defined in Theorem~\ref{thm:exact_linear} (c).
\end{theorem}
\begin{figure}[h]
    \centering
    \begin{minipage}[t]{0.49\linewidth}
    \centering
    {\includegraphics[width=1\textwidth]{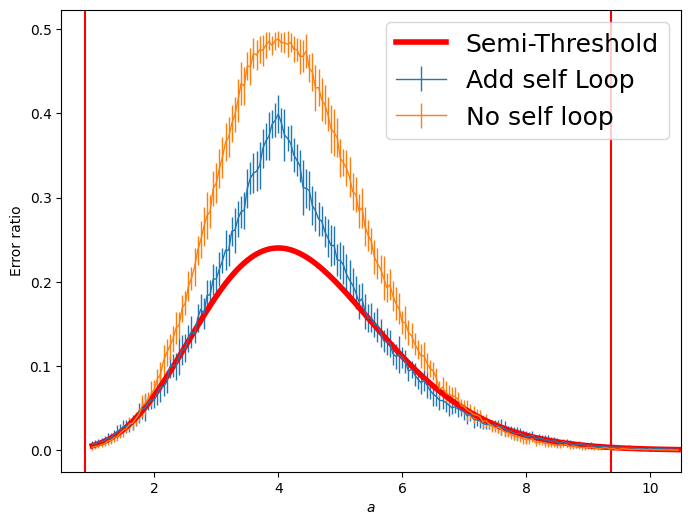}}
    \end{minipage}
    \caption{\small{The $y$-axis is $\E \psi_m$ , the average mismatch ratio over $20$ independent runs. The $x$-axis is $a$, varying from $0$ to $10.5$. Fix $b = 4$, $\tau = 0.25$, $c_{\tau} = 0.5$, $N = 400$. The red curve is $m^{-I(a_{\tau}, b_{\tau}, c_{\tau})}$, the lower bound predicted by Theorem \ref{thm:ITlowerbounds_CSBM} with $q_m = \log(m)$. This experiment shows that $\widehat{\by}_{\mathrm{LRR}}$ achieves a lower mismatch ratio when adding self-loop in the area $I(a_{\tau}, b_{\tau}, c_{\tau}) < 1$, where the exact recovery is impossible.}}\label{fig:lrr_optimal_c50}
\end{figure}

\subsection{Performance of GCN with gradient-based training}\label{sec:NN}
In this section, we study the feature learning of GCN on
$(\bA, \bX) \sim \CSBM (\by, \bmu, \alpha, \beta, \theta)$ with $n$ known labels and $m$ unknown labels to be classified. We focus on gradient-based training processes. From Section~\ref{sec:ridge}, we can indicate that the self-connection (or self-loop) weight $\rho$ plays an important role in exact recovery on test feature vertices. It turns out that we need to find the optimal $\rho$ in \eqref{eq:optimal_rho} for self-loop weight to ensure the exact recovery threshold approaches to the IT bound studied in Section~\ref{sec:ITLowerBoundsCSBM} for graph learning. A similar idea is also mentioned in \cite{kipf2017semisupervised}, whereas the equal status of self-connections and edges to neighboring nodes may not be a good assumption for a general graph dataset. Hence, we raise a modified training process for GCNs: in general, learning on graphs also requires learning the optimal self-loop weight for the graph, i.e., we should also view parameter $\rho$ in graph $\bA_\rho=\bA+\rho\bI_N$ as a trainable parameter. Although the optimal $\rho$ in Section~\ref{sec:ridge} for CSBM on semisupervised learning is due to LDP analysis (see Appendix~\ref{sec:LDP_ridge}), we can generally apply a spectral method to achieve oracle $\rho$ in $\eqref{eq:optimal_rho}$. We denote
\[\bA_s=\bA+s q_m\bI_N,\quad \bD_s = sq_m\bI_N+\bD,\]
where $\bD$ is a diagonal matrix with the average degree for each diagonal. In this section, we view $s\in\R$ as a trainable parameter.
Let us consider a general two-layer graph convolutional neural network defined by
\begin{equation}\label{eq:NN}
    f(\bX):=\frac{1}{\sqrt{K}}\sigma(\bD_s^{-1}\bA_s\bX\bW)\ba
\end{equation}
where the first-layer weight matrix is $\bW\in\R^{d\times K}$ and second layer weight matrix is $\ba\in\R^K$ for some $K\in\N$. Here, $\bW,s$ are training parameters for this GCN. We aim to train this neural network with training label $\by_\sL$ to predict the labels for vertices in $\cV_\sU$. Notice that when training $\bW$, we want $\bW$ to learn (align with) the correct feature $\bmu$ in the dataset. As studied in \cite{baranwal2021graph}, for CSBM with a large threshold, the data point feature is linearly separable, hence there is no need to introduce a nonlinear convolution layer in \eqref{eq:NN}. So we will consider $\sigma(x)=x$ below. In practice, nonlinearity for node classification may not be important in certain graph learning problems \cite{wu2019simplifying,he2020lightgcn}.

We train this GCN in two steps. First, we train the $\bW$ with a large gradient descent step on training labels. By choosing the suitable learning rate, we can allow $\bW$ to learn the feature $\bmu$. Let us define the MSE loss by
\begin{equation}\label{eq:loss}
    \cL(\bW,s)= \frac{1}{2n}(f(\bX)-\by)^\top\bP_\sL(f(\bX)-\by).
\end{equation}
The analysis for GD with a large learning rate to achieve feature learning is analogous with \cite{ba2022high,damian2022neural}. We extend this analysis to one-layer GCNs. Precisely, we take a one-step GD with a weight decay $\lambda_1$ and learning rate $\eta_1$: 
\begin{equation}
 \bW^{(1)} = \bW^{(0)} - \eta_1 \Big(\nabla_{\bW^{(0)}} \mathcal{L}(\bW^{(0)},s^{(0)}) + \lambda_1 \bW^{(0)} \Big).
\end{equation}
Secondly, we find out the optimal $s$ based on \eqref{eq:optimal_rho}. Here, we only use training labels $\by_\sL$. Let
\begin{equation}
s^{(1)}=\frac{2}{n^2q_m}(\by_\sL^\top\bX_\sL\bW^{(1)}\ba)\Big/\log\Big(\frac{\mathbf{1}^\top\bA\mathbf{1} + \by^\top_\sL\bA_\sL\by_\sL}{\mathbf{1}^\top\bA\mathbf{1} -\by^\top_\sL\bA_\sL\by_\sL}\Big).
\end{equation}
This construction resembles the spectral methods defined in \eqref{eqn:ytildeSBM}. Meanwhile, we can also replace this estimator with the gradient-based method to optimize $s$ in MSE loss which is shown in Appendix~\ref{sec:NN_proof}. However, to attain the IT bound, the \textit{nonlinearity} of $\sigma(x)$ in \eqref{eq:NN} plays an important role when applying GD to find optimal self-loop weight $s$. This observation is consistent with results by \citet{wei2022understanding,baranwal2023optimality}, where nonlinearity needed for GCN to obtain certain Bayes optimal in sparse graph learning.

\begin{assumption}
\label{assump:NN}  
Consider $N,d,K\to\infty$, $n\asymp N$, $K\asymp N$, $q_m = \log (m)$ and $d=o(q_m^2)$. We assume that at initialization $s^{(0)}=0$,  and 
$\sqrt{K}\cdot[\bW^{(0)}]_{ij}\iid\cN(0,1), ~\sqrt{K}\cdot[\ba]_j\iid \Unif\{\pm 1\}$, for all $i\in[d],j\in [K]$. 
\end{assumption}
With initialization stated in Assumption~\ref{assump:NN} and trained parameters $\bW^{(1)}$ and $s^{(1)}$, we derive a GCN estimator for unknown labels which matches with Theorems~\ref{thm:ITlowerbounds_CSBM} and \ref{thm:impossibility_CSBM}.
\begin{theorem}\label{thm:NN}
    Under Assumptions~\ref{ass:asymptotics} and~\ref{assump:NN}, suppose that learning rate $\eta_1=\Theta(K/\sqrt{q_m})$ and weight decay rate $\lambda_1=\eta_1^{-1}$. Then, estimator $\widehat\by_{\mathrm{GCN}}=\bS_{\sU}f(\bX)$ with $\bW=\bW^{(1)}$ and $s=s^{(1)}$ satisfies that $$\P(\psi_m(\by_\sU,\sign(\widehat\by_{\mathrm{GCN}}))=0)\to 1$$ when $I(a_\tau, b_\tau, c_\tau )>1$, as $m\to\infty.$ Hence, GCN can attain the IT bound for the exact recovery of CSBM.
\end{theorem}

\begin{figure}
\begin{minipage}[h]{0.45\linewidth}
\centering 
{\includegraphics[width=1.0\textwidth]
{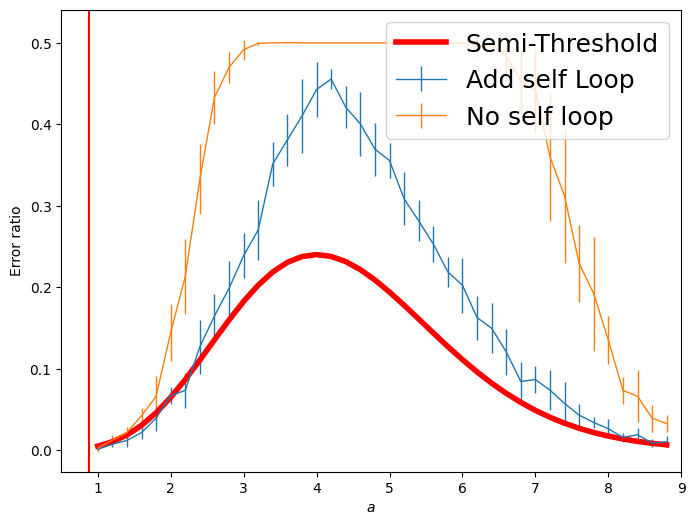}} 
\end{minipage}  
\caption{\small{Mismatch ratio difference of $\widehat\by_{\mathrm{GCN}}$ when with or without self-loop for fixed $b = 4$, $c_{\tau} = 0.5$, $N = 400$.}} \label{fig:GCN_optimal_c50} 
\end{figure}

\begin{remark}
\cite{duranthon2023optimal} proposed the AMP-BP algorithm to solve the community detection problem under CSBM, where the expected degree of each vertex is constant, i.e., $q_m = O(1)$. By contrast, this manuscript focuses on the regime $q_m \gg 1$ as in Assumption~\ref{ass:asymptotics}. Theorem~\ref{thm:NN} shows that the GCN achieves exact recovery when $I(a_\tau, b_\tau, c_\tau )>1$. However, the performance of the GCN is not characterized when $I(a_\tau, b_\tau, c_\tau )<1$, and it is still unclear whether it would match the lower bound proved in Theorem~\ref{thm:ITlowerbounds_CSBM}, i.e., the optimality of GCN remains open. From simulations in Figure~\ref{fig:GCN_optimal_c50}, we observe that below the IT bound ($I(a_\tau, b_\tau, c_\tau )<1$), there is a gap between theoretical optimal error (red curve) and the simulated mismatch ratio by GCN estimators.
\end{remark}

\begin{figure*}[h]  
\centering
\begin{minipage}[h]{0.48\linewidth}
\centering
\subcaptionbox{Exact recovery counts without self-loop.}
{\includegraphics[width=1\textwidth]{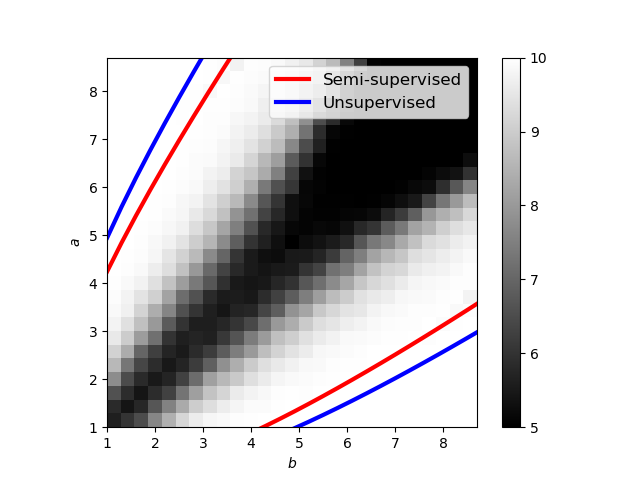}} 
\end{minipage}
\begin{minipage}[h]{0.48\linewidth}
\centering 
\subcaptionbox{Exact recovery counts with self loop $\rho$.}
{\includegraphics[width=1\textwidth]
{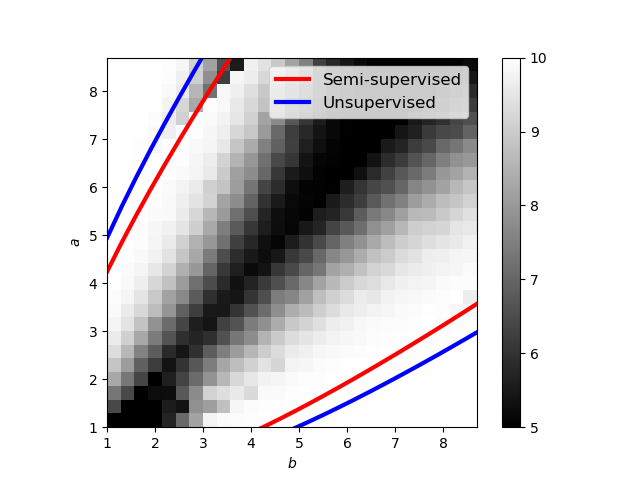}}
\end{minipage} 
\caption{
 \small Performance of $\widehat\by_{\mathrm{GCN}}$ when $N = 400$, $\tau = 0.25$, $c_{\tau} = 0.5$.} \label{fig:GCN_exact_c50}   
\end{figure*}

\section{Numerical simulations}\label{sec:simulation}
In this section, we present numerical simulations for the algorithms we investigated above.
\subsection{Optimal spectral method}
The efficacy of the spectral estimator $\widehat{\by}_{\mathrm{PCA}}$ is demonstrated in Figures \ref{fig:pca_exact_c50} (a) and \ref{fig:pca_optimal_c50} for $c_{\tau} = 0.5$ and Figure \ref{fig:pca_exact_c50} (b) for $c_{\tau} = 1.5$. We fix $N = 800$, $\tau = 0.25$, but vary $a$ ($y$-axis) and  $b$ ($x$-axis) from $1$ to $10.5$ in \Cref{fig:pca_exact_c50}, and compute the frequency of exact recovery over $20$ independent trials for each parameter configuration $(a_{\tau}, b_{\tau}, c_{\tau})$. Here, a lighter color represents a higher success chance. The (\textcolor{red}{red}) and (\textcolor{blue}{blue}) curves represent the boundaries for exact recovery under semi-supervised and unsupervised regimes respectively. A larger $c_{\tau}$ implies a stronger signal in node features, which shrinks the boundary for exact recovery and makes the problem easier. In Figure \ref{fig:pca_optimal_c50}, we fix $b = 5$ but vary $a$ ($x$-axis) from $1$ to $10.5$. The simulations for the average mismatch ratio are presented on the logarithmic scale over different choices of $N$. Clearly, $\log \E \psi_m$ will approach the lower bound (red curve) as proved in Theorems \ref{thm:ITlowerbounds_CSBM}, \ref{thm:achievability_CSBM} (2).

\subsection{Ridge regression on linear GCNs}
The efficacy of the ridge estimator $\widehat\by_{\mathrm{LRR}}$ is presented in Figures \ref{fig:lrr_exact_c50} and \ref{fig:lrr_optimal_c50}. We fix $N = 800$, $\tau = 0.25$ and $c_{\tau} = 0.5$ in \Cref{fig:lrr_exact_c50}, but vary $a$ ($y$-axis) and  $b$ ($x$-axis) from $1$ to $10.5$, where $20$ independent trials are performed on each $(a_{\tau}, b_{\tau}, c_{\tau})$. The difference between the (a) and (b) lies on the choice of the self-loop density $\rho$, where we take $\rho = 0$ in (a) but $\rho = 2c_{\tau} q_m \log(a_{\tau}/b_{\tau})$ in (b) as \eqref{eq:optimal_rho}. In \Cref{fig:lrr_optimal_c50}, we fix $b = 4$, $N = 400$ but vary $a$ ($x$-axis) from $1$ to $10.5$. When $I(a_{\tau}, b_{\tau}, c_{\tau}) < 1$, the performance difference between the choices of $\rho$ are presented. From simulations, the average mismatch ratio is closer to the predicted lower bound (red curve) when the optimal self-loop is added.

\subsection{Gradient-based training on GCN}
The efficacy of $\widehat\by_{\mathrm{GCN}}$ is presented in Figures \ref{fig:GCN_optimal_c50} and \ref{fig:GCN_exact_c50}. Similarly, we fix $N = 400$, $\tau = 0.25$ and $c_{\tau} = 0.5$, but vary $a$ ($y$-axis) and  $b$ ($x$-axis) from $1$ to $9$ in \Cref{fig:GCN_exact_c50}. For each $(a_{\tau}, b_{\tau}, c_{\tau})$, $10$ independent trials are performed. We plot the performance when adding self-loops to the graph data, where we take $\rho = 0$ in (a) but $\rho = 2c_{\tau} q_m \log(a_{\tau}/b_{\tau})$ in (b) as \eqref{eq:optimal_rho}. In Figure \ref{fig:GCN_optimal_c50}, we fix $b = 4$, $c_{\tau} = 0.5$, $N = 400$ but vary $a$ ($x$-axis) from $1$ to $9$. The performance difference between the choices of $\rho$ when $I(a_{\tau}, b_{\tau}, c_{\tau}) < 1$ are presented. From the simulations, the average mismatch ratio is closer to the predicted bound (red curve) when the optimal self-loop is added.

\section{Discussion and conclusion}\label{sec:conclusions}
Our research delves into the precise recovery threshold in semi-supervised learning on the CSBM. We present various strategies for achieving exact recovery, including the spectral method, linear ridge regression applied to linear GCNs, and gradient-based training techniques for GCNs.

Firstly, as shown in $\ell^*$ and $\widehat\kappa_{\ell^*}$ defined in \eqref{eqn:hatkappa_lstar}, all of our methods cover Erd\H{o}s-R\'{e}nyi graph ($a=b$),  homophilic graphs ($a>b$) and heterophilic graphs $(a<b)$. When $a=b$, we can only utilize the node feature from GMM for classification, which returns to the semisupervised learning on GMM \cite{lelarge2019asymptotic,oymak2021theoretical,nguyen2023asymptotic}. For heterophilic graphs with $a<b$, the optimal self-loop strength $\rho$ defined in \eqref{eq:optimal_rho} is negative, which validates the observation in Figure 5 of \cite{shi2024homophily}.

 Furthermore, for each method, we establish precise asymptotic lower bounds that depend on the sparsity of the SBM and the SNR in GMMs. In many instances, these bounds are optimal, compared with the IT bound. Notably, our findings support the notion that GCNs, when equipped with certain gradient-based training protocols, can flawlessly recover all unlabeled vertices provided the SNR exceeds the IT bound. This finding underscores the effectiveness of GCNs in addressing classification problems within CSBM settings. 
 
 For future research endeavors, one can explore the precise recovery rates for more complex and non-linear graph models, such as XOR-SBM and random geometric Gaussian graphs. Moreover, it is also interesting to illuminate the process of feature learning in GCNs and identify the optimal GCN architectures that mitigate over-smoothing and adhere to information theory constraints.

\section*{Acknowledgements}

H.X.W. and Z.W. acknowledge the support from NSF DMS-2154099. Z.W. is also partially supported by NSF DMS-2055340.

\bibliographystyle{abbrvnat}
\bibliography{bibliography.bib}

\begin{thebibliography}{50}
\providecommand{\natexlab}[1]{#1}
\providecommand{\url}[1]{\texttt{#1}}
\expandafter\ifx\csname urlstyle\endcsname\relax
  \providecommand{\doi}[1]{doi: #1}\else
  \providecommand{\doi}{doi: \begingroup \urlstyle{rm}\Url}\fi

\bibitem[Abbe(2018)]{abbe2018community}
E.~Abbe.
\newblock Community detection and stochastic block models: Recent developments.
\newblock \emph{Journal of Machine Learning Research}, 18\penalty0 (177):\penalty0 1--86, 2018.
\newblock URL \url{http://jmlr.org/papers/v18/16-480.html}.

\bibitem[Abbe et~al.(2015)Abbe, Bandeira, and Hall]{abbe2015exact}
E.~Abbe, A.~S. Bandeira, and G.~Hall.
\newblock Exact recovery in the stochastic block model.
\newblock \emph{IEEE Transactions on information theory}, 62\penalty0 (1):\penalty0 471--487, 2015.

\bibitem[Abbe et~al.(2020)Abbe, Fan, Wang, and Zhong]{abbe2020entrywise}
E.~Abbe, J.~Fan, K.~Wang, and Y.~Zhong.
\newblock Entrywise eigenvector analysis of random matrices with low expected rank.
\newblock \emph{Annals of statistics}, 48\penalty0 (3):\penalty0 1452, 2020.

\bibitem[Abbe et~al.(2022)Abbe, Fan, and Wang]{abbe2022lp}
E.~Abbe, J.~Fan, and K.~Wang.
\newblock An $\ell_p$ theory of pca and spectral clustering.
\newblock \emph{The Annals of Statistics}, 50\penalty0 (4):\penalty0 2359--2385, 2022.

\bibitem[Alt et~al.(2021)Alt, Ducatez, and Knowles]{alt2021extremal}
J.~Alt, R.~Ducatez, and A.~Knowles.
\newblock Extremal eigenvalues of critical {E}rd{\H{o}}s--{R}{\'e}nyi graphs.
\newblock \emph{The Annals of Probability}, 49\penalty0 (3):\penalty0 1347--1401, 2021.

\bibitem[Aminian et~al.(2022)Aminian, Abroshan, Khalili, Toni, and Rodrigues]{aminian2022information}
G.~Aminian, M.~Abroshan, M.~M. Khalili, L.~Toni, and M.~Rodrigues.
\newblock An information-theoretical approach to semi-supervised learning under covariate-shift.
\newblock In \emph{International Conference on Artificial Intelligence and Statistics}, pages 7433--7449. PMLR, 2022.

\bibitem[Azriel et~al.(2022)Azriel, Brown, Sklar, Berk, Buja, and Zhao]{azriel2022semi}
D.~Azriel, L.~D. Brown, M.~Sklar, R.~Berk, A.~Buja, and L.~Zhao.
\newblock Semi-supervised linear regression.
\newblock \emph{Journal of the American Statistical Association}, 117\penalty0 (540):\penalty0 2238--2251, 2022.

\bibitem[Ba et~al.(2022)Ba, Erdogdu, Suzuki, Wang, Wu, and Yang]{ba2022high}
J.~Ba, M.~A. Erdogdu, T.~Suzuki, Z.~Wang, D.~Wu, and G.~Yang.
\newblock High-dimensional asymptotics of feature learning: How one gradient step improves the representation.
\newblock \emph{Advances in Neural Information Processing Systems}, 35:\penalty0 37932--37946, 2022.

\bibitem[Baranwal et~al.(2021)Baranwal, Fountoulakis, and Jagannath]{baranwal2021graph}
A.~Baranwal, K.~Fountoulakis, and A.~Jagannath.
\newblock Graph convolution for semi-supervised classification: Improved linear separability and out-of-distribution generalization.
\newblock In M.~Meila and T.~Zhang, editors, \emph{Proceedings of the 38th International Conference on Machine Learning}, volume 139 of \emph{Proceedings of Machine Learning Research}, pages 684--693. PMLR, 18--24 Jul 2021.
\newblock URL \url{https://proceedings.mlr.press/v139/baranwal21a.html}.

\bibitem[Baranwal et~al.(2023{\natexlab{a}})Baranwal, Fountoulakis, and Jagannath]{baranwal2022effects}
A.~Baranwal, K.~Fountoulakis, and A.~Jagannath.
\newblock Effects of graph convolutions in multi-layer networks.
\newblock In \emph{The Eleventh International Conference on Learning Representations}, 2023{\natexlab{a}}.
\newblock URL \url{https://openreview.net/forum?id=P-73JPgRs0R}.

\bibitem[Baranwal et~al.(2023{\natexlab{b}})Baranwal, Fountoulakis, and Jagannath]{baranwal2023optimality}
A.~Baranwal, K.~Fountoulakis, and A.~Jagannath.
\newblock Optimality of message-passing architectures for sparse graphs.
\newblock In \emph{Thirty-seventh Conference on Neural Information Processing Systems}, 2023{\natexlab{b}}.
\newblock URL \url{https://openreview.net/forum?id=d1knqWjmNt}.

\bibitem[Belkin et~al.(2004)Belkin, Matveeva, and Niyogi]{belkin2004regularization}
M.~Belkin, I.~Matveeva, and P.~Niyogi.
\newblock Regularization and semi-supervised learning on large graphs.
\newblock In \emph{Learning Theory: 17th Annual Conference on Learning Theory, COLT 2004, Banff, Canada, July 1-4, 2004. Proceedings 17}, pages 624--638. Springer, 2004.

\bibitem[Bojchevski and Günnemann(2018)]{bojchevski2018deep}
A.~Bojchevski and S.~Günnemann.
\newblock Deep gaussian embedding of graphs: Unsupervised inductive learning via ranking.
\newblock In \emph{International Conference on Learning Representations}, 2018.
\newblock URL \url{https://openreview.net/forum?id=r1ZdKJ-0W}.

\bibitem[Bruna and Li(2017)]{bruna2017community}
J.~Bruna and X.~Li.
\newblock Community detection with graph neural networks.
\newblock \emph{stat}, 1050:\penalty0 27, 2017.

\bibitem[Chakrabortty and Cai(2018)]{chakrabortty2018efficient}
A.~Chakrabortty and T.~Cai.
\newblock Efficient and adaptive linear regression in semi-supervised settings.
\newblock \emph{The Annals of Statistics}, pages 1541--1572, 2018.

\bibitem[Chen et~al.(2019)Chen, Li, and Bruna]{chen2018supervised}
Z.~Chen, L.~Li, and J.~Bruna.
\newblock Supervised community detection with line graph neural networks.
\newblock In \emph{International Conference on Learning Representations}, 2019.
\newblock URL \url{https://openreview.net/forum?id=H1g0Z3A9Fm}.

\bibitem[Damian et~al.(2022)Damian, Lee, and Soltanolkotabi]{damian2022neural}
A.~Damian, J.~Lee, and M.~Soltanolkotabi.
\newblock Neural networks can learn representations with gradient descent.
\newblock In \emph{Conference on Learning Theory}, pages 5413--5452. PMLR, 2022.

\bibitem[Deshpande et~al.(2018)Deshpande, Sen, Montanari, and Mossel]{deshpande2018contextual}
Y.~Deshpande, S.~Sen, A.~Montanari, and E.~Mossel.
\newblock Contextual stochastic block models.
\newblock In S.~Bengio, H.~Wallach, H.~Larochelle, K.~Grauman, N.~Cesa-Bianchi, and R.~Garnett, editors, \emph{Advances in Neural Information Processing Systems}, volume~31. Curran Associates, Inc., 2018.
\newblock URL \url{https://proceedings.neurips.cc/paper_files/paper/2018/file/08fc80de8121419136e443a70489c123-Paper.pdf}.

\bibitem[Dumitriu and Wang(2023)]{dumitriu2023exact}
I.~Dumitriu and H.-X. Wang.
\newblock Optimal and exact recovery on general non-uniform hypergraph stochastic block model.
\newblock \emph{arXiv preprint arXiv:2304.13139}, 2023.

\bibitem[Duranthon and Zdeborov{\'a}(2023)]{duranthon2023optimal}
O.~Duranthon and L.~Zdeborov{\'a}.
\newblock Optimal inference in contextual stochastic block models.
\newblock \emph{arXiv preprint arXiv:2306.07948}, 2023.

\bibitem[Duranthon and Zdeborov{\'a}(2024)]{duranthon2024asymptotic}
O.~Duranthon and L.~Zdeborov{\'a}.
\newblock Asymptotic generalization error of a single-layer graph convolutional network.
\newblock \emph{arXiv preprint arXiv:2402.03818}, 2024.

\bibitem[Esser et~al.(2021)Esser, Chennuru~Vankadara, and Ghoshdastidar]{esser2021learning}
P.~Esser, L.~Chennuru~Vankadara, and D.~Ghoshdastidar.
\newblock Learning theory can (sometimes) explain generalisation in graph neural networks.
\newblock \emph{Advances in Neural Information Processing Systems}, 34:\penalty0 27043--27056, 2021.

\bibitem[Feige and Ofek(2005)]{feige2005spectral}
U.~Feige and E.~O. Ofek.
\newblock Spectral techniques applied to sparse random graphs.
\newblock \emph{Random Structures \& Algorithms}, 27, 2005.
\newblock URL \url{https://onlinelibrary.wiley.com/doi/abs/10.1002/rsa.20089}.

\bibitem[Fountoulakis et~al.(2022)Fountoulakis, He, Lattanzi, Perozzi, Tsitsulin, and Yang]{fountoulakis2022classification}
K.~Fountoulakis, D.~He, S.~Lattanzi, B.~Perozzi, A.~Tsitsulin, and S.~Yang.
\newblock On classification thresholds for graph attention with edge features.
\newblock \emph{arXiv preprint arXiv:2210.10014}, 2022.

\bibitem[Fountoulakis et~al.(2023)Fountoulakis, Levi, Yang, Baranwal, and Jagannath]{fountoulakis2023graph}
K.~Fountoulakis, A.~Levi, S.~Yang, A.~Baranwal, and A.~Jagannath.
\newblock Graph attention retrospective.
\newblock \emph{Journal of Machine Learning Research}, 24\penalty0 (246):\penalty0 1--52, 2023.

\bibitem[He et~al.(2022)He, Yan, and Tan]{he2022information}
H.~He, H.~Yan, and V.~Y. Tan.
\newblock Information-theoretic characterization of the generalization error for iterative semi-supervised learning.
\newblock \emph{The Journal of Machine Learning Research}, 23\penalty0 (1):\penalty0 13041--13092, 2022.

\bibitem[He et~al.(2020)He, Deng, Wang, Li, Zhang, and Wang]{he2020lightgcn}
X.~He, K.~Deng, X.~Wang, Y.~Li, Y.~Zhang, and M.~Wang.
\newblock Lightgcn: Simplifying and powering graph convolution network for recommendation.
\newblock In \emph{Proceedings of the 43rd International ACM SIGIR conference on research and development in Information Retrieval}, pages 639--648, 2020.

\bibitem[Huang et~al.(2023)Huang, Cao, Wang, Cao, and Suzuki]{huang2023graph}
W.~Huang, Y.~Cao, H.~Wang, X.~Cao, and T.~Suzuki.
\newblock Graph neural networks provably benefit from structural information: A feature learning perspective.
\newblock \emph{arXiv preprint arXiv:2306.13926}, 2023.

\bibitem[Kim et~al.(2018)Kim, Bandeira, and Goemans]{kim2018stochastic}
C.~Kim, A.~S. Bandeira, and M.~X. Goemans.
\newblock Stochastic block model for hypergraphs: Statistical limits and a semidefinite programming approach.
\newblock \emph{arXiv preprint arXiv:1807.02884}, 2018.

\bibitem[Kipf and Welling(2017)]{kipf2017semisupervised}
T.~N. Kipf and M.~Welling.
\newblock Semi-supervised classification with graph convolutional networks.
\newblock In \emph{International Conference on Learning Representations}, 2017.
\newblock URL \url{https://openreview.net/forum?id=SJU4ayYgl}.

\bibitem[Lampert and Scholtes(2023)]{lampert2023self}
M.~Lampert and I.~Scholtes.
\newblock The self-loop paradox: Investigating the impact of self-loops on graph neural networks.
\newblock \emph{arXiv preprint arXiv:2312.01721}, 2023.

\bibitem[Lelarge and Miolane(2019)]{lelarge2019asymptotic}
M.~Lelarge and L.~Miolane.
\newblock Asymptotic bayes risk for gaussian mixture in a semi-supervised setting.
\newblock In \emph{2019 IEEE 8th International Workshop on Computational Advances in Multi-Sensor Adaptive Processing (CAMSAP)}, pages 639--643. IEEE, 2019.

\bibitem[Lu and Sen(2023)]{lu2023contextual}
C.~Lu and S.~Sen.
\newblock Contextual stochastic block model: Sharp thresholds and contiguity.
\newblock \emph{Journal of Machine Learning Research}, 24\penalty0 (54):\penalty0 1--34, 2023.

\bibitem[Lu(2022)]{lu2022learning}
W.~Lu.
\newblock {LEARNING} {GUARANTEES} {FOR} {GRAPH} {CONVOLUTIONAL} {NETWORKS} {ON} {THE} {STOCHASTIC} {BLOCK} {MODEL}.
\newblock In \emph{International Conference on Learning Representations}, 2022.
\newblock URL \url{https://openreview.net/forum?id=dpXL6lz4mOQ}.

\bibitem[Ma et~al.(2022)Ma, Liu, Shah, and Tang]{ma2022is}
Y.~Ma, X.~Liu, N.~Shah, and J.~Tang.
\newblock Is homophily a necessity for graph neural networks?
\newblock In \emph{International Conference on Learning Representations}, 2022.
\newblock URL \url{https://openreview.net/forum?id=ucASPPD9GKN}.

\bibitem[Nguyen and Couillet(2023)]{nguyen2023asymptotic}
M.-T. Nguyen and R.~Couillet.
\newblock Asymptotic bayes risk of semi-supervised multitask learning on gaussian mixture.
\newblock In \emph{International Conference on Artificial Intelligence and Statistics}, pages 5063--5078. PMLR, 2023.

\bibitem[Oono and Suzuki(2019)]{oono2019graph}
K.~Oono and T.~Suzuki.
\newblock Graph neural networks exponentially lose expressive power for node classification.
\newblock \emph{arXiv preprint arXiv:1905.10947}, 2019.

\bibitem[Oymak and Cihad~Gulcu(2021)]{oymak2021theoretical}
S.~Oymak and T.~Cihad~Gulcu.
\newblock A theoretical characterization of semi-supervised learning with self-training for gaussian mixture models.
\newblock In A.~Banerjee and K.~Fukumizu, editors, \emph{Proceedings of The 24th International Conference on Artificial Intelligence and Statistics}, volume 130 of \emph{Proceedings of Machine Learning Research}, pages 3601--3609. PMLR, 13--15 Apr 2021.
\newblock URL \url{https://proceedings.mlr.press/v130/oymak21a.html}.

\bibitem[Ryan and Culp(2015)]{ryan2015semi}
K.~J. Ryan and M.~V. Culp.
\newblock On semi-supervised linear regression in covariate shift problems.
\newblock \emph{The Journal of Machine Learning Research}, 16\penalty0 (1):\penalty0 3183--3217, 2015.

\bibitem[Shchur et~al.(2018)Shchur, Mumme, Bojchevski, and G{\"u}nnemann]{shchur2018pitfalls}
O.~Shchur, M.~Mumme, A.~Bojchevski, and S.~G{\"u}nnemann.
\newblock Pitfalls of graph neural network evaluation.
\newblock \emph{arXiv preprint arXiv:1811.05868}, 2018.

\bibitem[Shi et~al.(2024)Shi, Pan, Hu, and Dokmani{\'c}]{shi2024homophily}
C.~Shi, L.~Pan, H.~Hu, and I.~Dokmani{\'c}.
\newblock Homophily modulates double descent generalization in graph convolution networks.
\newblock \emph{Proceedings of the National Academy of Sciences}, 121\penalty0 (8):\penalty0 e2309504121, 2024.

\bibitem[Tang and Liu(2023)]{tang2023generalization}
H.~Tang and Y.~Liu.
\newblock Towards understanding generalization of graph neural networks.
\newblock In A.~Krause, E.~Brunskill, K.~Cho, B.~Engelhardt, S.~Sabato, and J.~Scarlett, editors, \emph{Proceedings of the 40th International Conference on Machine Learning}, volume 202 of \emph{Proceedings of Machine Learning Research}, pages 33674--33719. PMLR, 23--29 Jul 2023.
\newblock URL \url{https://proceedings.mlr.press/v202/tang23f.html}.

\bibitem[Tony~Cai and Guo(2020)]{tony2020semisupervised}
T.~Tony~Cai and Z.~Guo.
\newblock Semisupervised inference for explained variance in high dimensional linear regression and its applications.
\newblock \emph{Journal of the Royal Statistical Society Series B: Statistical Methodology}, 82\penalty0 (2):\penalty0 391--419, 2020.

\bibitem[Vershynin(2018)]{vershynin2018high}
R.~Vershynin.
\newblock \emph{High-Dimensional Probability: An Introduction with Applications in Data Science}.
\newblock Cambridge Series in Statistical and Probabilistic Mathematics. Cambridge University Press, 2018.

\bibitem[Wang(2023)]{wang2023strong}
H.-X. Wang.
\newblock Information-theoretic limits and strong consistency on binary non-uniform hypergraph stochastic block models.
\newblock \emph{arXiv preprint arXiv:2306.06845}, 2023.

\bibitem[Wei et~al.(2022)Wei, Yin, Jia, Benson, and Li]{wei2022understanding}
R.~Wei, H.~Yin, J.~Jia, A.~R. Benson, and P.~Li.
\newblock Understanding non-linearity in graph neural networks from the bayesian-inference perspective.
\newblock \emph{Advances in Neural Information Processing Systems}, 35:\penalty0 34024--34038, 2022.

\bibitem[Wu et~al.(2019)Wu, Souza, Zhang, Fifty, Yu, and Weinberger]{wu2019simplifying}
F.~Wu, A.~Souza, T.~Zhang, C.~Fifty, T.~Yu, and K.~Weinberger.
\newblock Simplifying graph convolutional networks.
\newblock In \emph{International conference on machine learning}, pages 6861--6871. PMLR, 2019.

\bibitem[Wu et~al.(2022)Wu, Chen, Wang, and Jadbabaie]{wu2022non}
X.~Wu, Z.~Chen, W.~Wang, and A.~Jadbabaie.
\newblock A non-asymptotic analysis of oversmoothing in graph neural networks.
\newblock \emph{arXiv preprint arXiv:2212.10701}, 2022.

\bibitem[Yang et~al.(2016)Yang, Cohen, and Salakhudinov]{yang2016revisiting}
Z.~Yang, W.~Cohen, and R.~Salakhudinov.
\newblock Revisiting semi-supervised learning with graph embeddings.
\newblock In \emph{International conference on machine learning}, pages 40--48. PMLR, 2016.

\bibitem[Zhu and Koniusz(2021)]{zhu2021simple}
H.~Zhu and P.~Koniusz.
\newblock Simple spectral graph convolution.
\newblock In \emph{International Conference on Learning Representations}, 2021.
\newblock URL \url{https://openreview.net/forum?id=CYO5T-YjWZV}.

\end{thebibliography}
\appendix
\addcontentsline{toc}{section}{Appendices}
\newpage

\section{Information-theoretic limits}\label{sec:ITLowerBoundsCSBM_proof}
In this section, we will provide the proofs for \Cref{thm:impossibility_CSBM} and \Cref{thm:ITlowerbounds_CSBM}.

\subsection{Impossibility for exact recovery}
The proof sketch of \Cref{thm:impossibility_CSBM} is presented in this section, with some proofs of Lemmas deferred.

Let $\by\in \{\pm 1\}^{N}$ denote the true label vector with $\by = [\by_{\sL}^{\sT}, \by_{\sU}^{\sT}]^{\sT}$, where $\by_{\sL}$ and $\by_{\sU}$ denote the observed and uncovered label vector respectively. Assume $(\bA, \bX) \sim \CSBM (\by, \bmu, \alpha, \beta, \theta)$ as in model \ref{def:SemiCSBM}, and the access to $\bA, \bX, \by_{\sL}$ are provided. Let $\widehat{\by}_{\sU} \in \{\pm 1\}^{m}$ denote an estimator of $\by_{\sU}$ obtained from algorithm. The probability that $\widehat{\by}_{\sU}$ fails recovering every entry of $\by_{\sU}$ is
\begin{align}
  &\,\P_{\mathrm{fail}}\coloneqq \P(\widehat{\by}_{\sU} \neq \pm \by_{\sU}) = \sum_{\bA, \bX, \by_{\sL}}[1 -\P(\widehat{\by}_{\sU} = \pm \by_{\sU} | \bA, \bX, \by_{\sL}) ] \cdot \P(\bA, \bX,  \by_{\sL}), \label{eqn:probFailureExactRecovery}
\end{align}
where the \textit{Maximum A Posteriori} (MAP) estimator achieves its minimum. Since the prior distribution of $\by$ is uniform sampled in Definition~\ref{def:SemiCSBM}, the discussion on the ideal estimator can be transferred to \textit{Maximum Likelihood Estimation} (MLE),
\begin{align}
    \widehat{\by}_{\textnormal{MLE}} \coloneqq \underset{\bz \in \{\pm 1\}^{m}, \ones^{\top}\bz = 0}{\arg\max} \P(\bA |\by_{\sL}, \by_{\sU} = \bz) \cdot \P(\bX|\by_{\sL}, \by_{\sU} = \bz).\label{eqn:MLEestimator}
\end{align}
Furthermore, \Cref{lem:MAPMLEMax} shows the function that MLE is maximizing over $\bz \in \{\pm 1\}^{m}$
\begin{align}
    f(\bz) \coloneqq \log \P(\bA |\by_{\sL}, \by_{\sU} = \bz) + \log \P(\bX|\by_{\sL}, \by_{\sU} = \bz). \label{eqn:MLEMax}
\end{align}
From the discussion above, MLE is equivalent to the best estimator MAP. No algorithm would be able to assign all labels correctly if MLE fails. In the view of \eqref{eqn:MLEMax}, the failure of MLE indicates that some configuration $\bsigma \in \{\pm 1\}^{m}$ other than the true $\by_{\sU}$ achieves its maximum, and MLE prefers $\bsigma$ other than $\by_{\sU}$. 

To establish the necessity, we explicitly construct some $\bsigma \in \{\pm 1\}^{m}$ with $\ones^{\sT}\bsigma = 0$ such that $\bsigma \neq \by_{\sU}$ but $f(\bsigma) \geq f(\by_{\sU})$ when below the threshold, i.e., $I_{\tau}(a, b, c) < 1$. An example of such $\bsigma$ can be constructed as follows. Pick $u\in \cV_{\sU, +}$ and $v\in \cV_{\sU, -}$ where $\cV_{\sU, \pm} = \cV_{\sU} \cap \cV_{\pm}$, and switch the labels of $u$ and $v$ in $\by_{\sU}$ but keep all the others. \Cref{lem:WmuLDP} characterizes the scenarios of failing exact recovery in terms of $u$ and $v$.

\begin{lemma}\label{lem:WmuLDP}
Given some subset $\cS \subset \cV = [N]$, for vertex $u\in \sU$, define the following random variable
\begin{align}
        W_{m,u}(\cS) \coloneqq y_u \cdot \bigg( \log (a/b)\cdot \sum_{j\in \cS} A_{uj}y_j + \frac{2}{N + d/\theta^2}\sum_{j\in \cS} \langle \bx_u,\bx_j\rangle y_j \bigg).\label{eqn:Wmu}
    \end{align}
   Denote by $W_{m, u} \coloneqq W_{m,u}([N] \setminus \{u\})$ for any $u \in \cV_{\sU}$. Define the rate function
\begin{align}
    I(t,a_{\tau},b_{\tau},c_\tau)
    \coloneqq\frac{1}{2}\Big(a_{\tau} - a_{\tau} \Big(\frac{a_{\tau}}{b_{\tau}}\Big)^t + b_{\tau} - b_{\tau} \Big(\frac{b_{\tau}}{a_{\tau}} \Big)^t \Big) -2c_\tau(t+t^2). \label{eqn:rateFunctiontau}
\end{align}
Then, it supreme over $t$ is attained at $t^{\star} = -1/2$,
\begin{align}
    \sup_{t\in\R} I(t,a_{\tau},b_{\tau},c_\tau) = I(-1/2,a_{\tau},b_{\tau},c_{\tau}) = \frac{1}{2} \Big( (\sqrt{a_\tau} - \sqrt{b_\tau})^2 + c_\tau \Big) \eqqcolon I(a_\tau,b_\tau,c_\tau),
\end{align}
where the last equality holds as in \eqref{eqn:rate_I_abc_tau}.
   \begin{enumerate}[topsep=0pt,itemsep=-1ex,partopsep=1ex,parsep=1ex,label=(\alph*)]
       \item For any $\eps<\frac{a-b}{2(1-\tau)}\log (a/b)+2c_\tau$ and $\delta>0$, there exists some sufficiently large $m_0>0$, such that for $I(t, a_{\tau}, b_{\tau}, c_{\tau})$ in \eqref{eqn:rateFunctiontau}, the following holds for any $m\ge m_0$
        \begin{align}
            \P(W_{m,u}\le \eps q_m) = (1 + o(1))\cdot \exp{ \Big( -q_m\cdot\big(-\delta+\sup_{t\in\R}\{\eps t+I(t,a_{\tau},b_{\tau},c_\tau)\} \big) \Big) }.
        \end{align}
       \item  For the pair $u\in \cV_{\sU, +}$ and $v\in \cV_{\sU, -}$, the event $\{ W_{m,u} \leq 0\} \cap \{ W_{m,v} \leq 0\}$ implies $f(\by_{\sU}) \leq f(\bsigma)$ with probability at least $1 - e^{-q_m}$.
   \end{enumerate}
\end{lemma}

However, for any pair $u\in \cV_{\sU, +}$, $v\in \cV_{\sU, -}$, the variables $W_{m, u}$ and $W_{m, v}$ are not independent due to the existence of common random edges. To get rid of the dependency, let $\cU$ be a subset of $\cV_{\sU}$ with cardinality $|\cU| = \delta m$ where $\delta = \log^{-3}(m)$, such that $|\cU \cap \cV_{\sU, +}| = |\cU \cap \cV_{\sU, -}| = \delta m/2$. Define the following random variables
\begin{align}
    U_{m, u} \coloneqq W_{m,u}([N] \setminus \cU ),\quad J_{m, u} \coloneqq W_{m,u}(\cU \setminus \{u\} ), \quad J_{m} \coloneqq &\, \max_{u\in \cV_{\sU}} J_{m, u} \label{eqn:UmuJmuJm}
\end{align}
Obviously, for some $\zeta_m >0$,  $\{U_{m, u} \leq -\zeta_{m} q_m \} \cap \{J_m \leq \zeta_{m} q_m\}$ implies $\{W_{m, u} \leq 0\}$ since $W_{m, u} = U_{m, u} + J_{m, u}$. Furthermore, $\{U_{m, u} \leq -\zeta_{m} q_m \}$ does not reply on $\{J_m \leq \zeta_{m} q_m\}$ since $J_m$ is independent to any vertex in $\cU$. Also, $\{ U_{m, u}\}_{u \in \cV_{\sU, +} \cap \cU}$ is a set of independent random variables since no overlap edges. Thus the failure probability can be lower bounded by
\begin{align}
     \P_{\mathrm{fail}} \geq &\, \P(\exists u \in \cV_{\sU, + },\, v \in \cV_{\sU, -} \textnormal{  s.t. } f(\by_{\sU}) \leq f(\bsigma))\\
    \geq &\, \P\Big(\cup_{u \in \cV_{\sU, +}} \{ W_{m,u} \leq 0 \} \bigcap \cup_{v \in \cV_{\sU, -}} \{ W_{m,v}\leq 0 \} \Big) \geq \P\Big(\cup_{u \in \cV_{\sU, +} \cap \cU} \{ W_{m,u} \leq 0 \} \bigcap \cup_{v \in \cV_{\sU, -} \cap \cU } \{ W_{m,v}\leq 0 \} \Big) \\
    \geq &\,  \P\Big(\cup_{u \in \cV_{\sU, +} \cap \cU} \{ U_{m,u} \leq -\zeta_{m} q_m \} \bigcap \cup_{v \in \cV_{\sU, -} \cap \cU } \{ U_{m,v}\leq -\zeta_{m} q_m \} \Big| \{J_m \leq \zeta_{m} q_m\}  \Big) \cdot \P(J_m \leq \zeta_{m} q_m) \\
    \geq &\,  \P\Big(\cup_{u \in \cV_{\sU, +} \cap \cU} \{ U_{m,u} \leq -\zeta_{m} q_m \} \Big| \{J_m \leq \zeta_{m} q_m\} \Big)\\
    &\,\cdot \P \Big( \cup_{v \in \cV_{\sU, -} \cap \cU } \{ U_{m,v}\leq -\zeta_{m} q_m \} \Big| \{J_m \leq \zeta_{m} q_m\}  \Big) \cdot \P(J_m \leq \zeta_{m} q_m)  \\
    = &\,  \P\Big(\cup_{u \in \cV_{\sU, +} \cap \cU} \{ U_{m,u} \leq -\zeta_{m} q_m \} \Big)\cdot \P \Big( \cup_{v \in \cV_{\sU, -} \cap \cU } \{ U_{m,v}\leq -\zeta_{m} q_m \} \Big) \cdot \P(J_m \leq \zeta_{m} q_m).
\end{align}

\begin{lemma}\label{lem:lowboundsJmuUmu}
For $\zeta_{m} = (\log\log m)^{-1}$ and $q_m = \log(m)$ and some constant $\widetilde{\delta}> 0$, the following holds
\begin{align}
    \P(J_m \leq \zeta_{m} q_m) \geq 1 - \log^{-3}(m) \cdot m^{-1 + o(1)}, \quad \P\Big(\cup_{u \in \cV_{\sU, +} \cap \cU} \{ U_{m,u} \leq -\zeta_{m} q_m \} \Big) \geq 1 - \exp\Big( -\frac{m^{1 - I(a_{\tau}, b_{\tau}, c_{\tau}) + \widetilde{\delta}}}{2\log^3(m)}\Big).
\end{align}
\end{lemma}
With the lower bounds of the three components obtained in \Cref{lem:lowboundsJmuUmu}, while $I(a_{\tau}, b_{\tau}, c_{\tau}) = 1 - \epsilon < 1$ for some $\epsilon > 0$ and $\widetilde{\delta}> 0$, one has
\begin{align}
    \P_{\mathrm{fail}} \geq &\, \Big[ 1 - \exp\Big( -\frac{m^{1 - I(a_{\tau}, b_{\tau}, c_{\tau}) + \widetilde{\delta}}}{2\log^3(m)}\Big) \Big]^2 \cdot \Big(1 - \frac{m^{-1 + o(1)}}{\log^{3}(m)} \Big) \\
    \geq &\, 1 - 2\exp\Big( -\frac{m^{\epsilon + \widetilde{\delta}}}{2\log^3(m)}\Big) - \frac{m^{-1 + o(1)}}{\log^{3}(m)} - \exp\Big( -\frac{m^{\epsilon + \widetilde{\delta}}}{2\log^3(m)}\Big) \cdot \frac{m^{-1 + o(1)}}{\log^{3}(m)} \,\, \overset{m \to \infty }{\rightarrow} 1.
\end{align}
Therefore, the with probability tending to $1$, the best estimator MLE (MAP) fails exact recovery, hence no other algorithm could succeed.

\subsection{Information-theoretic lower bounds}
\begin{proof}[Proof of \Cref{thm:ITlowerbounds_CSBM}]
For each node $i\in \cV_{\sU}$, denote $f(\cdot| \widetilde{\bA}, \widetilde{\bX}, \widetilde{\by}_{\sL}, \widetilde{\by}_{\sU, -i}) = \P( y_i = \cdot| \bA = \widetilde{\bA}, \bX = \widetilde{\bX}, \by_{\sL} = \widetilde{\by}_{\sL}, \by_{\sU, -i} = \widetilde{\by}_{\sU, -i})$. Due to the symmetry of the problem, vertices are interchangeable if in the same community, then it suffices to consider the following event
\begin{align}
    \cA = \big \{ f(y_1| \widetilde{\bA}, \widetilde{\bX}, \widetilde{\by}_{\sL}, \widetilde{\by}_{\sU, -1}) < f(-y_1| \widetilde{\bA}, \widetilde{\bX}, \widetilde{\by}_{\sL}, \widetilde{\by}_{\sU, -1}) \big \}.
\end{align}
By {Lemma F.3 in \cite{abbe2022lp}} and symmetry between vertices, for any sequence of estimators $\widehat{\by}_{\sU}$, the following holds
\begin{align}
    \E \psi_m (\widehat{\by}_{\sU}, \by_{\sU}) \geq \frac{m-1}{3m-1} \cdot \P( \cA ).
\end{align}
Recall the definition of $W_{m, u}$ in \eqref{eqn:Wmu}. We denote $W_{m, 1}$ by taking $u = 1$ and $\sS = [N] \setminus \{u\}$. Define the following two events
\begin{align}
    \cB_{\epsilon} \coloneqq &\, \bigg\{ \bigg| \log\Big( \frac{ f(y_1| \widetilde{\bA}, \widetilde{\bX}, \widetilde{\by}_{\sL}, \widetilde{\by}_{\sU, -1}) }{f(-y_1| \widetilde{\bA}, \widetilde{\bX}, \widetilde{\by}_{\sL}, \widetilde{\by}_{\sU, -1})} \Big) - W_{m, 1}\bigg| < \epsilon q_m\bigg\}, \quad \cC_{\epsilon} = \Big\{ W_{m, 1} \leq -\epsilon q_m \Big\}.
\end{align}
By triangle inequality, $\cB_{\epsilon} \cap \cC_{\epsilon}$ implies $\cA$, thus $\cB_{\epsilon} \cap \cC_{\epsilon} \subset \cA$, and
\begin{align}
    \E \psi_m (\widehat{\by}_{\sU}, \by_{\sU}) \gtrsim \P(\cA) \geq \P(\cB_{\epsilon} \cap \cC_{\epsilon}) \geq \P(\cC_{\epsilon}) - \P(\cB_{\epsilon}^{ \text{c}}).
\end{align}
According to \Cref{lem:optimalApprox}, $\P(\cB_{\epsilon}^{\text{c}}) \ll e^{-q_m}$. Together with the results above, and by \Cref{lem:WmuLDP}, we have
\begin{align}
    \liminf_{m \to \infty} q^{-1}_m \log \E \eta_{m}(\widehat{\by}_{\sU}, \by_{\sU}) \geq -\sup_{t\in \R} \{\eps t + I(a, b, c_{\tau}, \tau)\},
\end{align}
and the proof is finished by taking $\eps \to 0$.
\end{proof}

\subsection{Deferred proofs}
For the sake of convenience, we introduce the following notations for the remaining of this section. For some realization $\bA = \widetilde{\bA}$, $\bX = \widetilde{\bX}$, $\by_{\sL} =\widetilde{\by}_{\sL}\in \{\pm 1\}^{n}$ and $\bmu = \widetilde{\bmu}$, we write
\begin{align}
    &\,\P(\widetilde{\bA}, \widetilde{\bX},  \widetilde{\by}_{\sL}) = \P(\bA = \widetilde{\bA}, \bX = \widetilde{\bX}, \by_{\sL} =\widetilde{\by}_{\sL}), \quad \P(\widetilde{\bA}, \widetilde{\bX}| \widetilde{\by}_{\sL}, \by_{\sU} = \bz ) = \P(\bA = \widetilde{\bA}, \bX = \widetilde{\bX}| \by_{\sL} = \widetilde{\by}_{\sL}, \by_{\sU} = \bz)\\
    &\,\P(\widetilde{\bA}|\widetilde{\by}_{\sL}, \by_{\sU} = \bz) = \P(\bA = \widetilde{\bA}|\by_{\sL} = \widetilde{\by}_{\sL}, \by_{\sU} = \bz),\quad 
    \P(\widetilde{\bX}| \widetilde{\by}_{\sL}, \by_{\sU} = \bz) = \P(\bX = \widetilde{\bX}| \by_{\sL} = \widetilde{\by}_{\sL}, \by_{\sU} = \bz).
\end{align}
\begin{lemma}\label{lem:MAPMLEMax}
The MAP estimator minimizes the $\P_{\textnormal{fail}}$, and MAP is equivalent to the MLE \eqref{eqn:MLEestimator}. The quantity that MLE is maximizing is defined in \eqref{eqn:MLEMax}.
\end{lemma}
\begin{proof}[Proof of \Cref{lem:MAPMLEMax}]
From model \ref{def:SemiCSBM}, independently, $\by_{\sL}$ and $\by_{\sU}$ are uniformly distributed over the spaces $\{\pm 1\}^{m}$ and $\{\pm 1\}^{m}$ respectively, thus the following factorization holds
\begin{align}
    \P( \by_{\sL}, \by_{\sU} = \bz) \coloneqq \P(\by_{\sL} = \widetilde{\by}_{\sL}, \by_{\sU} = \bz) = \P(\by_{\sL} = \widetilde{\by}_{\sL}) \cdot \P(\by_{\sU} = \bz),
\end{align}
which is some constant irrelevant to $\bz$. The first sentence of the Lemma can be established by Bayes Theorem, since
\begin{align}
    \widehat{\by}_{\mathrm{MAP}} =&\, \underset{\bz \in \{\pm 1\}^{m}, \ones^{\top}\bz = 0}{\arg\max} \P(\by_{\sU} = \bz|\bA,\bX, \by_{\sL}) = \underset{\bz \in \{\pm 1\}^{m}, \ones^{\top}\bz = 0}{\arg\max} \frac{\P(\bA, \bX|\by_{\sL}, \by_{\sU} = \bz)\cdot \P(\by_{\sL}, \by_{\sU} = \bz)}{\P(\bA, \bX, \by_{\sL})}\\
    =&\, \underset{\bz \in \{\pm 1\}^{m}, \ones^{\top}\bz = 0}{\arg\max} \P(\bA, \bX |\by_{\sL}, \by_{\sU} = \bz) = \underset{\bz \in \{\pm 1\}^{m}, \ones^{\top}\bz = 0}{\arg\max} \P(\bA |\by_{\sL}, \by_{\sU} = \bz) \cdot \P(\bX|\by_{\sL}, \by_{\sU} = \bz)
\end{align}
where $\P( \by_{\sL}, \by_{\sU} = \bz)$ and $\P(\bA, \bX, \by_{\sL})$ in the first line are factored out since they are irrelevant to $\bz$, and the last equality holds due to the independence between $\bA$ and $\bX$ when given $\by$. For the second sentence of the Lemma, the function $f(\bz)$ could be easily obtained by taking the logarithm of the objectve probability.
\end{proof}

\begin{proof}[Proof of \Cref{lem:WmuLDP} (1)]
By definition of $W_{m,i}$, we have 
\begin{align}
    \E[e^{tW_{m,i}}|y_i]=~&\E\left[\exp{\left(\frac{2t\theta^2}{N\theta^2+d}\sum_{j\in[N]\setminus\{i\}} \langle \bx_i,\bx_j\rangle y_jy_i\right)}\;\middle|\;y_i\right]\\
    ~&\cdot\E\left[\exp{\left(t\log (a/b)\sum_{j\in\cV_{\sU}\setminus\{i\}} A_{ij} y_jy_i\right)}\;\middle|\;y_i\right]\cdot\E\left[\exp{\left(t\log (a/b)\sum_{j\in\cV_{\sL}} A_{ij} y_jy_i\right)}\;\middle|\;y_i\right].
\end{align}
Following the same calculation as Lemma F.2 in \cite{abbe2022lp}, we know that 
\begin{align}
    \log \E\left[\exp{\left(\frac{2t\theta^2}{N\theta^2+d}\sum_{j\in[N]\setminus\{i\}} \langle \bx_i,\bx_j\rangle y_jy_i\right)}\;\middle|\;y_i\right]=~& 2c_\tau(t+t^2)(1+o(1))q_m\\
    \log \E\left[\exp{\left(t\log (a/b)\sum_{j\in\cV_{\sU}\setminus\{i\}} A_{ij} y_jy_i\right)}\;\middle|\;y_i\right]=~& \frac{1}{2}\left(-a+a\left(\frac{a}{b}\right)^t-b+b\left(\frac{b}{a}\right)^t\right)(1+o(1))q_m.
\end{align}
Meanwhile, since
\begin{align}
    \E[e^{-tA_{ij}y_iy_j}|y_i]=1+\frac{1}{2}\big(\alpha(e^t-1)+\beta(e^{-t}-1)\big),
\end{align}
we can get
\begin{align}
    \E\left[\exp{\left(t\log (a/b)\sum_{j\in\cV_{\sL}} A_{ij} y_jy_i\right)}\;\middle|\;y_i\right] =~& n \log \left(1+\frac{1}{2}\big(\alpha(e^t-1)+\beta(e^{-t}-1)\big)\right)\\
    =~& \frac{n}{2m}\left(a\left(\frac{a}{b}\right)^t-a+b\left(\frac{b}{a}\right)^t-b\right)(1+o(1))q_m.
\end{align}
Thus by using $\log(1 + x) = x$ when $x = o(1)$, we obtain
\begin{align}
    &\,\lim_{m \to \infty} q_m^{-1}\log \E[e^{tW_{m,i}}|y_i] = \lim_{m \to \infty} \left( 2c_\tau(t+t^2) +\frac{N}{2m}\left(a\left(\frac{a}{b}\right)^t-a+b\left(\frac{b}{a}\right)^t-b\right)\right)(1+o(1))\\
    =&\, -I(t,a_{\tau},b_{\tau},c_{\tau})(1+o(1)).
\end{align}
The proof is then completed by applying {Lemma H.5 in \cite{abbe2022lp}}.
\end{proof}

\begin{proof}[Proof of \Cref{lem:WmuLDP} (2)]
First, we plug in $\sigma$, $\by_{\sU}$ into \eqref{eqn:MLEMax}, and consider the effect of $u$ and $v$, 
\begin{align}
  f(\by_{\sU}) - f(\bsigma) =&\, \log\P(\bA |\by_{\sL}, \by_{\sU} = \bsigma) - \log\P(\bA |\by_{\sL}, \by_{\sU}) + \log \P(\bX|\by_{\sL}, \by_{\sU} = \bsigma) - \log \P(\bX|\by_{\sL}, \by_{\sU})\\
  =&\, \log \P(\bA |y_{u} = 1, y_v = -1, \by_{\sL}, \by_{\sU \setminus \{ u, v \} }) - \log \P(\bA |y_{u} = -1, y_v = 1, \by_{\sL}, \by_{\sU \setminus \{ u, v \} })\\
   &\, + \log \P(\bX|y_{u} = 1, y_v = -1, \by_{\sL}, \by_{\sU \setminus \{ u, v \} }) - \log \P(\bX|y_{u} = -1, y_v = 1, \by_{\sL}, \by_{\sU \setminus \{ u, v \} } ).
\end{align}
By \Cref{eqn:decoupleuandv}, the term above can be further reformulated as
\begin{align}
     f(\by_{\sU}) - f(\bsigma) = &\, \log \Big( \frac{p_{\bA}(\bA|y_{u}, \by_{-u})}{p_{\bA}(\bA|-y_{u}, \by_{-u})} \Big) + \log \Big( \frac{p_{\bX}(\bX|y_{u}, \by_{-u})}{p_{\bX}(\bX|-y_{u}, \by_{-u})} \Big)\\
      &\, + \log \Big( \frac{p_{\bA}(\bA|y_{v}, \by_{-v})}{p_{\bA}(\bA|-y_{v}, \by_{-v})} \Big) + \log \Big( \frac{p_{\bX}(\bX|y_{v}, \by_{-v})}{p_{\bX}(\bX|-y_{v}, \by_{-v})} \Big).
\end{align}
Note that $y_u^2 = 1$, according to \Cref{lem:F4ABBElp,lem:F5ABBElp}, with probability at least $1 - e^{-q_m}$, we have 
\begin{align}
    &\, \bigg| \log \Big( \frac{p_{\bA}(\bA|y_{u}, \by_{-u})}{p_{\bA}(\bA|-y_{u}, \by_{-u})} \Big) + \log \Big( \frac{p_{\bX}(\bX|y_{u}, \by_{-u})}{p_{\bX}(\bX|-y_{u}, \by_{-u})} \Big) - y_{u}\log\Big( \frac{a}{b}\Big) \sum_{j \neq i} A_{ij} y_j - y_{u}\frac{2\theta^2}{N\theta^2 + d} \sum_{j\neq i} \<\bx_i, \bx_{j}\>y_j \bigg| \\
    \leq &\, \bigg| \log \Big( \frac{p_{\bA}(\bA|y_{u}, \by_{-u})}{p_{\bA}(\bA|-y_{u}, \by_{-u})} \Big) - y_{u}\log\Big( \frac{a}{b}\Big) \sum_{j \neq i} A_{ij} y_j \Bigg| + \Bigg |\log \Big( \frac{p_{\bX}(\bX|y_{u}, \by_{-u})}{p_{\bX}(\bX|-y_{u}, \by_{-u})} \Big) - y_{u}\frac{2\theta^2}{N\theta^2 + d} \sum_{j\neq i} \<\bx_i, \bx_{j}\>y_j \bigg| \\
    \ll &\, q_m.
\end{align}
Consequently by triangle inequality, there exists some large enough constant $c > 0$ such that with probability at least $1 - e^{-cq_m}$,
\begin{align}
    |f(\by_{\sU}) - f(\bsigma) - W_{m, u} - W_{m, v}|/ q_m = o(1),
\end{align}
The proof is then completed.
\end{proof}

\begin{proof}[Proof of \Cref{lem:lowboundsJmuUmu}]
    First, $J_{m} \coloneqq \max_{u\in \cV_{\sU}} J_{m, u}$ in \eqref{eqn:UmuJmuJm}, then it suffices to focus on $\P(J_{m, u} >\zeta_{m} q_m)$, since an argument based on the union bound leads to
\begin{align}
    \P(J_{m} > \zeta_{m} q_m ) = \P(\exists u \in \cV_{\sU} \textnormal{ s.t. } J_{m, u} >\zeta_{m} q_m) \leq \sum_{u\in \cV_{\sU}} \P(J_{m, u} >\zeta_{m} q_m).
\end{align}
We claim the following upper bound with the proof deferred later
\begin{align}
    \E J_{m, u} \leq \exp(q_m \log^{-3}m). \label{eqn:expectationJmu}
\end{align}
Then by Markov inequality and the fact $q_m = \log(m)$, one has $\P(J_{m, u} > \zeta_{m} q_m) \leq \E J_{m, u}/ (\zeta_{m} q_m) \asymp n^{-2 + o(1)}$. Thus by the union bound
\begin{align}
    \P( J_{m} \leq \zeta_{m} q_m) = 1 - \P(J_{m} > \zeta_{m} q_m) \geq  1 - \gamma m \cdot m^{-2 + o(1)} = 1 - \log^{-3}(m) \cdot m^{-1 + o(1)}.
\end{align}

For the second desired inequality, note that the difference between $U_{m, u}$ \eqref{eqn:UmuJmuJm} and $W_{m, u}$ \eqref{eqn:Wmu} is relatively negligible since $|\cU|/m = \log^{-3}(m) = o(1)$, thus $U_{m, u}$ exhibits the same concentration behavior as $W_{m, u}$. One could follow the same calculation as in \Cref{lem:WmuLDP} to figure out that for any $m\ge m_0$ and $\widetilde{\delta} >0$, with $I(t,a_{\tau},b_{\tau},c_\tau)$ defined in \eqref{eqn:rateFunctiontau}, the following holds
\begin{align}
    \P(U_{m, u}\leq -\zeta_{m} q_m) = \exp{ \big(-q_m \cdot(-\widetilde{\delta} +\sup_{t\in\R}\{ -\zeta_{m} t+I(t,a_{\tau},b_{\tau},c_\tau)\}) \big)}.
\end{align}
Note that $\{ U_{m, u}\}_{u \in \cV_{\sU, +} \cap \cU}$ is a set of independent for different since no edge overlap, then one has
\begin{align}
    &\, \P\Big(\cap_{u \in \cV_{\sU, +} \cap \cU} \{ U_{m,u} \leq -\zeta_{m} q_m \} \Big) = \prod_{u \in \cV_{\sU, +}\cap \cU} \P( U_{m,u} > -\zeta_{m} q_m ) \\
    =&\, \Big( 1 - m^{-I(a_{\tau}, b_{\tau}, c_{\tau}) + \widetilde{\delta}}\Big)^{\delta m/2} \leq \exp\Big( -\frac{m^{1 - I(a_{\tau}, b_{\tau}, c_{\tau}) + \widetilde{\delta}}}{2\log^3(m)}\Big),
\end{align}
where the last inequality holds since $1 - x \leq e^{-x}$, and it leads to our desired result
\begin{align}
    \P\Big(\cup_{u \in \cV_{\sU, +} \cap \cU} \{ U_{m,u} \leq -\zeta_{m} q_m \} \Big) = 1 -  \P\Big(\cap_{u \in \cV_{\sU, +} \cap \cU} \{ U_{m,u} \leq -\zeta_{m} q_m \} \Big) \geq 1 - \exp\Big( -\frac{m^{1 - I(a_{\tau}, b_{\tau}, c_{\tau}) + \widetilde{\delta}}}{2\log^3(m)}\Big).
\end{align}

We now establish the proof of \eqref{eqn:expectationJmu}. Recall that $J_{m, u}$ \eqref{eqn:UmuJmuJm} is a summation of independent random variables, where the number of such type random variables is at most $|\cU| = \gamma m \asymp m\log^{-3}(m)$. Denote $\widehat{\bmu}^{(-u)} = \frac{1}{|\cU| - 1} \sum_{j\in \cU\setminus \{u\}} \bx_j y_j$. Recall $\bx_i = \theta y_i \bmu + \bz_i$ with $\|\bmu\|_2 = 1$ in \eqref{eqn:gauss_mixture}, then $y_i \bx_i \sim \Normal (\theta\bmu, \bI_d)$ given $y_i$, and $\sqrt{|\cU| - 1}\,\, \widehat{\bmu}^{(-u)} \sim \Normal(\sqrt{|\cU| - 1}\,\,\theta \bmu, \bI_d)$, while $y_i \bx_i$ and $\sqrt{|\cU| - 1}\,\, \widehat{\bmu}^{(-u)}$ are independent. Following {Lemma H.4 in \cite{abbe2022lp}}, for all $t\in (-\sqrt{|\cU| - 1}, \sqrt{|\cU| - 1})$, one has
\begin{align}
    &\, \log\E\big( \exp(t \langle \bx_u , \widehat{\bmu}^{(-u)}\rangle y_u ) | y_u \big) = \log\E\big( \exp(t/\sqrt{|\cU| - 1} \langle \bx_u , \sqrt{|\cU| - 1}\,\,\widehat{\bmu}^{(-u)}\rangle y_u ) | y_u \big)\\
    = &\, \frac{\frac{t^2}{|\cU|-1}}{2(1 - \frac{t^2}{|\cU|-1} )} \big(\theta^2 \|\bmu\|_2^2 + (|\cU| -1)\cdot \theta^2 \|\bmu\|^2_2 \big) + \frac{\frac{t}{\sqrt{|\cU| - 1}} }{1 - \frac{t^2}{|\cU|-1}} \theta^2 \langle \bmu, \sqrt{|\cU| - 1} \bmu \rangle - \frac{d}{2}\log\Big(1 - \frac{t^2}{|\cU|-1} \Big)\\
    =&\, \frac{t \theta^2}{1 - \frac{t^2}{|\cU|-1}} \Big(1 + \frac{|\cU|t}{2(|\cU| - 1)}  \Big) - \frac{d}{2}\log\Big(1 - \frac{t^2}{|\cU|-1} \Big) = \log\E\big( \exp(t \langle \bx_u , \widehat{\bmu}^{(-u)}\rangle y_u ),
\end{align}
where the last inequality holds since the result above is independent of $y_u$. We substitute $s = 2t\Tilde{p}/\theta^2$, where $\Tilde{p} = \theta^4(|\cU| - 1)/(N\theta^2 + d)$. We focus on the critical case $\theta^2 \asymp q_m \asymp \log(m)$, $|\cU| = m \log^{-3}m$,  $d/N = \gamma \asymp 1$, thus $s^2/(|\cU| - 1) = m^{-1}\log^{-3}(m) = o(1)$, $\log(1 - x) = -x$ for $x = o(1)$, then
\begin{align}
    &\, \log\E\big( \exp(s \langle \bx_u , \widehat{\bmu}^{(-u)}\rangle y_u ) = \frac{s \theta^2}{1 - \frac{s^2}{|\cU|-1}} \Big(1 + \frac{|\cU|s}{2(|\cU| - 1)}  \Big) - \frac{d}{2}\log\Big(1 - \frac{s^2}{|\cU|-1} \Big)\\
    =&\, [1 + o(1)] s\theta^2 (1 + s/2) + \frac{d}{2} \cdot \frac{s^2}{|\cU| - 1} = [1 + o(1)] \cdot \Big[2t\Tilde{p}(1 + \frac{t\Tilde{p}}{\theta^2}) + \frac{d}{2(|\cU| - 1)} \cdot \frac{4t^2\Tilde{p}^2}{\theta^4} \Big]\\
    =&\, [1 + o(1)] \cdot 2\Tilde{p}t \Big[  1 + \frac{t\Tilde{p}}{\theta^2} \Big(1 + \frac{d}{\theta^2(|\cU| - 1)} \Big)\Big] = [1 + o(1)]\cdot 2\Tilde{p}t\Big[\Big(1 + t \, \frac{d + \theta^2 (|\cU| - 1)}{d + N\theta^2}\Big)\Big]\\
    =&\, [1 + o(1)]\cdot 2\Tilde{p}t \big(1 +  t (1 - \tau)\log^{-3}(m) \big).
\end{align}
By \eqref{eqn:ctau}, we focus on the critical case $\theta^2 \asymp q_m \asymp \log(m)$, $|\cU| = m \log^{-3}m$,  $d/N = \gamma \asymp 1$, then
\begin{align}
    c_{\tau} q_m = \frac{\theta^4}{\theta^2 + (1 - \tau)d/m} \asymp q_m, \quad \Tilde{p} = \frac{\theta^4(|\cU| - 1)}{N\theta^2 + d}\asymp \log^{-2}(m) \asymp c_{\tau}q_m \cdot \log^{-3}(m),
\end{align}
which leads to
\begin{align}
    &\, \log \E \exp \bigg(\frac{2}{N + d/\theta^2}\sum_{j\in \cU \setminus \{ u \} } \langle \bx_u,\bx_j\rangle y_j \bigg)\\
    = &\, \log\E\big( \exp(s \langle \bx_u , \widehat{\bmu}^{(-u)}\rangle y_u ) = 2c_{\tau} \big(t +  t^2 (1 - \tau)\log^{-3}(m) \big) \log^{-3}(m)\cdot q_m.
\end{align}
On the other hand, we have
\begin{align}
    &\, \E[e^{-tA_{uj}y_u y_j}|y_u] = \frac{1}{2} \E[e^{tA_{uj}}|y_u y_j = 1] + \frac{1}{2} \E[e^{-tA_{uj}}|y_u y_j = -1] = \frac{1}{2}[\alpha e^t + (1 - \alpha)] + \frac{1}{2}[\beta e^t + (1 - \beta)]\\
    = &\, 1 + \frac{\alpha(e^t - 1) + \beta(e^{-t} - 1)}{2} = \E[e^{-tA_{uj}y_u y_j}],
\end{align}
where the last equality holds since the result on the second line does not depend on $y_u$ again. Conditioned on $y_u$, $\{A_{uj}y_u y_j\}_{j \neq u}$ are i.i.d. random variables, then followed by $\alpha = aq_m /m$, $\beta = bq_m /m$ and $\log(1 + x) = x$ for $x = o(1)$, we have
\begin{align}
    &\, \log\E \Big[ \exp\Big(t\log(a/b) y_u \sum_{j\in \cU\setminus \{u\} } A_{uj}y_j \Big) \Big] = (|\cU| - 1) \cdot \log\Big(1 + \frac{aq_m [(a/b)^t - 1] + bq_m [(b/a)^t - 1] }{2m}\Big)\\
    =&\, (1 + o(1)) \frac{a[(a/b)^t - 1] + b[(b/a)^t - 1]}{2} \log^{-3}(m) \cdot q_m.
\end{align}
Due to the independence between the graph $\bA$ and feature vectors $\bX$ conditioned on $y_{u}$, one has
\begin{align}
    &\, q^{-1}_m \log \E e^{tJ_{m, u}} = q^{-1}_m \log \E \exp \bigg(\frac{2}{N + d/\theta^2}\sum_{j\in \cU \setminus \{ u \} } \langle \bx_u,\bx_j\rangle y_j \bigg) + q^{-1}_m \log\E \Big[ \exp\Big(t\log(a/b) y_u \sum_{j\in \cU\setminus \{u\} } A_{uj}y_j \Big) \Big]\\
    =&\, [1 + o(1)]\cdot \Big[ \frac{a[(a/b)^t - 1] + b[(b/a)^t - 1]}{2} + 2c_{\tau} \big(t +  t^2 (1 - \tau)\log^{-3}(m) \big)\Big] \log^{-3}(m) \asymp \log^{-3}(m),
\end{align}
where the last line holds since $a, b, c \asymp 1$. The proof of \eqref{eqn:expectationJmu} is then established once the large deviation results from graph $\bA$ and feature matrix $\bX$ are added together. 
\end{proof}

\begin{lemma}\label{eqn:decoupleuandv}
Denote by $p_{\bA}(\cdot|\widetilde{\ell}_i, \widetilde{\by}_{-i})$ the conditional probability mass function of $\bA$ given $y_i = \widetilde{\ell}_i\in \{\pm 1\}$ and $\by_{-i} = \widetilde{\by}_{-i} \in \{\pm 1\}^{N-1}$. Denote by $p_{\bX}(\cdot|\widetilde{\ell}_i, \widetilde{\by}_{-i})$ the conditional probability density function of $\bX$ given $y_i = \widetilde{\ell}_i\in \{\pm 1\}$ and $\by_{-i} = \widetilde{\by}_{-i} \in \{\pm 1\}^{N-1}$. Then
\begin{align}
   \log \bigg( \frac{\P(\bA |y_{u}, y_v, \by_{\sL}, \by_{\sU \setminus \{ u, v \} })}{\P(\bA |-y_{u}, -y_v, \by_{\sL}, \by_{\sU \setminus \{ u, v \} })} \bigg) = \log \Big( \frac{p_{\bA}(\bA|y_{u}, \by_{-u})}{p_{\bA}(\bA|-y_{u}, \by_{-u})} \Big) + \log \Big( \frac{p_{\bA}(\bA|y_{v}, \by_{-v})}{p_{\bA}(\bA|-y_{v}, \by_{-v})} \Big),\\
   \log \bigg( \frac{\P(\bX |y_{u}, y_v, \by_{\sL}, \by_{\sU \setminus \{ u, v \} })}{\P(\bX | -y_{u}, -y_v, \by_{\sL}, \by_{\sU \setminus \{ u, v \} })} \bigg) = \log \Big( \frac{p_{\bX}(\bX|y_{u}, \by_{-u})}{p_{\bX}(\bX|-y_{u}, \by_{-u})} \Big) + \log \Big( \frac{p_{\bX}(\bX|y_{v}, \by_{-v})}{p_{\bX}(\bX|-y_{v}, \by_{-v})} \Big).
\end{align}
\end{lemma}
\begin{proof}[Proof of \Cref{eqn:decoupleuandv}]
We start with the graph part. For each vertex $u\in \cV_{\sU}$, denote $\cT_{u} = \{j\in [N]\setminus{u} : y_u y_j = 1\}$ and $\cS_{u} = \{j\in [N]\setminus{u} : A_{uj} = 1\}$, then
\begin{align}
    p_{\bA}(\bA|y_{u}, \by_{-u}) \propto &\, \alpha^{|\cT_{u} \cap \cS_{u}|} \cdot (1 - \alpha)^{|\cT_{u} \setminus \cS_{u}|} \cdot \beta^{|\cS_{u} \setminus \cT_{u}|} \cdot (1 - \beta)^{|[N] \setminus (\cT_{u} \cup \cS_{u} \cup \{u\}) |},\\
    p_{\bA}(\bA|-y_{u}, \by_{-u}) \propto &\, \alpha^{|\cS_{u}\setminus \cT_{u}|} \cdot (1 - \alpha)^{|[N] \setminus (\cT_{u} \cup \cS_{u} \cup \{u\}) |} \cdot \beta^{|\cT_{u} \cap \cS_{u}|} \cdot (1 - \beta)^{|\cT_{u} \setminus \cS_{u} |},
\end{align}
where $\propto$ hides the factor not involving $\{A_{uj}\}_{j=1}^{N}$ and $y_u$. Then
\begin{align}
    \log \Big( \frac{p_{\bA}(\bA|y_{u}, \by_{-u})}{p_{\bA}(\bA|-y_{u}, \by_{-u})} \Big) = \Big( |\cT_{u} \cap \cS_{u}| -  |\cS_{u}\setminus \cT_{u}|\Big) \log\Big( \frac{\alpha}{\beta} \Big) + \Big( |\cT_{u} \setminus \cS_{u}| -  |[N] \setminus (\cT_{u} \cup \cS_{u} \cup \{u\}) |\Big) \log\Big(\frac{1 - \alpha}{1 - \beta} \Big).
\end{align}
For the left hand side, we assume $y_{u} = 1, y_v = -1$ and factor out the terms irrelevant to $u$ and $v$, then
\begin{align}
     \P(\bA |y_{u}, y_v, \by_{\sL}, \by_{\sU \setminus \{ u, v \} })
    \propto \beta^{A_{uv}}(1 - \beta)^{1 - A_{uv}} \cdot &\, \alpha^{|\cT_{u} \cap \cS_{u}|} \cdot (1 - \alpha)^{|\cT_{u} \setminus \cS_{u}|} \cdot \beta^{|\cS_{u} \setminus \cT_{u}|} \cdot (1 - \beta)^{|[N] \setminus (\cT_{u} \cup \cS_{u} \cup \{u\}) |}\\
    \cdot &\,\alpha^{|\cT_{v} \cap \cS_{v}|} \cdot (1 - \alpha)^{|\cT_{v} \setminus \cS_{v}|} \cdot \beta^{|\cS_{v} \setminus \cT_{v}|} \cdot (1 - \beta)^{|[N] \setminus (\cT_{v} \cup \cS_{v} \cup \{v\}) |}\, .
\end{align}
We perform the same calculation under the assumption $y_{u} = -1, y_v = 1$, which gives
\begin{align}
    \P(\bA |-y_{u}, -y_v, \by_{\sL}, \by_{\sU \setminus \{ u, v \} })
    \propto \beta^{A_{uv}}(1 - \beta)^{1 - A_{uv}} \cdot &\, \alpha^{|\cS_{u}\setminus \cT_{u}|} \cdot (1 - \alpha)^{|[N] \setminus (\cT_{u} \cup \cS_{u} \cup \{u\}) |} \cdot \beta^{|\cT_{u} \cap \cS_{u}|} \cdot (1 - \beta)^{|\cT_{u} \setminus \cS_{u} |}\\
    \cdot &\, \alpha^{|\cS_{v}\setminus \cT_{v}|} \cdot (1 - \alpha)^{|[N] \setminus (\cT_{v} \cup \cS_{v} \cup \{v\}) |} \cdot \beta^{|\cT_{v} \cap \cS_{v}|} \cdot (1 - \beta)^{|\cT_{v} \setminus \cS_{v} |}\,,
\end{align}
where the probability of generating edge $(u, v)$ remains unchanged when flipping the signs of $u$ and $v$ at the same time. The proof follows easily by rearranging and separating relevant terms.

For the second part, note that
\begin{align}
    p_{\bX}(\bX|\by) \propto \E_{\bmu} \exp\Big( - \frac{1}{2}\sum_{j\in \cV} \|\bx_j - y_j \bmu\|_2^2 \Big) \propto \E_{\bmu} \exp\Big( \Big\langle \sum_{j\in \cV} \bx_j y_j, \bmu \Big\rangle \Big),
\end{align}
where $\propto$ hides quantities that do not depend on $\by$. Consequently,
\begin{align}
    \frac{p_{\bX}(\bX|y_{u}, \by_{-u})}{p_{\bX}(\bX|-y_{u}, \by_{-u})} = \E_{\bmu}\exp(2y_u\bx_u^{\sT}\bmu).
\end{align}
For the left hand side, similarly, let $\propto$ hide the quantities independent of $u$, $v$, then
\begin{align}
    \P(\bX |y_{u}, y_v, \by_{\sL}, \by_{\sU \setminus \{ u, v \} }) \propto \E_{\bmu}\exp(y_u\bx_u^{\sT}\bmu + y_v\bx_v^{\sT}\bmu) = \E_{\bmu}\exp(y_u\bx_u^{\sT}\bmu) \cdot \E_{\bmu}\exp(y_v\bx_v^{\sT}\bmu).
\end{align}
The conclusion follows easily by the linearity of expectation.
\end{proof}

\begin{lemma}[{Lemma F.4, \cite{abbe2022lp}}]\label{lem:F4ABBElp}
    Denote by $p_{\bX}(\cdot|\widetilde{\ell}_i, \widetilde{\by}_{-i})$ the conditional probability density function of $\bX$ given $y_i = \widetilde{\ell}_i\in \{\pm 1\}$ and $\by_{-i} = \widetilde{\by}_{-i} \in \{\pm 1\}^{N-1}$. Then there exists some large enough constant $c>0$ such that for each $i \in [N]$, with probability at least $1 - e^{-c q_N}$,
    \begin{align}
        \bigg| y_i \log \bigg(  \frac{p_{\bX}(\bX|y_i, \by_{-i}) }{p_{\bX}(\bX|y_i, \by_{-i}) } \bigg) - \frac{2\theta^2}{N\theta^2 + d} \sum_{j\neq i} \<\bx_i, \bx_{j}\>y_j \bigg|/ q_N = o(1).
    \end{align}
\end{lemma}

\begin{lemma}[{Lemma F.5, \cite{abbe2022lp}}]\label{lem:F5ABBElp}
    Denote by $p_{\bA}(\cdot|\widetilde{\ell}_i, \widetilde{\by}_{-i})$ the conditional probability mass function of $\bA$ given $y_i = \widetilde{\ell}_i\in \{\pm 1\}$ and $\by_{-i} = \widetilde{\by}_{-i} \in \{\pm 1\}^{N-1}$. Then there exists some large enough constant $c>0$ such that for each $i \in [N]$, with probability at least $1 - e^{-c q_N}$,
    \begin{align}
        \bigg| y_i \log \bigg(  \frac{p_{\bA}(\bA|y_i, \by_{-i}) }{p_{\bX}(\bA|y_i, \by_{-i}) } \bigg) - \log\Big( \frac{a}{b}\Big) \sum_{j \neq i} A_{ij} y_j \bigg| / q_N = o(1).
    \end{align}
\end{lemma}

\section{Performance of optimal spectral estimator}
According to \cite{abbe2022lp} and the discussion in \Cref{sec:ITLowerBoundsCSBM}, the ideal estimator for the label of the node $i\in\cV_{\sU}$ could be derived from
\begin{align}
    \widehat{y}_i^{\,\,\mathrm{genie}} = \underset{y = \pm 1} {\arg \max}~\P(y_i = y|\bA, \bX, \by_{-i})
\end{align}

\begin{lemma}\label{lem:optimalApprox}
    For each given $i\in \cV_{\sU}$, following the $o_{\P}(q_m;q_m)$ notation in \cite{abbe2022lp}, we have
\begin{align}
    \bigg| \log \Big( \frac{\P(y_i = 1|\bA, \bX, \by_{-i})}{\P(y_i = -1|\bA, \bX, \by_{-i})} \Big) -   \Big[\log\big(\frac{a}{b} \big) \bA \by  + \frac{2}{N + d/\theta^2}\bG \by  \Big]_i\bigg| = o_{\P}(q_m;q_m).\notag
\end{align}
\end{lemma}
\begin{proof}[Proof of \Cref{lem:optimalApprox}]
    By definition of conditional probability and the independence between $\bA|\by$ and $\bX|\by$, for vertex $i\in \cV_{\sU}$, one has
    \begin{subequations}
        \begin{align}
            &\,\log \Big( \frac{\P(y_i = 1|\bA, \bX, \by_{-i})}{\P(y_i = -1|\bA, \bX, \by_{-i})} \Big) = \log \Big( \frac{\P(\bA, \bX| y_i = 1, \by_{-i})}{\P(\bA, \bX| y_i = -1, \by_{-i})} \Big)\\
            =&\, \log \Big( \frac{\P(\bA| y_i = 1, \by_{-i})}{\P(\bA| y_i = -1, \by_{-i})} \Big) + \log \Big( \frac{\P(\bX| y_i = 1, \by_{-i})}{\P(\bX| y_i = -1, \by_{-i})} \Big).
        \end{align}
    \end{subequations}
    Then, one could apply Lemmas F.4, F.5 in \cite{abbe2022lp} separately to conclude the results for the two terms above.
\end{proof}
From \Cref{lem:optimalApprox} above, the ideal estimator $(\widehat{y}_1^{\mathrm{genie}},\ldots, \widehat{y}_m^{\mathrm{genie}})^\sT$ can be approximated by 
\begin{equation}\label{eq:approx_genie}
    \sign \left(\log(a/b) (\bA_{\sU\sL} \by_{\sL} +\bA_{\sU}\by_{\sU} ) + \frac{2}{N + d/\theta^2}(\bG_{\sU\sL}\by_{\sL}+\bG_{\sU}\by_{\sU})\right).
\end{equation} 
Note that $\bA_{\sU\sL}, \by_{\sL}$ and $\bG_{\sU\sL}$ are accessible for us in semi-supervised setting. Below, \Cref{lem:intermediate_close_w} indicates that a scaled version of \eqref{eq:approx_genie} is entrywisely close to $\widehat{\by}_{\mathrm{PCA}}$ in \eqref{eqn:pcaEstimator} up to a global sign flip.
\begin{lemma}\label{lem:intermediate_close_w}
    Denote $\bar{\bu} \coloneqq \by_{\sU}/\sqrt{m}$. For each $i \in \sU$, define
    \begin{align}
        \bw \coloneqq \log(a/b) \Big( \bA_{\sU\sL} \by_{\sL}/\sqrt{m} + \bA_{\sU} \bar{\bu}  \Big) + \frac{2}{N + d/\theta^2}(\bG_{\sU\sL}\by_{\sL}/\sqrt{m}+\bG_{\sU}\bar{\bu}).
    \end{align}
    Then for $\widehat{\by}_{\mathrm{PCA}}$ in \eqref{eqn:pcaEstimator}, there exists some sequence $\{\epsilon_m\}_{m}$ going to $0$ such that
    \begin{align}
        \P(\min_{c = \pm 1}\|c\widehat{\by}_{\mathrm{PCA}}- \bw\|_{\infty} \geq \epsilon_m \, m^{-1/2}\log(m)) \lesssim m^{-2}.
    \end{align}
\end{lemma}
\begin{proof}[Proof of \Cref{lem:intermediate_close_w}]
    Define the following intermediate-term
    \begin{align}
        \bv = \log \Big( \frac{\alpha}{\beta} \Big) \cdot \bigg( \bA_{\sU\sL} \frac{1}{\sqrt{m}} \by_{\sL}+ \frac{m(\alpha - \beta)}{2} \cdot \bu_{2}(\bA_{\sU}) \bigg) + \frac{2\theta^2}{N\theta^2 + d} \Bigg( \bG_{\sU\sL} \frac{1}{\sqrt{m}}\by_{\sL} + m\theta^2 \cdot \bu_1(\bG_{\sU}) \Bigg).
    \end{align}
    By definition of $\alpha$ and $\beta$ in \Cref{ass:asymptotics}, $\alpha/\beta = a/b$. We focus on the case $a > b$, then
    \begin{align}
        \|\bv - \bw\|_{\infty} \leq &\, \log(a/b)\cdot \|\bA_{\sU}\bar{\bu} - \frac{(a - b)q_m}{2}\bu_2(\bA_{\sU})\|_{\infty} + \frac{2\theta^2}{N\theta^2 + d} \|\bG_{\sU}\bar{\bu} - m\theta^2 \bu_1(\bG_{\sU})\|_{\infty}\\
        \|\bv - \widehat{\by}_{\mathrm{PCA}}\|_{\infty} \leq &\, \bigg| \log \Big( \frac{\lambda_1(\bA) + \lambda_2(\bA)}{\lambda_1(\bA) - \lambda_2(\bA)} \Big) - \log(\alpha/\beta) \bigg| \cdot \frac{1}{\sqrt{m}} \| \bA_{\sU \sL} \by_{\sL}\|_{\infty}  \\
         & + \, \bigg| \log \Big( \frac{\lambda_1(\bA) + \lambda_2(\bA)}{\lambda_1(\bA) - \lambda_2(\bA)} \Big) \lambda_2(\bA_{\sU}) - \log(a/b) \frac{(a - b)q_m}{2} \bigg| \cdot \|\bu_2(\bA_{\sU})\|_{\infty}\\
         & + \, \frac{1}{\sqrt{m}}\bigg|\frac{2\lambda_1 (\bG)}{N\lambda_1(\bG) + dN} -  \frac{2\theta^2}{N\theta^2 + d}\bigg| \cdot \|\bG_{\sU\sL}\by_{\sL}\|_{\infty}, \\
         & + \, \bigg|\frac{2\lambda_1 (\bG) \lambda_1 (\bG_{\sU})}{N\lambda_1(\bG) + dN} -  \frac{2m\theta^4}{N\theta^2 + d}\bigg| \cdot \|\bu_1(\bG_{\sU})\|_{\infty}. 
    \end{align}
Without loss of generality, we assume $\<\bu_1(\bG_{\sU}, \bar{\bu})\> \geq 0$ and $\<\bu_2(\bA_{\sU}, \bar{\bu})\> \geq 0$. Also, by {Lemma B.1, Theorem 2.1 in \cite{abbe2022lp}}, with probability at least $1 - e^{-n}$,
\begin{align}
    \lambda_1 (\bG) = (1 + o(1)) \cdot N\theta^2, \quad \lambda_1 (\bG_{\sU}) = (1 + o(1)) \cdot m\theta^2,
\end{align}
and for some large constant $c>4$, with probability at least $1 - n^{-c}$, there exists some vanishing sequence $\{\epsilon_m\}_m$ such that
\begin{align}
    \| \bu_1(\bG_{\sU}) - \bG_{\sU} \bar{\bu}/(m \theta^2) \|_{\infty} \lesssim \epsilon_m \cdot m^{-1/2}, \quad \| \bu_1(\bG_{\sU}) \|_{\infty} \lesssim m^{-1/2}.
\end{align}
One can also obtain the upper bounds for $\| \bA_{\sU\sL}\|_{2 \to \infty}$ and $\| \bG_{\sU\sL}\|_{2 \to \infty}$. The remaining procedure follows similarly as Lemma F.1 in \cite{abbe2022lp} and Corollary 3.1 in \cite{abbe2020entrywise}.
\end{proof}

\begin{proof}[Proof of \Cref{thm:achievability_CSBM} (1)]
    First, for each node $i\in \cV_{\sU}$, if there exists some positive constant $\xi$ such that $q_m^{-1}\sqrt{m} y_i (\widehat{\by}_{\mathrm{PCA}})_{i} \geq \xi$, then the estimator $\sign(\widehat{\by}_{\mathrm{PCA}})$ recovers the label of each node correctly. Thus a sufficient condition for exact recovery is
    \begin{align}
        q_m^{-1} \sqrt{m}\min_{i\in \sU} y_i (\widehat{\by}_{\mathrm{PCA}})_{i} \geq \xi, \quad \textnormal{ for some positive constant } \xi.
    \end{align}
    Remind the result of \Cref{lem:intermediate_close_w}, for some vanishing positive sequence $\{\epsilon_{m}\}_m$, we have $\min_{c = \pm 1}\|c\widehat{\by}_{\mathrm{PCA}} - \bw\|_{\infty} \geq \epsilon_m \, m^{-1/2}q_m$ with probability at most $m^{-2}$. Denote $\hat{c} \coloneqq \argmin_{c = \pm 1} \|c\widehat{\by}_{\mathrm{PCA}} - \bw\|_{\infty}$ and $\hat{\bv} = \hat{c} \cdot \widehat{\by}_{\mathrm{PCA}}$. Based on the facts above, the sufficient condition for exact recovery can be further simplified as
    \begin{align}
        q_m^{-1} \sqrt{m}\min_{i\in \sU} y_i\hat{v}_{i} \geq q_m^{-1} \sqrt{m}\min_{i\in \sU} y_iw_{i} - \epsilon_m \geq \xi,
    \end{align}
where the last inequality holds since $\epsilon_{m}$ vanishes to $0$. Then we have
    \begin{align}
        \P(\psi_m = 0) = &\, \P(\sign(\widehat{\by}_{\mathrm{PCA}}) = \by ) \geq \P(q_m^{-1}\sqrt{m}\min_{i\in \sU} y_i \cdot (\widehat{\by}_{\mathrm{PCA}})_{i} \geq \xi )\\
        \geq &\, \P(q_m^{-1}\sqrt{m}\min_{i\in \sU} y_i \hat{v}_{i} \geq \xi, \,\, q_m^{-1}\sqrt{m}\|\hat{\bv} - \bw\|_{\infty} < \epsilon_m)\\
        \geq &\, \P(q_m^{-1}\sqrt{m}\min_{i\in \sU} y_i w_i \geq \xi, \,\, q_m^{-1}\sqrt{m}\|\hat{\bv} - \bw\|_{\infty} < \epsilon_m)\\
        \geq &\, \P(q_m^{-1}\sqrt{m}\min_{i\in \sU} y_i w_i \geq \xi) - \P( q_m^{-1}\sqrt{m}\|\hat{\bv} - \bw\|_{\infty} \geq \epsilon_m)\\
        \geq &\, 1 - \sum_{i \in \sU} \P(q_m^{-1}\sqrt{m} y_i w_i < \xi ) - m^{-2} = 1 - m\cdot \P(q_m^{-1}\sqrt{m} y_i w_i < \xi ) - m^{-2}.
    \end{align}
    Note that $\sqrt{m} w_i y_i  = W_{n, i}([N])$ defined in \Cref{lem:WmuLDP}. We take $0 < \epsilon < \frac{a-b}{2(1 - \tau)}\log(a/b) + 2c_{\tau}$, then for any $\delta >0$, there exists some large enough positive constant $M$ such that for $m \geq M$, $\epsilon_m < \xi$, it follows that
    \begin{align}
        \P(\sqrt{m} w_i y_i \leq \xi q_m) \leq \exp\Big( - \Big(\sup_{t\in \R} \Big\{ \xi t + I(t, a_{\tau}, b_{\tau}, c_{\tau})) \Big\} + \delta \Big)\cdot  \log(m)\Big).
    \end{align}
    By combining the arguments above, the probability of accomplishing exact recovery is lower bounded by
    \begin{align}
       \P(\psi_m = 0) \geq 1 - m^{1 - \sup_{t\in \R}\{ \xi t + I(t, a_{\tau}, b_{\tau}, c_{\tau}))\} + \delta} - m^{-2}\quad \overset{m \to \infty}{\longrightarrow} 1,
    \end{align}
    since $I(a_{\tau}, b_{\tau}, c_{\tau}) = \sup_{t\in \R}\{ \xi t + I(t, a_{\tau}, b_{\tau}, c_{\tau}))\} > 1$ by assumption when choosing small enough $\xi$ and $\delta$.
\end{proof}

\begin{proof}[Proof of \Cref{thm:achievability_CSBM} (2)]
The proof procedure follows similarly to Theorem 4.2 in \cite{abbe2022lp}, where we should apply the large deviation results \Cref{lem:WmuLDP} and \Cref{lem:intermediate_close_w} instead.
\end{proof}

\begin{proof}[Proof of \Cref{thm:general_pcaEstimator}]
    The proof procedure follows similarly to Theorem 4.4 in \cite{abbe2022lp}, where we should apply the new large deviation result \Cref{lem:WmuLDP} instead. The proof is simplified since $\by_{\sL}$ is accessible under the semi-supervised learning regime.
\end{proof}

\section{The analysis of the ridge regression on linear graph convolution}

For $\CSBM ( \by, \bmu, \alpha, \beta, \theta)$, in this section, we focus on analyzing how these parameters $c_\tau, a_\tau $ and $b_\tau$ defined in Assumption~\ref{ass:asymptotics} affect the learning performances of the \textit{linear} graph convolutional networks. We consider a graph convolutional kernel $h(\bX) \in\R^{N\times d}$ which is a function of data matrix $\bX$ and adjacency matrix $\bA$ sampled from $\CSBM ( \by, \bmu, \alpha, \beta, \theta)$. We add self-loop and define the new adjacency matrix $\bA_{\rho} \coloneqq \bA + \rho \bI_{N}$, where $\rho \in\R$ represents the intensity. Let $\bD_{\rho}$ be the diagonal matrix whose diagonals are the average degree for $\bA_{\rho}$, i.e., $[\bD_{\rho}]_{ii} = \frac{1}{N}\sum_{i=1}^{N}\sum_{j=1}^N(\bA_{\rho})_{ij}$ for each $i\in [N]$. Denote $\bD:=\bD_0$, indicating no self-loop added.  Recall that the normalization we applied for the linear graph convolutional layer is
\begin{align} 
 h(\bX)=\frac{1}{\sqrt{Nq_m}} \bD^{-1}_{\rho}\bA_{\rho}\bX.
\end{align}
Let us denote the expectation of the average degree of $\bA_\rho$ by
\begin{equation}\label{eq:tilde_d}
    \widetilde{d} \coloneqq \frac{1}{N} \sum_{i=1}^{N} \E[\bD_\rho]_{ii} = \frac{a_\tau + b_\tau}{2}q_m + \rho.
\end{equation} 
However, $h(\bX)$ is hard to deal with when we consider its large deviation principle. Instead, we use the following $\widetilde{h}(\bX)$ for analysis
\begin{align}
 \widetilde{h}(\bX)=\frac{1}{\widetilde d \cdot\sqrt{Nq_m}} \bA_{\rho}\bX.
\end{align}

\subsection{Large Deviation Results for Ridge Regression Estimators}\label{sec:LDP_ridge}
For any $i\in\cV$, we denote $\cN_i\subset \cV$ as the neighborhood of vertex $i\in\cV$.
We consider the case that the feature learned by the GCN is $\zeta \sqrt{q_m}\,  \bmu$ for some constant $\zeta$, i.e., $\bw =\zeta \sqrt{q_m} \, \bmu$. 
Notice that
\begin{align}
    h_i =&\,  y_i \zeta \sqrt{q_m} (\E \bD_\rho)^{-1} ( \bA_\rho \bX)_{i:} \bmu \\
    =&\, \frac{y_i \zeta \sqrt{q_m} }{ \widetilde{d} } \Big(\sum_{j\in \cN_i} \bmu^{\sT}(\theta y_j \bmu + \bz_j) + \rho\bmu^{\sT}(\theta y_i \bmu + \bz_i)\Big)\\
    =&\, \underbrace{\frac{\rho \zeta \theta \sqrt{q_m} y_i^2 \|\bmu\|_2^2}{\widetilde{d}} }_{I_1} + \underbrace{\frac{\rho \zeta \sqrt{q_m}  y_i \bmu^{\sT}\bz_i}{\widetilde{d}}}_{I_2} + \underbrace{\frac{\zeta \theta \sqrt{q_m} \|\bmu\|_2^2}{\widetilde{d}}\sum_{j\neq i} \bA_{ij} y_i y_j }_{I_3} + \underbrace{\frac{\zeta \sqrt{q_m}}{\widetilde{d}}\sum_{j\neq i} y_i \bA_{ij}\bmu^{\sT} \bz_j }_{I_4}.\label{eq:decompose_h_i}
\end{align}
Here $\cN_i$ is the neighborhood of vertex $i\in\cV_{\sU}$. 

\begin{proposition}[LDP for Ridge Regression]\label{prop:LDP_regression}
Under the Assumption~\ref{ass:asymptotics} with $(\bA, \bX) \sim \CSBM (\by, \bmu, \alpha, \beta, \theta)$ and $\rho = s q_m$ for some constant $s \in \R$. Assume $d/N \ll q_m$, then for any fixed $i\in\cV_\sU$ and constant $\zeta>0$, we have 
    \[\lim_{m\to\infty}q_m^{-1}\log \P(y_i \zeta  (\E \bD_\rho)^{-1} ( \bA_\rho \bX)_{i:}  \bmu \le \eps \sqrt{q_m})=-\sup_{t\in\R}\{\eps t+g(a, b, c, \tau, \zeta, s, t)\}\]
for sufficiently small $\eps>0$ and all large $m$, where
    \begin{align}
    g(a_\tau, b_\tau, c_\tau,  \zeta, s, t)= &\, g_1(t) + g_2(t),\\
    g_1(t) =&\, - \frac{2ts \zeta \sqrt{c_\tau}}{a_\tau+b_\tau+2s} - \frac{2t^2 s^2 \zeta^2}{(a_\tau+b_\tau+2s)^2},\\
    g_2(t) =&\, - \frac{a_\tau}{2 }\Big[\exp\Big( \frac{2t\zeta \sqrt{c_\tau}}{a_\tau+b_\tau+2s}\Big)  - 1\Big] -  \frac{b_\tau}{2 }\Big[\exp\Big( - \frac{2t\zeta \sqrt{c_\tau}}{a_\tau+b_\tau+2s} \Big)  - 1\Big].
\end{align}
Consequently, for any sufficiently small $\eps>0$ and any $\delta>0$, there exists some $N_0>0$ such that for all $N\ge N_0$, we have
\[\P(y_i \zeta  (\E \bD_\rho)^{-1} ( \bA_\rho \bX)_{i:}  \bmu \le \eps \sqrt{q_m})=\exp{(-q_m[\sup_{t\in\R}\{\eps t+g(a_\tau, b_\tau, c_\tau,  \zeta, s, t)\}-\delta])}\]
\end{proposition}
 
\begin{proof}[Proof of \Cref{prop:LDP_regression}]
Our goal is to calculate the following moment-generating function 
\begin{align}
    \E[\exp(t h_i)] \coloneqq \E_{\bA}[\E_{\bX}[\exp(t h_i) | \bA]].
\end{align}
First, since $\|\bmu\|_2 = 1$, $y_i^2 = 1$, then in \eqref{eq:decompose_h_i}, $ I_1 = \rho \zeta \theta \sqrt{q_m}/\widetilde{d}$, and it is deterministic. Second, $\bmu^{\sT} \bz_i \sim \Normal(0, 1)$, then $I_2 \sim \Normal(0, \rho^2\zeta^2q_m/\widetilde{d}^2)$, and
\begin{align}
    \E_{\bX}[\exp(tI_2)|y_i] = \exp\Big( \frac{t^2 \rho^2 \zeta^2 q_m }{2\widetilde{d}^2} \Big) = \E[\exp(tI_2)],
\end{align}
where the last equality holds since the result we obtained is independent of $y_i, \bA$.
Let $\cN_i$ denote the set of neighbors of node $i$ and $|\cN_i|$ denote its cardinality. 

Conditioned on $\bA, \by, \bmu$, then $I_4 \sim \Normal(0, |\cN_i|\zeta^2 q_m/\widetilde{d}^2)$, and
\begin{align}
    \E_{\bX}[\exp(tI_4)|\bA, y_i, \bmu] = \exp\Big( \frac{t^2 \zeta^2 |\cN_i|q_m}{2\widetilde{d}^2} \Big).
\end{align}
Note that $|\cN_i| = \sum_{j = 1}^{N} A_{ij}$, and $I_3$ is independent of $\bX$,  then
\begin{align}
   \E_{\bX}[\exp\big(t(I_3 + I_4)\big)|\bA, y_i, \bmu] = \exp\bigg( \frac{t\zeta\theta\sqrt{q_m}}{\widetilde{d}}\sum_{j\neq i} A_{ij} \Big(y_i y_j + \frac{t\zeta\sqrt{q_m}}{2\widetilde{d}\theta} \Big) \bigg)
\end{align}
One could take the expectation over $\bA$ conditioned on $\by$, then
\begin{align}
    &\, \E_{\bA}\Big[\exp \bigg( \frac{t\zeta \theta \sqrt{q_m}}{\widetilde{d}} A_{ij} \Big(y_i y_j + \frac{t\zeta\sqrt{q_m}}{2\widetilde{d}\theta} \Big) \bigg) \Big| y_i \Big] \\
    =&\, \frac{1}{2} \E_{\bA}\Big[\exp \bigg( \frac{t\zeta \theta\sqrt{q_m}}{\widetilde{d}} A_{ij} \Big( 1 + \frac{t\zeta\sqrt{q_m}}{2\widetilde{d}\theta} \Big) \bigg) \Big| y_i y_j = 1\Big] + \frac{1}{2} \E_{\bA}\Big[\exp \bigg( \frac{t\zeta\theta\sqrt{q_m}}{\widetilde{d}} A_{ij} \Big( -1 + \frac{t\zeta\sqrt{q_m}}{2\widetilde{d}\theta} \Big) \bigg)\Big| y_i y_j = -1\Big]\\
    =&\, \frac{\alpha}{2} \exp\bigg( \frac{t\zeta \theta\sqrt{q_m}}{\widetilde{d}} + \frac{t^2 \zeta^2 q_m}{2\widetilde{d}^2} \bigg) + \frac{1 - \alpha}{2} + \frac{\beta}{2}\exp\bigg( -\frac{t\zeta \theta\sqrt{q_m}}{\widetilde{d}} + \frac{t^2 \zeta^2 q_m}{2\widetilde{d}^2} \bigg) + \frac{1 - \beta}{2}\\
    =&\, 1 + \frac{\alpha}{2} \bigg( \exp\Big( \frac{t\zeta \theta\sqrt{q_m}}{\widetilde{d}} + \frac{t^2 \zeta^2 q_m}{2\widetilde{d}^2} \Big) - 1\bigg) + \frac{\beta}{2} \bigg( \exp\Big( - \frac{t\zeta \theta\sqrt{q_m}}{\widetilde{d}} + \frac{t^2 \zeta^2 q_m}{2\widetilde{d}^2} \Big) - 1\bigg)\\
    =&\, 1 + \frac{\alpha}{2} \bigg( \exp\Big( (1 + o(1))\frac{t\zeta \theta\sqrt{q_m}}{\widetilde{d}} \Big) - 1\bigg) + \frac{\beta}{2} \bigg( \exp\Big( - (1 + o(1)) \frac{t\zeta \theta\sqrt{q_m}}{\widetilde{d}} \Big) - 1\bigg),
\end{align}
where the last equality holds since $\widetilde{d} \asymp q_m$, $\theta \asymp \sqrt{q_m}$, $\zeta \asymp 1$, and for some fix $t$, $1 \asymp |\frac{t\zeta \theta\sqrt{q_m}}{\widetilde{d}}| \gg \frac{t^2 \zeta^2 q_m}{2\widetilde{d}^2} \asymp q^{-1}_m = o(1)$ for sufficiently large $m$. At the same time, the result above is again independent of $y_i, \bmu$. Recall $\alpha = a q_m/m = o(1)$, $\beta = bq_m/m = o(1)$, $\frac{\theta^4}{\theta^2 + d/N} = c_{\tau} q_m$ in \Cref{ass:asymptotics}. By using $\log(1 + x) = x$ for $x = o(1)$, we then have
\begin{align}
     q^{-1}_m \log \E_{\bA} \big[ \E_{\bX}[\exp\{ t(I_3 + I_4)\}] \big] =&\, \frac{1}{q_m}\sum_{j\in [N]\setminus \{ i \} } \log \bigg( 1 + \frac{\alpha}{2}\Big[\exp\Big( \frac{t\zeta \theta\sqrt{q_m}}{\widetilde{d}}\Big)  - 1\Big] +  \frac{\beta}{2}\Big[\exp\Big( - \frac{t\zeta \theta\sqrt{q_m}}{\widetilde{d}} \Big)  - 1\Big]\bigg)\\
     =&\, \frac{N - 1}{q_m} \cdot \bigg( \frac{a q_m}{2m}\Big[\exp\Big( \frac{t\zeta \theta\sqrt{q_m}}{\widetilde{d}}\Big)  - 1\Big] +  \frac{b q_m}{2m}\Big[\exp\Big( - \frac{t\zeta \theta\sqrt{q_m}}{\widetilde{d}} \Big)  - 1\Big] \bigg)\\
     =&\, \frac{a}{2(1 - \tau)}\Big[\exp\Big( \frac{t\zeta \theta\sqrt{q_m}}{\widetilde{d}}\Big)  - 1\Big] +  \frac{b}{2(1 - \tau)}\Big[\exp\Big( - \frac{t\zeta \theta\sqrt{q_m}}{\widetilde{d}} \Big)  - 1\Big]\\
     = &\, \frac{a}{2(1 - \tau)}\Big[\exp\Big( \frac{2t\zeta c_{\tau} }{a + b + 2s}\Big)  - 1\Big] +  \frac{b}{2(1 - \tau)}\Big[\exp\Big( - \frac{2t\zeta c_{\tau} }{a + b + 2s} \Big)  - 1\Big],
\end{align}
where the last equality holds since we apply $\widetilde{d} = (\frac{a + b}{2} + s)q_m$ and $\theta^2 = (1 + o(1)) c_{\tau} q_m$ since $d/N \ll q_m$ by assumption. Combining the calculations above, we then compute the following rate function
\begin{align}
   g(a_\tau, b_\tau, c_\tau,  \zeta, s, t) =&\, - \,q_m^{-1}\log\E[\exp(t h_i)] = - \,q_m^{-1}\log\E_{\bA}\big[ \E_{\bX}[\exp(t(I_1 + I_2 + I_3 + I_4))| \bA] \big]\\
   =&\, -(g_1(t) + g_2(t)).
\end{align}
Then, we can apply Lemma H.5 in \cite{abbe2022lp} to conclude our results in this proposition.
\end{proof}

\begin{lemma}\label{lem:rate_fun}
    For function $g(a_\tau, b_\tau, c_\tau,  \zeta, s, t)$ defined in Proposition~\ref{prop:LDP_regression}, we know that 
\[J(a_\tau, b_\tau, c_\tau,  \zeta, s )\coloneqq \sup_{t\in\R}g(a_\tau, b_\tau, c_\tau,  \zeta, s, t)\le I(a_\tau, b_\tau, c_\tau)\]
and the equality is attained when $s = \frac{2c_{\tau}}{\log(a_\tau/b_\tau)}=\frac{2c_{\tau}}{\log(a/b)}  $.

Moreover, if $s=0$, then $g_1(t)\equiv0$ and 
\[J(a_\tau, b_\tau, c_\tau,  \zeta, 0 )= I(a_\tau, b_\tau, 0)\le J(a_\tau, b_\tau, c_\tau,  \zeta, s ).\]
\end{lemma}
\begin{proof}[Proof of \Cref{lem:rate_fun}]
Notice that both $g_1(t)$ and $g_2(t)$ are concave. First, $g_1(t)$ achieves its maximum at the point $t_1 := - \sqrt{c_\tau}(a+b+2s)/(2s \zeta )$, and $g_2(t)$ achieves its maximum at the point $t_2 := (a+b+2s)\log(b/a) /(4\zeta \sqrt{c_\tau})$. Note that
\begin{align}
    \sup_{t\in\R}g(a_\tau, b_\tau, c_\tau,  \zeta, s, t) \leq \max_{t\in \R} g_1(t) + \max_{t\in \R} g_2(t) = g_1(t_1) + g_2(t_2),\label{eq:sup_g}
\end{align}
 where 
 \begin{align}
     g_1(t_1)=c_{\tau}/2,\quad \text{ and }g_2(t_2) = (\sqrt{a} - \sqrt{b})^2/(2(1 - \tau)).
 \end{align}
Thus, this proves the upper bound on $J(a_\tau, b_\tau, c_\tau,  \zeta, s )$.

Notice that the equality in \eqref{eq:sup_g} holds if $t_1=t_2$. It turns out that when $s = \frac{2c_{\tau}}{\log(a/b)}  $, then $t_1 = t_2$, and $g_1(t_1) = c_{\tau}/2$, $g_2(t_2) = (\sqrt{a} - \sqrt{b})^2/(2(1 - \tau))$. Therefore, in this case, we have
\begin{align}
    \max_{t\in \R} g(a, b, c, \tau, \zeta, t) = \frac{(1 - \tau)^{-1}(\sqrt{a} - \sqrt{b})^2 + c_{\tau}}{2}=I(a_\tau, b_\tau, c_\tau).
\end{align}
\end{proof}

\subsection{Preliminary Lemmas on Ridge Regression Estimator}

Note the facts that $(\bB^{\top}\bB + \bI_{d})^{-1}\bB^{\top} = \bB^{\top}(\bB\bB^{\top} + \bI_{N})^{-1}$ for any matrix $\bB^{N \times d}$, $\bP_\sU^2=\bP_\sU$, $\bP_\sL^2=\bP_\sL$ and $\bP_\sU = \bI_N - \bP_\sL$, then,
\begin{align}
    h(\bX)\widehat{\bbeta}  &\,= h(\bX)(h(\bX)^{\top} \bP_\sL h(\bX)+  \lambda \bI_d )^{-1}h(\bX)^{\top} \bP_\sL\by  \\
    &\, = h(\bX)[ (\bP_\sL h(\bX))^{\top} \bP_\sL h(\bX) + \lambda \bI_d ]^{-1}(\bP_\sL h(\bX))^{\top}\by  \\
    &\, = h(\bX) h(\bX)^{\top} \bP_\sL[ \bP_\sL h(\bX)h(\bX)^{\top}\bP_\sL + \lambda \bI_N ]^{-1}\by.
\end{align}
Consequently, the test risk can be written as 
\begin{align}
    \cR(\lambda ) =&\, \frac{1}{m}\by^{\top}   \bQ^{\top}\bP_{\sU}\bQ\by,\,\text{ where } \bQ:=h(\bX) h(\bX)^{\top} \bP_\sL[ \bP_\sL h(\bX)h(\bX)^{\top}\bP_\sL + \lambda \bI_N ]^{-1}-\bI_N
\end{align}

\begin{lemma}\label{lem:approx_A}
Assume that $|\rho|\lesssim q_m$ and $q_m\gtrsim \log N$. Let $\bD_{\rho}$ be the diagonal matrix where each diagonal represents the average degree of the graph $\bA_{\rho}$ after adding the self-loop $\rho$ and let $\widetilde{d}$ denote the expected average degree of $\bA_{\rho}$. Then $\|\bD^{-1}_{\rho} - (\widetilde{d})^{-1}\|\lesssim q_m^{-3/2}$ with probability at least $1 - e^{-N}$. Furthermore, with probability at least $1-2N^{-10}- 2e^{-N}$,
\begin{align}
    \| \bD^{-1}_{\rho}\bA_{\rho} - \E \bA_{\rho} /\widetilde{d} \| \lesssim q_m^{-1/2}.\label{eq:approx_A}
\end{align}
Consequently, $\|\bD^{-1}_{\rho}\bA_{\rho}\| \leq C$ with probability at least $1-2N^{-10}- 2e^{-N}$ for some constant $C > 0$.
\end{lemma}
\begin{proof}[Proof of \Cref{lem:approx_A}]
    First, for any $i\in[N]$, note that $[\bD_{\rho}]_{ii} = \frac{1}{N} \sum_{i = 1}^{N} (\rho + \sum_{j \neq i}\bA_{ij}) = \rho + \frac{1}{N}\sum_{i = 1}^{N}\sum_{j \neq i}\bA_{ij}$, and $\widetilde{d} = \E [\bD_{\rho}]_{ii} = \frac{a_{\tau} + b_{\tau}}{2 }q_m + \rho$, then by Bernstein inequality,
    \begin{align}
        \P\Big( \big| [\bD_{\rho}]_{ii} - \widetilde{d} \big| \geq \sqrt{q_m} \Big) = \P\Big( \Big| \sum_{i = 1}^{N}\sum_{j \neq i} (\bA_{ij} - \E \bA_{ij}) \Big| \geq N \sqrt{q_m} \Big) \lesssim \exp(-N).
    \end{align}
    Thus by comparing the entrywise difference of $[\bD_{\rho}]_{ii} - \widetilde{d}$, with probability at least $1 - e^{-N}$,  we have
    \begin{align}
        \|\bD^{-1}_{\rho} - (\widetilde{d})^{-1}\| \lesssim q_m^{-3/2}.\label{eq:d_inv_diff}
    \end{align}
   For the second part of the statement, by the triangle inequality, we have
    \begin{align}
        \| \bD^{-1}_{\rho}\bA_{\rho} - \E \bA_{\rho} /\widetilde{d} \| \leq \| \bD^{-1}_{\rho}\bA_{\rho} - \bD^{-1}_{\rho}\E \bA_{\rho}\| + \| \bD^{-1}_{\rho}\E \bA_{\rho} -  \E \bA_{\rho} /\widetilde{d} \|
    \end{align}
    For the first term, we proved that with probability at least $1 - e^{-N}$, $[D_{\rho}]_{ii} = (1 + o(1)) \widetilde{d} \asymp q_m$ with deviation at most $\sqrt{q_m}$, then according to \Cref{lem:concentrateA}, with probability at least $1-2N^{-10}- 2e^{-N}$, when $q_m \gtrsim \log(N)$, the following is bounded by
\begin{align}
    &\, \|\bD^{-1}_{\rho}\,(\bA_{\rho} - \E\bA_{\rho})\| \leq (\|\widetilde{d}^{-1}\| + \|\bD^{-1}_{\rho} - \widetilde{d}^{-1}\| ) \cdot \|(\bA_{\rho} - \E\bA_{\rho})\| \leq (q_m^{-1} + q_m^{-3/2} )\sqrt{q_m} \asymp q_m^{-1/2}.
\end{align}
For the second term, note that $\|\E \bA_{\rho}\| \lesssim q_m$, then by results above
\begin{align}
    &\, \| \big( \bD^{-1}_{\rho} - \widetilde{d}^{-1} \big) \, \E\bA_{\rho}\| \leq \| \bD^{-1}_{\rho} - (\widetilde{d})^{-1}\| \cdot \|\E\bA_{\rho}\|\lesssim \frac{1}{\sqrt{q_m}}.
\end{align}
Therefore, with probability at least $1-2N^{-10}- 2e^{-N}$,
\begin{align}
    \| \bD^{-1}_{\rho}\bA_{\rho} - \E \bA_{\rho} /\widetilde{d} \| \lesssim q_m^{-1/2}.
\end{align}
For the last part, $\|\E \bA_{\rho} /\widetilde{d}\| \lesssim 1$ since
\begin{align}
    \E \bA_{\rho} /\widetilde{d} = \frac{1}{N} \ones \ones^{\sT} + \frac{(a - b)q_m + 2\rho}{(a + b)q_m + 2\rho} \frac{1}{N} \by \by^{\sT}.
\end{align}
Then the proof is completed by triangle inequality since $q_m \gg 1$, and
\begin{align}
    \|\bD^{-1}_{\rho}\bA_{\rho}\| \leq \|\E \bA_{\rho} /\widetilde{d}\|  + \| \bD^{-1}_{\rho}\bA_{\rho} - \E \bA_{\rho} /\widetilde{d} \| \lesssim 1 + q_m^{-1/2}\lesssim 1.
\end{align}
\end{proof}

\begin{lemma}\label{lem:approx_X}
Consider $\bX\sim\GMM(\bmu,\by,\theta)$. Suppose that  $d\lesssim N$, then we have
\begin{align}
\Big\|\frac{1}{\sqrt{Nq_m}}\bX -\frac{\theta}{\sqrt{Nq_m}}\by\bmu^\top\Big\|\le  \frac{C}{\sqrt{q_m}},
    \label{eq:approx_X}
\end{align}
 with probability at least $1-2e^{-cN}$ for some constants $C,c>0$.
\end{lemma}
\begin{proof}[Proof of \Cref{lem:approx_X}]
Recall the concentration on the operator norm of the Gaussian random matrix for $\bZ\in\R^{N\times d}$ (see \cite{vershynin2018high}). Then for every $t>0$, there exists some constant $c>0$ such that 
\begin{align}\label{eq:Gaussian-spectral-norm}
    \P(\norm{\bZ}\ge \sqrt{N}+\sqrt{d}+t)\le 2\exp{(-ct^2)}. 
\end{align}
Then, we know that $\frac{\norm{\bZ}}{\sqrt{N}}\lesssim 1$ with probability at least $1-2e^{-cN}$ by taking $t = \sqrt{N}$. Then we have
\begin{align}
    \norm{\frac{1}{\sqrt{Nq_m}}\bX -\frac{\theta}{\sqrt{Nq_m}}\by\bmu^\top}\le \frac{\|\bZ\| }{ \sqrt{Nq_m}}  \lesssim  \frac{\sqrt{N} + \sqrt{d} }{ \sqrt{Nq_m}} \asymp \frac{1}{\sqrt{q_m}},
\end{align}with probability at least $1-2e^{-cN}$.
\end{proof}

\begin{lemma}\label{lem:approx_kernel}
Consider $(\bA, \bX) \sim \CSBM (\by, \bmu, \alpha, \beta, \theta)$. Under the Assumption~\ref{ass:asymptotics}, when $d\lesssim N$, we have that
    \[\|h(\bX)-\bH\|\leq \frac{C}{\sqrt{q_m}},\]
with probability at least $1-cN^{-10}$,  where $\bH:=\frac{ \kappa_m }{ \sqrt{N }}\by\bmu^\top$ and $\kappa_m:=\frac{\alpha-\beta+2\rho}{\alpha+\beta+2\rho}\cdot\frac{\theta}{\sqrt{q_m}}$, for all large $m$ and $n$ and some constants $c,C>0$.
\end{lemma}
\begin{proof}[Proof of \Cref{lem:approx_kernel}]
Notice that $\bH=\frac{\theta}{\widetilde d\sqrt{Nq_m}}(\E\bA_\rho)\by\bmu^\top$. We apply Lemmas~\ref{lem:approx_A} and \ref{lem:approx_X} to derive that
\begin{align}
    \|h(\bX)-\bH\|\le~& \norm{h(\bX)-\frac{1}{\widetilde d}(\E\bA_\rho)\frac{\bX}{\sqrt{Nq_m}}}+\norm{\frac{1}{\widetilde d}(\E\bA_\rho)\frac{\bX}{\sqrt{Nq_m}}-\bH}\\
    \le ~& \frac{1}{\sqrt{q_m}}(\theta+2+\sqrt{N/d})\cdot\norm{\bD_\rho^{-1}\bA_\rho-\frac{1}{\widetilde d}(\E\bA_\rho) }+\frac{\|\E\bA_\rho\|}{\widetilde d}\norm{\frac{\bX}{\sqrt{Nq_m}}-\frac{\theta}{\sqrt{Nq_m}}\by\bmu^\top}\lesssim 1/\sqrt{q_m}, 
\end{align}with probability at least $1-cN^{-10}$. Here we apply the fact that $\theta/\sqrt{q_m}\lesssim 1$.
\end{proof}

\begin{lemma}\label{lem:approx_beta}
Consider $(\bA, \bX) \sim \CSBM (\by, \bmu, \alpha, \beta, \theta)$. Under the Assumption~\ref{ass:asymptotics} with $d\lesssim N$,  the ridge regression solution $\widehat\bbeta$ defined in \eqref{eq:regression_solu} satisfies
    \[\frac{1}{\sqrt{N}}\|\widehat\bbeta-\widetilde\bbeta\|\le \frac{C}{\sqrt{q_m}},\]
with probability at least $1-cN^{-10}$, where $\widetilde\bbeta:=\frac{\sqrt{N}\kappa_m\tau}{\kappa_m^2\tau+\lambda}\bmu$ and $\kappa_m =\frac{\alpha-\beta+2\rho}{\alpha+\beta+2\rho}\cdot\frac{\theta}{\sqrt{q_m}} $, for all large $m$ and $n$ and some constants $c,C>0$. Moreover, $\|\widehat\bbeta\|\lesssim\sqrt{N}$ with probability at least $1-cN^{-10}$.
\end{lemma}
\begin{proof}
Notice that  $\widetilde\bbeta=  (\bH^{\top} \bP_{\sL} \bH + \lambda \bI_d )^{-1}\bH^{\top} \bP_{\sL}\by,$ where $\bH$ is defined by Lemma~\ref{lem:approx_kernel}. From Lemma~\ref{lem:approx_A}, we know that $\|h(\bX)\|\lesssim 1$ and $\|\bH\|\lesssim 1$ with probability at least $1-2N^{-10}$. Moreover, $\|(\bH^{\top} \bP_{\sL} \bH + \lambda \bI_d )^{-1}\|\le \lambda^{-1}$ and $\|(h(\bX)^{\top} \bP_{\sL} h(\bX) + \lambda \bI_d )^{-1}\|\le \lambda^{-1}$. Therefore, applying Lemma~\ref{lem:approx_kernel}, we derive that
\begin{align}
    \frac{1}{\sqrt{N}}\|\widehat\bbeta-\widetilde\bbeta\|\le~& \|(\bH^{\top} \bP_{\sL} \bH + \lambda \bI_d )^{-1}\bH^{\top} -(h(\bX)^{\top} \bP_{\sL} h(\bX) + \lambda \bI_d )^{-1}h(\bX)^\top\|\cdot\|\bP_{\sL}\by\|/\sqrt{N}\\
    \le ~& \|(\bH^{\top} \bP_{\sL} \bH + \lambda \bI_d )^{-1}\|(\bH-h(\bX))\|+\|(\bH^{\top} \bP_{\sL} \bH + \lambda \bI_d )^{-1}\|\\
     ~& \cdot\|\bH -h(\bX)\|\cdot(\|\bH \|+\|h(\bX)\|)\cdot\|(h(\bX)^{\top} \bP_{\sL} h(\bX) + \lambda \bI_d )^{-1}\|\cdot\|h(\bX)\| \\
      \lesssim ~& \|\bH -h(\bX)\|\lesssim 1/\sqrt{q_m},
\end{align}with at least $1-cN^{-10}$, for some constant $c>0$.
\end{proof}

\subsection{Exact Recovery Threshold for Ridge Regression on Linear GCN}

\begin{lemma}\label{lem:his}
    Let $h(\bX)^\top=[\bh_1,\ldots,\bh_N]$ and $\widetilde{h}(\bX)^\top=[\bar\bh_1,\ldots,\bar\bh_N]$. For any $i\in [N]$ and deterministic unit vector $\bu\in\R^d$, there exists some $c,C>0$ such that
    \begin{align}
        \P( |(\bar\bh_i-\bh_i)^\top\bu| \le  C/\sqrt{Nq_m})\ge~& 1-cN^{-10},\label{eq:h_i_diff}\\
        \P(|\bar\bh_i^\top\bu|\le C/\sqrt{N })\ge~& 1-cN^{-10},\label{eq:bar_h_i}\\
        \P\left(\|\bh_i\|\le C\sqrt{d/(Nq_m)} \right)\ge~& 1-cN^{-10},\label{eq:norm_h_i}
    \end{align}for all large $n$ and $m$.
\end{lemma}
\begin{proof}
 For any unit vector $\bu\in\R^d$ and $i\in[N]$, conditioning on event \eqref{eq:d_inv_diff}, we have
    \begin{align}
     \Big|(\bar\bh_i-\bh_i)^\top\bu\Big|\le~&  \frac{1}{\sqrt{Nq_m}}\big\|(\widetilde d)^{-1}\bI_N-\bD_\rho^{-1}\big\|\cdot |(\bA_\rho\bX)_{i:}\bu |
      \lesssim  \frac{1}{q_m^2\sqrt{N} }\Big| ( \bA_\rho\bX)_{i:}\bu\Big|,\label{eq:AX_row0}
    \end{align}
where we employ Lemma~\ref{lem:approx_A}. Then, for any $i\in[N]$, we can further have
    \begin{align}
        |(\bA_\rho\bX)_{i:}\bu |= ~& \left|\sum_{j\in\cN_i}(\theta y_j\bmu^\top\bu+\bz_j^\top\bu)+\theta\rho\bmu^\top\bu+\rho\bz_i^\top\bu\right|\\
        \le ~& \theta (|\cN_i|+\rho)+\Big|\sum_{j\in\cN_i} \bz_j^\top\bu+ \rho\bz_i^\top\bu\Big|\label{eq:AX_row}
    \end{align}   
Based on {Lemma 3.3 in \cite{alt2021extremal}}, we can upper bound the degree of each vertex by 
\begin{equation}\label{eq:degree_bound}
    \P(|\cN_i|\le C\log N)\ge 1-C_DN^{-D},
\end{equation}
for any $i\in [N]$, some constants $C,C_D>0$ and sufficiently large constant $D>0$. Meanwhile, since each $\bz_j^\top\bu\sim\cN(0,1)$, by applying Hoeffding's inequality ({Theorem 2.6.2 in \cite{vershynin2018high}}), we can deduce that
\begin{equation}\label{eq:sum_Gaussian_1}
    \P\Big(\Big|\sum_{j\in\cN_i} \bz_j^\top\bu+ \rho\bz_i^\top\bu\Big|\le t\Big||\cN_i|=k\Big)\ge 1- 2\exp{\Big(-\frac{ct^2}{k+\rho^2}\Big)},
\end{equation}
for any $k\in\N$, $t>0$, and some constant $c>0$.
Then combining \eqref{eq:degree_bound} and \eqref{eq:sum_Gaussian_1}, for any large $D>0$, there exists some constants $C,C_D>0$ such that
\begin{equation}\label{eq:sum_Gaussian}
    \P\Big(\Big|\sum_{j\in\cN_i} \bz_j^\top\bu+ \rho\bz_i^\top\bu \Big|\le  C\log N, ~ |\cN_i|\le  C\log N\Big)\ge 1- 2C_DN^{-D}.
\end{equation}
Thus, with \eqref{eq:AX_row} and $\rho \asymp \log N$, we can conclude that $|(\bA_\rho\bX)_{i:}\bu |\lesssim q_m^{3/2}$ 
with probability at least $1- 2C_DN^{-D}.$ Following with \eqref{eq:AX_row}, we can conclude that
\begin{equation}
    \P\left( \Big|(\bar\bh_i-\bh_i)^\top\bu\Big|\le C/\sqrt{q_mN}\right)\ge 1-cN^{-10}.
\end{equation}
For the second part, we can analogously get $|\bar\bh_i^\top\bu|\lesssim \frac{1}{q_m^{3/2}\sqrt{N}}|(\bA_\rho\bX)_{i:}\bu |$. Then, we can apply \eqref{eq:AX_row} and \eqref{eq:sum_Gaussian} to conclude \eqref{eq:bar_h_i}.

Finally, notice that
\begin{align}
    \|(\bA_\rho\bX)_{i:}\|= ~& \left\|\sum_{j\in\cN_i}(\theta y_j\bmu+\bz_j)+\theta\rho\bmu+\rho\bz_i\right\|\\
        \le ~& \theta (|\cN_i|+\rho)+\Big\|\sum_{j\in\cN_i} \bz_j+ \rho\bz_i\Big\|.
\end{align}
Applying Theorem 3.1.1 in \cite{vershynin2018high}, we know that 
\begin{equation}\label{eq:z_i_norm}
    \P(\|\bz_i\|\le 2\sqrt{d})\ge 1-2\exp{(-cd)}
\end{equation} 
for some constant $c>0$ and any $i\in [N]$. Thus, combining \eqref{eq:degree_bound} and Lemma~\ref{lem:Z_concentration}, we have that with probability at least $1-cN^{-10},$
\[\|\bh_i\|\le \frac{1}{q_m^{3/2}\sqrt{N}}\|(\bA_\rho\bX)_{i:}\|\lesssim \sqrt{\frac{d}{Nq_m}}\]
because of the fact that $q_m\lesssim d$ and $  N\asymp m$. This completes the proof of this lemma.
\end{proof}

Inspired by \cite{abbe2020entrywise,abbe2022lp}, we now apply a general version of leave-one-out analysis for $\widehat\bbeta$ by defining the following approximated estimator. For any $i\in\cV_\sU$, denote by
\begin{equation}\label{eq:beta_i}
    \widehat\bbeta^{(i)}=(h^{(i)}(\bX)^{\top} \bP_{\sL} h^{(i)}(\bX) + \lambda \bI_d )^{-1}h^{(i)}(\bX)^{\top} \bP_{\sL}\by,
\end{equation}
where $h^{(i)}(\bX):=\frac{1}{\sqrt{Nq_m}}\bD_\rho^{-1}\bA_\rho(\bX-\bZ^{(i)})$ and $\bZ^{(i)}:=[\bz_1\mathbf{1}_{1\in\cN_i\cup \{i\}},\ldots,\bz_k\mathbf{1}_{k\in\cN_i\cup \{i\}},\ldots, \bz_N\mathbf{1}_{N\in\cN_i\cup \{i\}}]^\top\in\R^{N\times d}$. Here, the difference between $h^{(i)}(\bX)$ and $h(\bX)$ is that we turn off the feature noises $\bz_i$ for vertices $\cN_i\cup\{i\}$. In this case, conditional on $\by$ and $\bmu$, both $\widetilde{\bbeta}$ and $\widehat\bbeta^{(i)}$ are independent with $\bh_i$ and $\bar\bh_i$ given any $i\in\cV_\sU$. Next, we present the following properties for $\widehat\bbeta^{(i)}$.
\begin{lemma}\label{lem:beta_i}
Assume that $q_m\ll d\ll Nq_m$. For \eqref{eq:regression_solu} and \eqref{eq:beta_i}, we have
    \begin{align}
        \frac{1}{\sqrt{N}}\norm{\widehat\bbeta^{(i)}-\widehat\bbeta}\le~& C\sqrt{\frac{d}{q_mN}},\\
        \big\|\widehat\bbeta^{(i)}\big\|\le~& C\sqrt{N},
    \end{align}with a probability at least $1-cN^{-10}$, for some constants $c,C>0$
\end{lemma}
\begin{proof}
 Based on Lemma~\ref{lem:approx_A}, we have that
    \begin{align}
        \|h^{(i)}(\bX)-h(\bX)\|=\frac{1}{\sqrt{Nq_m}}\|\bD_\rho^{-1}\bA_\rho \bZ^{(i)}\|\le \frac{1}{\sqrt{q_mN}}\|\bZ^{(i)}\| \lesssim \sqrt{\frac{d}{q_mN}}
    \end{align}
    with probability at least $1-cN^{-10}$, where we utilize \eqref{eq:degree_bound} and \eqref{eq:Gaussian-spectral-norm} for $\bZ^{(i)}$ as well. Thus, we know  that $\|h^{(i)}(\bX)\|\lesssim 1$ with probability at least $1-cN^{-10}$. 
    Then, analogously with Lemma~\ref{lem:approx_beta}, we have
\begin{align}
    \frac{1}{\sqrt{N}}\|\widehat\bbeta-\widehat\bbeta^{(i)}\|\le~& \|(h^{(i)}(\bX)^{\top} \bP_{\sL} h^{(i)}(\bX)  + \lambda \bI_d )^{-1}h^{(i)}(\bX)^{\top} -(h(\bX)^{\top} \bP_{\sL} h(\bX) + \lambda \bI_d )^{-1}h(\bX)^\top\| \\
    \le ~& \|(h^{(i)}(\bX)^{\top} \bP_{\sL} h^{(i)}(\bX)+ \lambda \bI_d )^{-1}\|\cdot \|(h^{(i)}(\bX)-h(\bX))\|+\|(h^{(i)}(\bX)^{\top} \bP_{\sL}h^{(i)}(\bX)+ \lambda \bI_d )^{-1}\|\\
     ~& \cdot\|h^{(i)}(\bX) -h(\bX)\|\cdot(\|h^{(i)}(\bX) \|+\|h(\bX)\|)\cdot\|(h(\bX)^{\top} \bP_{\sL} h(\bX) + \lambda \bI_d )^{-1}\|\cdot\|h(\bX)\| \\
      \lesssim ~& \|h^{(i)}(\bX) -h(\bX)\|\lesssim \sqrt{\frac{d}{q_mN}},
\end{align}with a probability at least $1-cN^{-10}$, for some constant $c>0$. Also, with Lemma~\ref{lem:approx_beta}, we can show that $\|\widehat\bbeta^{(i)}\|\lesssim \sqrt{N}$ with very high probability.
\end{proof}

\begin{lemma}\label{lem:approx_beta_i}
Under the same assumption as Lemma~\ref{lem:approx_beta}, for each $i\in\cV_\sU$ the estimator $\widehat\bbeta^{(i)}$ defined in \eqref{eq:beta_i} satisfies
    \[\frac{1}{\sqrt{N}}\|\widehat\bbeta^{(i)}-\widetilde\bbeta\|\le \frac{C}{\sqrt{q_m}},\]
with probability at least $1-cN^{-10}$, where $\widetilde\bbeta $ is defined in Lemma~\ref{lem:approx_beta}, for all large $m$ and $n$ and some constants $c,C>0$.  
\end{lemma}
The proof of this lemma is the same as Lemma~\ref{lem:approx_beta}, so we ignore the details here.

\begin{proof}[Proof of Theorem~\ref{thm:exact_linear}]
Recall that $$\zeta=\frac{\kappa\tau}{\kappa^2\tau+\lambda}\quad \text{ and }\quad\kappa=\sqrt{c_\tau}\cdot\frac{a_\tau-b_\tau+2s}{a_\tau+b_\tau+2s},$$
where both $\zeta$ and $\kappa$ are some constants in $\R$.
Hence, we know that
\[ \frac{\kappa_m\tau }{\kappa_m^2\tau +\lambda}=\zeta(1+o(1)).\]
Then, $\widetilde\bbeta/\sqrt{N}=\zeta\bmu+o(1/\sqrt{N})$. Denote $\by_{\sU,i}$ as the $i$-th entry of label $\by_{\sU}$. Firstly, we consider the general case when $\rho = s q_m$ for some fixed constant $s\in\R$.
For each $i\in\cV_\sU$, we can utilize Lemmas~\ref{lem:his}, ~\ref{lem:beta_i}, and~\ref{lem:approx_beta_i} to obtain that
\begin{align}
    \left|\frac{y_i}{\sqrt{N}}\bar\bh_i^\top\widetilde{\bbeta}-\frac{y_i}{\sqrt{N}} \bh_i^\top\widehat{\bbeta}\right|\le ~& \frac{1}{\sqrt{N}}\Big|(\bar\bh_i-\bh_i)^\top\widetilde{\bbeta}\Big|+\frac{1}{\sqrt{N}}\Big|(\bar\bh_i-\bh_i)^\top(\widetilde{\bbeta}-\widehat\bbeta^{(i)})\Big|\\
    ~&+\frac{1}{\sqrt{N}}\Big|\bar\bh_i^\top(\widetilde{\bbeta}-\widehat\bbeta^{(i)})\Big|+\frac{1}{\sqrt{N}}\Big| \bh_i^\top(\widehat\bbeta^{(i)}-\widehat\bbeta)\Big|\\
    \le ~& \frac{C}{\sqrt{Nq_m}}+C\frac{d}{Nq_m},
\end{align}
with probability at least $1-cN^{-10}$ for some constants $c,C>0$. Here, we applied \eqref{eq:h_i_diff} when $\bu=\widetilde{\bbeta}/\sqrt{N}$ and $\bu=(\widetilde{\bbeta}-\widehat\bbeta^{(i)})/\sqrt{N}$, \eqref{eq:bar_h_i} when $\bu=(\widetilde{\bbeta}-\widehat\bbeta^{(i)})/\sqrt{N}$, and \eqref{eq:norm_h_i}. Then, if $d\lesssim \sqrt{Nq_m}$, we can conclude that 
\begin{equation}
    \left|\frac{y_i}{\sqrt{N}}\bar\bh_i^\top\widetilde{\bbeta}-\frac{y_i}{\sqrt{N}} \bh_i^\top\widehat{\bbeta}\right|\le \frac{C}{\sqrt{Nq_m}}, 
\end{equation}with very high probability for some universal constant $C>0$.

Therefore, we can take $\eps_m=1/\sqrt{q_m}$ to get
\begin{align}
\P(\psi_m(\sign(\widehat\by_{\sU}),\by_{\sU})=0)=~&\P\big(\min_{i\in[m]}\by_{\sU,i}\cdot\widehat\by_{\sU,i}>0\big)= \P\Big( \min_{i\in\cV_{\sU}} \frac{y_{i}\cdot \bh_i^\top \widehat\bbeta}{\sqrt{N}}>0\Big)\\
\ge ~& \P\Big( \min_{i\in\cV_{\sU}}\frac{y_{i}}{\sqrt{N}}\cdot \bar\bh_{i}^\top \widetilde\bbeta>\frac{C\eps_m}{\sqrt{N}}\Big)- \sum_{i\in\cV_{\sU}}\P\Big(\left|\frac{y_i}{\sqrt{N}}\bar\bh_i^\top\Tilde{\bbeta}-\frac{y_i}{\sqrt{N}} \bh_i^\top\widehat{\bbeta}\right| > \frac{C\eps_m}{\sqrt{N}}\Big)\\
\ge ~& \P\big(\min_{i\in\cV_{\sU}}y_{i}\zeta\cdot \frac{1}{ \widetilde d}(\bA_{\rho}\bX)_{i:} \bmu>C\sqrt{q_m}\eps_m\big)-Cm^{-2}\\
\ge ~& 1-\sum_{i\in\cV_{\sU}}\P\Big(y_{i}\cdot \frac{\zeta}{ \widetilde d}(\bA_\rho\bX)_{i:}  \bmu\le C\sqrt{q_m}\eps_m\Big)-Cm^{-2}\\
\ge ~& 1-m\cdot\P\Big(y_{i}\cdot \frac{\zeta}{ \widetilde d}(\bA_\rho\bX)_{i:}  \bmu\le C\sqrt{q_m}\eps_m\Big)-Cm^{-2}\\
\ge~& 1-m^{1- \sup_{t\in\R} \{\eps_m t+g(a ,b ,c,\tau,\zeta, s,t)\}+\delta}-Cm^{-2},
\end{align}
for any $\delta>0$ and sufficiently large $m$, where in the last line we employ Proposition~\ref{prop:LDP_regression}. Thus, applying Lemma~\ref{lem:rate_fun}, we know that when $J(a_\tau, b_\tau, c_\tau,  \zeta, s )>1$, $\P(\psi_m(\sign(\widehat\by_{\sU}),\by_{\sU})=0)\to 1$ as $m\to\infty$.

When $s=\frac{2c_{\tau}}{\log(a/b)} $, Lemma~\ref{lem:rate_fun} implies that $J(a_\tau, b_\tau, c_\tau,  \zeta, s )=I(a_\tau, b_\tau, c_\tau )$ defined in \eqref{eqn:rate_I_abc_tau}. Notice that $J(a_\tau, b_\tau, c_\tau,  \zeta, s )\le I(a_\tau, b_\tau, c_\tau )$ for any $s\in\R$. Whereas $s=0$, Lemma~\ref{lem:rate_fun} implies that $J(a_\tau, b_\tau, c_\tau,  \zeta, s )=I(a_\tau, b_\tau, 0)$. Hence, this completes the proof of this theorem.
\end{proof}

\subsection{Asymptotic Errors for Ridge Regression on  Linear GCN}

\begin{lemma}\label{lem:risks_app}
Under the Assumption~\ref{ass:asymptotics}, there exist some constant $c,C>0$ such that with probability at least $1-cN^{-2}$,
\begin{align}
    |\overline{\cR}(\lambda)-\cR(\lambda)|\le~&\frac{C}{\sqrt{q_m}},\\
    |\overline{\cE}(\lambda)-\cE(\lambda)|\le~& \frac{C}{\sqrt{q_m}},
\end{align}
    where \begin{align}
        \overline{\cR}(\lambda):=~&\frac{1}{m}(\bH\Tilde{\bbeta}-\by)^\top\bP_{\sU}(\bH\Tilde{\bbeta}-\by),\\
        \overline{\cE}(\lambda):=~&\frac{1}{n}(\bH\Tilde{\bbeta}-\by)^\top\bP_{\sL}(\bH\Tilde{\bbeta}-\by).
    \end{align}
\end{lemma}
\begin{proof}
    From Lemmas~\ref{lem:approx_kernel} and~\ref{lem:approx_beta}, we know that
    \[\frac{1}{\sqrt{m}}\|\bH\Tilde{\bbeta}-h(\bX)\widehat\bbeta\|\le\frac{1}{\sqrt{m}}\|\bH\|\cdot\|\Tilde{\bbeta}-\widehat\bbeta\|+\frac{1}{\sqrt{m}}\|\bH -h(\bX)\|\cdot\|\widehat\bbeta\|\lesssim \frac{1}{\sqrt{q_m}},\]
    with probability at least $1-CN^{-2}$ for some constant $C>0$. Since $\norm{\bP_\sL}$ and $\norm{\bP_\sU}$ are both upper bounded by one, we can directly conclude Lemma~\ref{lem:risks_app}.
\end{proof}

\begin{proof}[Proof of Theorem~\ref{thm:error}]
Based on Lemma~\ref{lem:risks_app}, we can instead compute $\overline{\cE}(\lambda)$ and $\overline{\cR}(\lambda)$. Recall that $\bH=\frac{\kappa_m}{\sqrt{N}}\by\bmu^\top$ and $\frac{1}{\sqrt{N}}\widetilde\bbeta=\frac{\kappa_m\tau }{\kappa_m^2\tau +\lambda}\bmu$. Thus, $\bH\Tilde{\bbeta}=\frac{\kappa_m^2\tau }{\kappa_m^2\tau+\lambda}\by$. Then, since $\frac{1}{m}\by^\top\bP_{\sU}\by=\frac{1}{n}\by^\top\bP_{\sL}\by=1$, we have
    \begin{align}
        \overline{\cR}(\lambda) =~&\frac{1}{m}(\bH\Tilde{\bbeta}-\by)^\top\bP_{\sU}(\bH\Tilde{\bbeta}-\by)=\Big(1-\frac{\kappa_m^2\tau_n}{\kappa_m^2\tau_n+\lambda}\Big)^2=\frac{\lambda^2}{(\kappa^2\tau+\lambda)^2}+o(1),\\
        \overline{\cE}(\lambda) =~&\frac{1}{n}(\bH\Tilde{\bbeta}-\by)^\top\bP_{\sL}(\bH\Tilde{\bbeta}-\by)=\Big(1-\frac{\kappa_m^2\tau_n}{\kappa_m^2\tau_n+\lambda}\Big)^2=\frac{\lambda^2}{(\kappa^2\tau+\lambda)^2}+o(1).
    \end{align}
    Then taking $m\to\infty$, we can get the results of this lemma.
\end{proof}

\section{Feature Learning of Graph Convolutional Networks}\label{sec:NN_proof}
In this section, we complete the proof of Theorem~\ref{thm:NN} in Section~\ref{sec:NN}. Recall that
\begin{equation}\label{eq:W_GD}
 \bW^{(1)} = \bW^{(0)} - \eta_1 \Big(\nabla_{\bW^{(0)}} \mathcal{L}(\bW^{(0)},s^{(0)}) + \lambda_1 \bW^{(0)} \Big).
\end{equation}
and
\begin{equation}\label{eq:s_Trained}
s^{(1)}=\frac{2}{n^2q_m}\by_\sL^\top\bX_\sL\bW^{(1)}\ba\Big/\log\left(\frac{N\cdot D_0+\by^\top_\sL\bA_\sL\by_\sL}{N\cdot D_0-\by^\top_\sL\bA_\sL\by_\sL}\right),
\end{equation}
The algorithm we applied in Theorem~\ref{thm:NN} is given by Algorithm~\ref{alg:gradient0}. Below, we first construct an optimal solution for this problem and present the LDP analysis. Then, we present will use \cite{ba2022high,damian2022neural} to analyze the feature learned from $\bW^{(1)}$. Finally, inspired by the optimal solution, we will prove $s^{(1)}$ is close to the optimal $s$ in \eqref{eq:optimal_rho}. Meanwhile, we also present an additional gradient-based method in Algorithm~\ref{alg:gradient} to approach the optimal $s$ in \eqref{eq:optimal_rho}. We will leave the theoretical analysis for Algorithm~\ref{alg:gradient} as a future direction to explore.

\begin{algorithm} 
   \caption{Gradient-based training for GCN in Theorem~\ref{thm:NN}}
   \label{alg:gradient0}
\begin{algorithmic}
   \Require Learning rates $\eta_1$, weight decay $\lambda_1$ 
   \State {\textbf{Initialization:}}
   $s^{(0)}=0$, $\sqrt{K}\cdot[\bW^{(0)}]_{ij}\iid\cN(0,1), ~
    \sqrt{K}\cdot[\ba]_j\iid \Unif\{\pm 1\}$, $\forall i\in[d],j\in [K]$. 
   \State {\textbf{Training Stage 1:}}
    \State $\quad\quad$ Set $\sigma(x)=x$ in \eqref{eq:NN} 
   \State $\quad\quad$ $\bW^{(1)} \gets \bW^{(0)} - \eta_1 (\nabla_{\bW^{(0)}} \mathcal{L}(\bW^{(0)},s^{(0)}) + \lambda_1 \bW^{(0)} )$
   \State  {\textbf{Training Stage 2:}}
 
   \State $\quad\quad$ $s^{(1)}\gets s^{(0)}+\frac{2}{n^2q_m} \by_\sL^\top\bX_\sL\bW^{(1)}\ba \Big/\log\left(\frac{N\cdot D_0+\by^\top_\sL\bA_\sL\by_\sL}{N\cdot D_0-\by^\top_\sL\bA_\sL\by_\sL}\right)$
 
   \Ensure Prediction function for unknown labels: $\sign(\bS_\sU\bD_{s^{(1)}}^{-1}\bA_{s^{(1)}}\bW^{(1)}\ba)$
\end{algorithmic}
\end{algorithm}

\subsection{Thresholds for GCNs and LDP analysis}
Consider $(\bA, \bX) \sim \CSBM (\by, \bmu, \alpha, \beta, \theta)$. Let $\bA\in\R^{N\times N}$ denote the adjacency matrix of the graph $G$ and let us define the degree matrix by $\bD_0 \coloneqq \diag\{D_0, \ldots, D_0\}\in\R^{N\times N}$ where $D_0=\frac{1}{N}\sum_{j=1}^{N}\sum_{i=1}^{N} \bA_{ij}$. Let $\rho=sq_m$ denote the self-loop intensity \cite{kipf2017semisupervised,wu2019simplifying,shi2024homophily} for some $s\in\R$, and $\bA_s = \bA + \rho \bI$, $\bD_s = \bD + \rho \bI$ denote the adjacency, average degree matrices of the graph after adding self-loops, respectively.

The convolutional feature vector is $\widetilde{\bx}_i \coloneqq \big((\bD_s)^{-1} {\bA_s}\bX\big)_{i:}$. Ideally, our goal is to prove that the convoluted feature vectors are linearly separable, i.e., find some $\bw \in \R^{d}$ such that
\begin{align}
    \widetilde{\bx}_i^{\sT} \bw + b > 0 \textnormal{ if } y_i = 1, \quad \widetilde{\bx}_i^{\sT} \bw + b < 0 \textnormal{ if } y_i = -1,
\end{align}
for some $b\ge0$. We consider the case that the feature learned by the GCN is exactly $ \bmu$ in GMM, i.e., the optimal margin is $\bw =  \bmu$. Under this setting, we show the LDP results for this estimator. The proof is similar to Proposition~\ref{prop:LDP_regression}.

\begin{proposition}[LDP for GCNs]\label{prop:LDP}
For $(\bA, \bX) \sim \CSBM (\by, \bmu, \alpha, \beta, \theta)$ with Assumption~\ref{ass:asymptotics}, when $s=2c_\tau/\log(\frac{a}{b})$ and $c_\tau=\theta^2/q_m+o(1)$, we have that
    \[\lim_{m\to\infty}q_m^{-1}\log \P(y_i \cdot\sqrt{q_m} (\E  {\bD_s})^{-1} ( {\bA_s} \bX)_{i:} \bmu \le \eps q_m)=-\sup_{t\in\R}\{\eps t+I(a_\tau,b_\tau,c_\tau,t),\}\]
    where $\sup_{t\in\R}I(a_\tau,b_\tau,c_\tau,t)=I(a_\tau,b_\tau,c_\tau)$ defined in \eqref{eqn:rate_I_abc_tau}.
\end{proposition}
\begin{proof}
For simplicity, let $\widetilde{d} \coloneqq \frac{1}{N} \sum_{i=1}^{N} \widetilde{D}_i = \frac{a + b}{2}q_m + \rho$, denoting the expected degree of each node $i\in [N]$. Then
\begin{align}
    h_i :=&\, y_i \widetilde{\bx}_i^{\sT} \bw = y_i \theta (\E \widetilde{\bD})^{-1} (\widetilde{\bA} \bX)_{i:} \bmu \\
    =&\, \frac{y_i \theta}{\widetilde{d}} \Big(\sum_{j\in \cN_i} \bmu^{\sT}(\theta y_j \bmu + \bz_j) + \rho\bmu^{\sT}(\theta y_i \bmu + \bz_i)\Big)\\
    =&\, \underbrace{\frac{\rho \theta^2 y_i^2 \|\bmu\|_2^2}{\widetilde{d}}}_{I_1} + \underbrace{\frac{\rho \theta y_i \bmu^{\sT}\bz_i}{\widetilde{d}}}_{I_2} + \underbrace{\frac{\theta^2 \|\bmu\|_2^2}{\widetilde{d}}\sum_{j\neq i} A_{ij} y_i y_j }_{I_3} + \underbrace{\frac{\theta}{\widetilde{d}}\sum_{j\neq i} y_i A_{ij}\bmu^{\sT} \bz_j }_{I_4}.
\end{align}

Our goal is to calculate the following moment-generating function 
\begin{align}
    \E[\exp(t h_i)] \coloneqq \E_{\bA}[\E_{\bX}[\exp(t h_i) | \bA]].
\end{align}
First, since $\|\bmu\|_2 = 1$, $y_i^2 = 1$, $ I_1 = \rho \theta^2/\widetilde{d}^2$, and it is deterministic. Second, $\bmu^{\sT} \bz_i \sim \Normal(0, 1)$, then $I_2 \sim \Normal(0, \rho^2\theta^2/\widetilde{d}^2)$, and
\begin{align}
    \E_{\bX}[\exp(tI_2)|y_i] = \exp\Big( \frac{t^2 \rho^2 \theta^2}{2\widetilde{d}^2} \Big) = \E[\exp(tI_2)],
\end{align}
where the last equality holds since the result we obtained is independent of $y_i$.
Let $\cN_i$ denote the set of neighbors of node $i$ and $|\cN_i|$ denote its cardinality. Conditioned on $\bA, \by, \bmu$, $I_4 \sim \Normal(0, |\cN_i|\theta^2/\widetilde{d}^2)$, and
\begin{align}
    \E_{\bX}[\exp(tI_4)|\bA, y_i, \bmu] = \exp\Big( \frac{t^2 \theta^2 |\cN_i|}{2\widetilde{d}^2} \Big).
\end{align}
Note that $|\cN_i| = \sum_{j = 1}^{N} A_{ij}$, and $I_3$ is independent of $\bX$,  then
\begin{align}
   \E_{\bX}[\exp\big(t(I_3 + I_4)\big)|\bA, y_i, \bmu] = \exp\bigg( \frac{t\theta^2}{\widetilde{d}}\sum_{j\neq i} A_{ij} \Big(y_i y_j + \frac{t}{2\widetilde{d}} \Big) \bigg)
\end{align}
One could take the expectation over $\bA$ conditioned on $\by$, then
\begin{align}
    &\, \E_{\bA}\Big[\exp \Big( \frac{t\theta^2}{\widetilde{d}} A_{ij} \Big(y_i y_j + \frac{t}{2\widetilde{d}} \Big) \Big) \Big| y_i \Big] \\
    =&\, \frac{1}{2} \E_{\bA}\Big[\exp \Big( \frac{t\theta^2}{\widetilde{d}} A_{ij} \Big(y_i y_j + \frac{t}{2\widetilde{d}} \Big) \Big) \Big| y_i y_j = 1\Big] + \frac{1}{2} \E_{\bA}\Big[\exp \Big( \frac{t\theta^2}{\widetilde{d}} A_{ij} \Big(y_i y_j + \frac{t}{2\widetilde{d}} \Big) \Big) \Big| y_i y_j = -1\Big]\\
    =&\, \frac{1}{2} \Big[ \alpha \exp\Big( \frac{t^2 \theta^2}{2\widetilde{d}^2} + \frac{t \theta^2}{\widetilde{d}}\Big) + (1 - \alpha) + \beta \exp\Big( \frac{t^2 \theta^2}{2\widetilde{d}^2} - \frac{t \theta^2}{\widetilde{d}}\Big)  + (1 - \beta) \Big]\\
    =&\, 1 + \frac{\alpha}{2}\Big[\exp\Big( \frac{t^2 \theta^2}{2\widetilde{d}^2} + \frac{t \theta^2}{\widetilde{d}}\Big)  - 1\Big] +  \frac{\beta}{2}\Big[\exp\Big( \frac{t^2 \theta^2}{2\widetilde{d}^2} - \frac{t \theta^2}{\widetilde{d}}\Big)  - 1\Big],
\end{align}
where the result is again independent of $y_i, \bmu$. Recall $\alpha = a q_m/m = o(1)$, $\beta = bq_m/m = o(1)$, $\frac{\theta^4}{\theta^2 + (1 - \tau )d/m} = c_{\tau} q_m$ in \Cref{ass:asymptotics}, thus $\theta^2 = (1 + o(1)) c_{\tau} q_m$. By using $\log(1 + x) = x$ for $x = o(1)$, we then have
\begin{align}
     q^{-1}_m \log \E_{\bA} \big[ \E_{\bX}[\exp\{ t(I_3 + I_4)\}] \big] =&\, \log \bigg( 1 + \frac{\alpha}{2}\Big[\exp\Big( \frac{t^2 \theta^2}{2\widetilde{d}^2} + \frac{t \theta^2}{\widetilde{d}}\Big)  - 1\Big] +  \frac{\beta}{2}\Big[\exp\Big( \frac{t^2 \theta^2}{2\widetilde{d}^2} - \frac{t \theta^2}{\widetilde{d}}\Big)  - 1\Big]\bigg)\\
     =&\, \frac{a}{2}\Big[\exp\Big( \frac{t^2 \theta^2}{2\widetilde{d}^2} + \frac{t \theta^2}{\widetilde{d}}\Big)  - 1\Big] +  \frac{b}{2}\Big[\exp\Big( \frac{t^2 \theta^2}{2\widetilde{d}^2} - \frac{t \theta^2}{\widetilde{d}}\Big)  - 1\Big]\\
     =&\, (1 + o(1)) \frac{a}{2}\Big[\exp\Big( \frac{t \theta^2}{\widetilde{d}}\Big)  - 1\Big] +  (1 + o(1)) \frac{b}{2}\Big[\exp\Big( - \frac{t \theta^2}{\widetilde{d}}\Big)  - 1\Big]
\end{align}

Combining the calculations above, compute the following rate function
\begin{align}
    g(a, b, c_\tau, t) \coloneqq - \,q_m^{-1}\log\E[\exp(t h_i)].
\end{align}
Recall $\alpha = a q_m/m = o(1)$, $\beta = bq_m/m = o(1)$, $\frac{\theta^4}{\theta^2 + (1 - \tau )d/m} = c_{\tau} q_m$ in \Cref{ass:asymptotics}, thus $\theta^2 = (1 + o(1)) c_{\tau} q_m$. By using $\log(1 + x) = x$ for $x = o(1)$, the rate function $g(a, b, c_\tau, t)$ can be calculated as
\begin{align}   
    g(a, b, c_\tau, t) =&\, - \frac{t\rho \theta^2}{q_m \widetilde{d}} - \frac{t^2\rho^2 \theta^2 }{2q_m \widetilde{d}^2} + \frac{(N-1)}{2m} \Big[ a - a\exp\Big( \frac{t^2 \theta^2}{2\widetilde{d}^2} + \frac{t \theta^2}{\widetilde{d}}\Big)  + b - b\exp\Big( \frac{t^2 \theta^2}{2\widetilde{d}^2} - \frac{t \theta^2}{\widetilde{d}}\Big) \Big] \\
    =&\, \frac{1}{2(1 - \tau)} \Big[a\Big(1 - \exp\Big( \frac{2c_{\tau}t}{a + b+ 2s} \Big) \Big) + b\Big(1 - \exp\Big( -\frac{2c_{\tau} t}{a+ b + 2s} \Big) \Big) \Big] - \frac{2c_{\tau}st}{a+ b + 2s} - \frac{2c_{\tau} s^2t^2}{(a + b + 2s)^2},
\end{align}
where in the last line, we used $\rho = s q_m$, $\widetilde{d} = (\frac{a + b}{2} + s)q_m$. By choosing $s = \frac{2c_\tau}{\log(a/b)}$, we can conclude that 
\begin{align}
    I^{\star} = \sup_{t\in \R} I(a_\tau, b_\tau, c_\tau, t) = \frac{(1 - \tau)^{-1}(\sqrt{a} - \sqrt{b})^2 + c_\tau}{2}\equiv I(a_\tau,b_\tau,c_\tau),
\end{align}
which completes the proof.    
\end{proof}

\subsection{Gradient descent for the first layer weight matrix $\bW$}
For simplicity, we denote $\widetilde\bX=\bD_s^{-1}\bA_s\bX=(\widetilde\bx_1,\ldots,\widetilde\bx_N)^\top$ where $\widetilde\bx_i\in\R^d$ for $i\in[N]$ and $s=0$.
In this case, we will explore the feature learning on $\bW$. Below, we will always fix $\ba$ (at initialization in Assumption~\ref{assump:NN}) and perform gradient descent on $\bW$ in \eqref{eq:W_GD}.  
To ease the notions, we write the initialized first-layer weights as $\bW^{(0)}=\bW_{\!0}$, and the weights after one gradient step as $\bW^{(1)}=\bW_{1}$, where the learning rate of the first gradient descent is $\eta_1>0$. Let $s^{(0)}=0$. Following the notions in \cite{ba2022high}, we denote that 
\begin{align}\label{eq:gradient-step-MSE}
\bG_1
:= 
-\nabla_{\bW_{0}}\cL(\bW_0,s^{(0)})
=
    \frac{1}{n} \widetilde\bX_\sL^\top \left[\left(\frac{1}{\sqrt{K}}\left(\by_\sL - \frac{1}{\sqrt{K}}\sigma(\widetilde\bX_\sL\bW_{\!0})\ba\right)\ba^\top\right) \odot \sigma'(\widetilde\bX_\sL\bW_{\!0})\right], 
\end{align}  
where $\widetilde\bX_\sL=\bS_\sL\widetilde\bX\in \R^{n\times d}$, $\odot$ is the Hadamard product, and $\sigma'$ is the derivative of $\sigma$ (acting entry-wise). Here $K$ represents the number of neurons in the hidden GCN layer in \eqref{eq:NN}.  
Then, from \eqref{eq:W_GD} with $\lambda_1=\eta_1^{-1}$, we have
$$\bW_{\! 1} = \bW_{\!0} + \eta_1  \cdot\bG_1-\bW_{\!0}=\eta_1 \cdot\bG_1.$$ 
Thus, our target is to analyze the gradient matrix $\bG_1$. The following proposition is similar to 
Proposition 2 in \cite{ba2022high}, implies that this gradient matrix is approximately rank one.
\begin{proposition}\label{prop:G_1}
 Under the same assumption as Theorem~\ref{thm:NN}, we have that
 \begin{align}
     \Big\|\bG_1-\frac{1}{n\sqrt{K}} \widetilde\bX_{\sL}^\top\by_\sL\ba^\top\Big\|_F\le~& \frac{Cq_m}{K},
 \end{align}
 with probability at least $1-\exp{(-c\log^2N)}$, for some constant $c,C>0$. 
\end{proposition}
\begin{proof}
First of all, analogously to the proof of Lemma~\ref{lem:approx_kernel}, we can show that
\begin{equation}\label{eq:bound_tild_X}
    \norm{\widetilde\bX_\sL}\le \sqrt{q_mN},
\end{equation}
with very high probability, since $d\lesssim N$. Moreover, $\|\by_\sL\|=\sqrt{n}$ and we can always view $\by$ as a deterministic vector in $\R^N$.
By the definition, the gradient matrix $\bG_1$ under the MSE can be simplified as follows 
\begin{align} 
    \bG_1
&=  
    -\frac{1}{n}\widetilde\bX_\sL^\top \left[\left(\frac{1}{\sqrt{K}}\left(\frac{1}{\sqrt{K}}\sigma(\widetilde\bX_\sL\bW_0)\ba - \by_\sL\right)\ba^\top\right) \odot \sigma'(\widetilde\bX_\sL\bW_0)\right]  \\
&= 
     \frac{1}{n}\cdot\frac{\mu_1}{\sqrt{K}}\widetilde\bX_\sL^\top\left(\by_\sL - \frac{1}{\sqrt{K}}\sigma(\widetilde\bX_\sL\bW_0)\ba \right)\ba^\top  + 
    \frac{1}{n}\cdot\frac{1}{\sqrt{K}}\bX^\top\left(\left(\by_\sL - \frac{1}{\sqrt{N}}\sigma(\widetilde\bX_\sL\bW_0)\ba\right)\ba^\top\odot\sigma'_\perp(\widetilde\bX_\sL\bW_0)\right), 
\end{align} 
where we utilized the orthogonal decomposition: $\sigma'(z) = \mu_1 + \sigma'_\perp(z)$. By Stein's lemma, we know that $\E[z\sigma(z)] = \E[\sigma'(z)] = \mu_1$, and hence $\E[\sigma'_\perp(z)]=0$ for $z\sim\cN(0,1)$. Notice that we consider $\sigma(x)=x$, hence $\mu_1=1$ and $\sigma'_\perp(z)\equiv 0$. Therefore, we have
\begin{align} 
\bG_1
&=\frac{1}{n}\cdot\frac{1}{\sqrt{K}}\widetilde\bX_\sL^\top \by_\sL\ba^\top -  \underbrace{\frac{1}{nK}\widetilde\bX_\sL^\top  \widetilde\bX_\sL\bW_0 \ba \ba^\top}_{\bDelta}. 
\end{align}
Notice that $\|\ba\|=1$ and we can apply \eqref{eq:Gaussian-spectral-norm} for Gaussian random matrix $\bW_0$. Thus, because of Lemma~\ref{lem:approx_kernel}, $d\lesssim n\asymp N$ and $d\lesssim K$, we have that 
\begin{align}
    \|\bDelta\|_F = \|\bDelta\|\le \frac{Cq_m}{K}(1+\sqrt{d/K}), 
\end{align}
with very high probability, which completes the proof of this proposition.

\end{proof}
This proposition shows that for $\bW$ at Gaussian initialization, the corresponding gradient matrix can be approximated in operator norm by the \textbf{rank-1} matrix only related to labels $\by_\sL$, feature matrix $\widetilde\bX_\sL$, and $\ba$.  

In the following, we will use the parameter 
\[\zeta:= \sqrt{c_\tau}\frac{\eta_1\sqrt{q_m}}{K}\frac{\alpha-\beta}{\alpha+\beta}.\]
Notice that $\zeta=\Theta(1)$ if $K/\eta_1 = \Theta(\sqrt{q_m})$. Then we can tune the learning rate $\eta_1$ to ensure that this trained and normalized weight matrix $\frac{1}{\sqrt{K}}\bW^{(1)}$ can be aligned with $\bmu$ perfectly.

\begin{lemma}\label{lem:Wa}
Under the assumption as Theorem~\ref{thm:NN}, we have that
    \[\norm{\frac{1}{\sqrt{K}}\bW^{(1)}\ba-\sqrt{c_\tau}\frac{\eta_1\sqrt{q_m}}{K}\frac{\alpha-\beta}{\alpha+\beta} \bmu}=O\big(\frac{\eta_1}{K}\big),\]
    with a probability at least $1-cN^{-10}$, for some constants $c,C>0$.
\end{lemma}
\begin{proof}
Notice that $\bW_{\! 1} =\eta_1 \cdot\bG_1$ and $\ba^\top\ba=1$.  Notice that $\widetilde\bX_{\sL}^\top\by_\sL=\sqrt{Nq_m}\cdot h(\bX)^\top\bP_\sL\by$.  Following from Proposition~\ref{prop:G_1} and Lemma~\ref{lem:his}, we can have 
    \begin{align}
      \norm{\frac{1}{\sqrt{K}}\bW^{(1)}\ba-\zeta \bmu}\le~& \frac{1}{\sqrt{K}}\Big\| \eta_1 \bDelta\ba\Big\|+\Big\|\zeta\bmu-\frac{\eta_1 }{nK} \widetilde\bX_{\sL}^\top\by_\sL\ba^\top\ba\Big\|\\
      \lesssim~& \frac{\eta_1}{K}+ \frac{\eta_1 q_m}{K^{3/2}},
    \end{align}with very high probability. Notice that here $\ba^\top\ba=1$.
Then, we can assume $q_m^2\lesssim K$ to finish this proof.
\end{proof}

\subsection{Learning the optimal self-loop weight}
\begin{lemma}\label{lem:D_0}
Under Assumption~\ref{ass:asymptotics}, we know that
    \[\left|D_0 - \bar{d}\right|\le \frac{C}{q_m^{1/2}},\]
with probability at least $1-ce^{-N}$ for some constants $c,C>0$, where $\bar d :=\frac{a_\tau+b_\tau}{2} q_m$.
\end{lemma}
This is straightforward based on the proof of Lemma~\ref{lem:approx_A}, hence we ignore the proof here.
\begin{lemma}
Under the assumption as Theorem~\ref{thm:NN}, we have that
    \[\left|\frac{2}{n^2q_m} \by_\sL^\top\bX_\sL\bW^{(1)}\ba -2c_\tau\right|\le \frac{C}{n\sqrt{q_m}}\]
    with probability at least $1-cN^{-10}$, for some constants $c,C>0$.
\end{lemma}
\begin{proof}
By Proposition~\ref{prop:G_1} and Lemma~\ref{lem:Wa}, we can replace $\bW^{(1)}\ba$ by $\zeta\bmu $. Notice that \eqref{eq:bound_tild_X} and Lemma~\ref{lem:Wa} indicate that 
\begin{align}
    \left|\frac{2}{n^2q_m} \by_\sL^\top\bX_\sL(\bW^{(1)}\ba -\zeta\bmu)\right|\le \frac{2}{n^2q_m}\|\by_\sL\|\cdot\|\bX_\sL\|\cdot\|\bW^{(1)}\ba-\zeta \bmu\|\lesssim 1/(nq_m),
\end{align}
with very high probability.
 for $s=0$.
Then, we can apply Lemmas~\ref{lem:Z_concentration} and \ref{lem:approx_X} to conclude that
\begin{align}
    \left|\frac{1}{n^2q_m} \by_\sL^\top\bX_\sL \zeta\bmu-c_\tau \right| \lesssim \frac{1}{n\sqrt{q_m}},
\end{align}
with a very high probability for sufficiently large $n$ and $m$.
\end{proof}
\begin{lemma}
Following the assumptions in Theorem~\ref{thm:NN}, we have that
    \[\left|\frac{1}{n}\by^\top_\sL\bA_\sL\by_\sL-\frac{a_\tau-b_\tau}{2}q_m\right|=o(q_m),\]
    with a probability of at least $1-cN^{-10}$, for some constants $c,C>0$.
\end{lemma}
\begin{proof}
    This lemma follows from  Lemma F.6 in \cite{abbe2022lp} and Corollary 3.1 in \cite{abbe2020entrywise}. Notice that
    \begin{align}
        \frac{1}{n}\by^\top_\sL\bA_\sL\by_\sL=~& \frac{1}{n}\sum_{i,j\in\cV_\sL}\bA_{ij}y_iy_j\\
        =~& \frac{1}{n}\sum_{i,j\text{ in same block of } \cV_{\sL }}\bA_{ij} -\sum_{i,j\text{ in different blocks of }\cV_{\sL }}\bA_{ij}.
    \end{align}
    Then, we can apply the proof of (F.23) in \cite{abbe2022lp} to conclude this lemma.
\end{proof}
Combining all the above lemmas in this section, we can derive the following lemma.
\begin{lemma}\label{lem:s_1}
    Following the assumptions in Theorem~\ref{thm:NN}, we have that
    \[\left|s^{(1)}-\frac{2c_\tau}{\log\Big(\frac{a_\tau}{b_\tau}\Big)}\right|\le \frac{C}{\sqrt{q_m}},\]
    with probability at least $1-cN^{-10}$, for some constants $c,C>0$, where $s^{(1)}$ is given by \eqref{eq:s_Trained}.
\end{lemma}


\begin{algorithm}
\caption{Gradient-based training for both $\bW$ and $s$ in GCN} \label{alg:gradient}
\begin{algorithmic}
\Require  Learning rates $\eta_t$, weight decay $\lambda_t$, number of steps $T$
\State {\textbf{Initialization:}}
   $s^{(0)}\sim\Unif([-1,1])$, $\sqrt{K}\cdot[\bW^{(0)}]_{ij}\iid\cN(0,1), ~
     \sqrt{K}\cdot[\ba]_j\iid \Unif\{\pm 1\}$, $\forall i\in[d],j\in [K]$. 
    \State {\textbf{Training Stage 1:}}
    \State {$\quad\quad$ Set $\sigma(x)=x$ in \eqref{eq:NN}}
    \State $\quad\quad$ $\bW^{(1)} \gets \bW^{(0)} - \eta_1 (\nabla_{\bW^{(0)}} \mathcal{L}(\bW^{(0)},s^{(0)}) + \lambda_1 \bW^{(0)} )$
    \State $\quad\quad$ $s^{(1)} \gets s^{(0)}$
    \State $\quad\quad$ $\bW^{(1)} \gets \bW^{(1)}\ba$
    \State  $\quad\quad$ $\ba \gets 1$
    \State {\textbf{Training Stage 2:}}
    \State $\quad\quad$ Set $\sigma(x)=\tanh(x)$ in \eqref{eq:NN}
    \State $\quad\quad$ {\textbf{For} $t=2$ to $T$ \textbf{do}}
    \State  $\quad\quad\quad\quad$ $s^{(t)} \gets s^{(t-1)} - \eta_t \nabla_{s^{(t)} } \mathcal{L}(\bW^{(1)},s^{(t-1)}) + \lambda_t s^{(t-1)}$
    \State $\quad\quad$ \textbf{End For}
    \Ensure Prediction function for unknown labels: $\sign(\bS_\sU\bD_{s^{(T)}}^{-1}\bA_{s^{(T)}}\bW^{(1)}\ba)$
\end{algorithmic}
\end{algorithm}

\subsection{Proof of Theorem~\ref{thm:NN}}

Let us recall that we consider $K/\eta_1\asymp  \sqrt{q_m}$ and $d=o(q_m^2)$ with $\bW_{i,j}\sim\cN(0,1/K)$. Finally, we complete the proof of Theorem~\ref{thm:NN} for Algorithm~\ref{alg:gradient0} as follows. Recall that $\bD_s:=(D_0+sq_m)\bI\in\R^{n\times n}$, for any $s\in\R$, where $D_0$ is the average degree of the graph. Denote that 
 \begin{align}
 s_{\text{opt}}:=~&\frac{2c_\tau}{\log\Big(\frac{a_\tau}{b_\tau}\Big)}\\
 \bD_{s^{(1)}}^{-1}\bA_{s^{(1)}}\bX=:~&[\widehat\bg_1,\ldots,\widehat\bg_N]^\top\in\R^{N\times d},\\
  \bD_{s_{\text{opt}}}^{-1}\bA_{s_{\text{opt}}}\bX=:~&[ \bar{\bg}_1,\ldots, \bar{\bg}_N]^\top\in\R^{N\times d},\\
  (\tilde{d}+s_{\text{opt}}\cdot q_m)^{-1} \bA_{s_{\text{opt}}}\bX=:~&[ \tilde\bg_1,\ldots, \tilde\bg_N]^\top\in\R^{N\times d},\\
 \frac{1}{\sqrt{K}}\bW^{(1)}\ba =:~& \widehat{\bmu}.
 \end{align}
Then, by definition, $\widehat\by_{\mathrm{GCN},i}=\widehat\bg_i^\top \widehat\bmu$ for $i\in\cV_{\sU}$. As a remark, notice that Lemma~\ref{lem:his} verifies that with high probability $\norm{\widehat\bg_i}\lesssim \sqrt{d}$. Because of this bound, we can only consider the regime when $d=o(q_m^2)$ for our following analysis. To improve this to a high dimensional regime, e.g., $d\asymp N$, we improve the following concentration without simply using the bound of $\norm{\widehat\bg_i}$. Similarly with the proof of Lemma~\ref{lem:beta_i} in ridge regression of linear GCN part, we need to do certain leave-one-out analysis to achieve a larger regime for $d$. 

Next, based on the above decomposition, we follow the proof idea of Theorem~\ref{thm:exact_linear} to complete the proof of Theorem~\ref{thm:NN}.  Combining Lemmas~\ref{lem:his},~\ref{lem:Wa}
 and \ref{lem:s_1}, for each $i\in\cV_\sU$, we can  obtain that
\begin{align}
    |y_{i}\cdot \widehat\bg_i^\top \widehat\bmu-y_{i}\cdot \tilde\bg_i^\top  \bmu|\le ~&  \Big|(\tilde\bg_i-\bar\bg_i)^\top {\bmu}\Big|+ \Big|(\bar\bg_i- \widehat\bg_i)^\top {\bmu}\Big|
    +  \Big| \widehat\bg_i^\top(\widehat\bmu- \bmu)\Big| = o(\sqrt{q_m}),
\end{align}
with probability at least $1-cN^{-10}$ for some constants $c,C>0$.  
Therefore, we can take $\zeta=1$, $\eps_m=o(1)$ and $\rho= s_{\text{opt}}\cdot q_m$ to get
\begin{align}
\P(\psi_m(\sign(\widehat\by_{\mathrm{GCN}}),\by_{\sU})=0)=~&\P\big(\min_{i\in[m]}\by_{\sU,i}\cdot\widehat\by_{\mathrm{GCN},i}>0\big)= \P\Big( \min_{i\in\cV_{\sU}}  y_{i}\cdot \widehat\bg_i^\top \widehat\bmu >0\Big)\\
\ge ~& \P\Big( \min_{i\in\cV_{\sU}}  y_{i}\cdot \tilde\bg_i^\top  \bmu >\epsilon_m \sqrt{q_m},~|y_{i}\cdot \widehat\bg_i^\top \widehat\bmu-y_{i}\cdot \tilde\bg_i^\top  \bmu|\le \epsilon_m \sqrt{q_m},~\forall i\in \cV_{\sU}\Big)\\
\ge ~& \P\Big( \min_{i\in\cV_{\sU}} y_{i} \cdot \tilde\bg_{i}^\top  \bmu> \eps_m {\sqrt{q_m}}\Big)- \sum_{i\in\cV_{\sU}}\P\Big(|y_{i}\cdot \widehat\bg_i^\top \widehat\bmu-y_{i}\cdot \tilde\bg_i^\top  \bmu|> \epsilon_m \sqrt{q_m}\Big)\\
\ge ~& \P\big(\min_{i\in\cV_{\sU}}y_{i}\zeta\cdot \frac{1}{ \widetilde d}(\bA_{\rho}\bX)_{i:} \bmu>C\sqrt{q_m}\eps_m\big)-Cm^{-2}\\
\ge ~& 1-\sum_{i\in\cV_{\sU}}\P\Big(y_{i}\cdot \frac{\zeta}{ \widetilde d}(\bA_\rho\bX)_{i:}  \bmu\le C\sqrt{q_m}\eps_m\Big)-Cm^{-2}\\
\ge ~& 1-m\P\Big(y_{i}\cdot \frac{\zeta}{ \widetilde d}(\bA_\rho\bX)_{i:}  \bmu\le C\sqrt{q_m}\eps_m\Big)-Cm^{-2}\\
\ge~& 1-m^{1- \sup_{t\in\R} \{\eps_m t+g(a ,b ,c,\tau,1, s_{\text{opt}},t)\}+\delta}-Cm^{-2},
\end{align}
for any $\delta>0$ and sufficiently large $m$, where in the last line we employ Proposition~\ref{prop:LDP}. Thus, applying Lemma~\ref{lem:rate_fun}, we know that when $J(a_\tau, b_\tau, c_\tau,  1, s_{\text{opt}} )=I(a_\tau,b_\tau,c_\tau)>1$, $\P(\psi_m(\sign(\widehat\by_{\mathrm{GCN}}),\by_{\sU})=0)\to 1$ as $m\to\infty$.

\section{Auxiliary Lemmas and Proofs}
\begin{lemma}\label{lem:Z_concentration}
Let $\bZ\in\R^{N\times d}$ defined in \eqref{eqn:gauss_mixture}. Then, there exists some constant $c,K>0$ such that for any $t>0$
    \begin{align}
        \Parg{\left|\frac{1}{\sqrt{N}} \ones^{\top} \bZ\bmu\right|\ge  t}\le~&2\exp{(-ct^2d)},\\
        \Parg{\left|\frac{1}{N  } \ones^{\top} \bZ\bZ^{\top} \by\right|\ge  t}\le~&2\exp{\left(-cd\min\left\{\frac{t^2}{K^2 },\frac{t}{K}\right\}\right)},\\
        \Parg{\left|\frac{1}{N  } \ones^{\top} \bZ\bZ^{\top} \ones - 1\right|\ge  t}\le~&2\exp{\left(-cd\min\left\{\frac{t^2}{K^2 },\frac{t}{K}\right\}\right)}.
    \end{align}
\end{lemma}
\begin{proof}
Based on general Hoeffding’s inequality Theorem 2.6.3 in \cite{vershynin2018high}, we can get 
\begin{align}
    \Parg{\left|\frac{1}{\sqrt{N}} \ones^{\top} \bZ\bmu\right|\le t}=~&\Parg{\left|\frac{1}{\sqrt{N}} \sum_{i=1}^N \bz_i^\top\bmu\right|\le  t}
    \le  1-2\exp{(-ct^2d)}.
\end{align}
Similarly, by Bernstein’s inequality Theorem 2.8.2 in \cite{vershynin2018high}, we have
\begin{align}
    \Parg{\left|\frac{1}{N  } \ones^{\top} \bZ\bZ^{\top} \by\right|\le  t}\ge~& 1- 2\exp{\left(-cd\min\left\{\frac{t^2}{K^2 },\frac{t}{K}\right\}\right)},\\
    \Parg{\left|\frac{1}{N  } \ones^{\top} \bZ\bZ^{\top} \ones - 1\right|\le  t}\ge~& 1- 2\exp{\left(-cd\min\left\{\frac{t^2}{K^2 },\frac{t}{K}\right\}\right)},
\end{align}
where $K=\|\xi\|_{\psi_2}^2$ for $\xi\sim\cN(0,1)$. 
\end{proof}

\begin{lemma}[Simplified version of Theorem 3.3 in \cite{dumitriu2023exact}]\label{lem:concentrateA}
Let $G = ([N], E)$ be an inhomogeneous Erd\H{o}s-R\'{e}nyi graph associated with the probability matrix $\bP$, that is, each edge $e = \{i, j\}\subset [N]^2$ is sampled from $\mathrm{Ber}(P_{ij})$, namely, $\P(A_{ij} = 1) = P_{ij}$. Let $\bA$ denote the adjacency matrix of $G$. Denote $P_{\max}\coloneqq \max_{i, j\in [N]} P_{ij}$. Suppose that
\begin{align}\label{eqn:concentrateA_condition}
    N \cdot P_{\max} \geq c\log N\,,
\end{align}
for some positive constant $c$, then with probability at least $1-2n^{-10}- 2e^{-N}$, adjacency matrix $\bA$ satisfies
\begin{align}\label{eqn:concentrateA}
    \|\bA - \E \bA \| \leq \const_{\eqref{eqn:concentrateA}}\cdot \sqrt{N \cdot P_{\max}}\,,
\end{align}
\end{lemma}

\begin{lemma}[Bernstein's inequality, Theorem 2.8.4 of \cite{vershynin2018high}]\label{lem:Bernstein}
    Let $X_1,\dots,X_n$ be independent mean-zero random variables such that $|X_i|\leq K$ for all $i$. Let $\sigma^2 = \sum_{i=1}^{n}\E X_i^2$. 
    Then for every $t \geq 0$,
    \begin{align}
        \P \Bigg( \Big|\sum_{i=1}^{n} X_i \Big| \geq t \Bigg) \leq 2 \exp \Bigg( - \frac{t^2/2}{\sigma^2 + Kt/3} \Bigg)\,.
    \end{align}
\end{lemma}

\end{document}